\documentclass[11pt,twoside]{article} 

\usepackage{eqnarray,amsmath}
\usepackage[utf8]{inputenc} 
\usepackage[T1]{fontenc}    

\usepackage{booktabs}       
\usepackage{amsfonts}       
\usepackage{nicefrac}       
\usepackage{microtype}      
\usepackage{adjustbox} 

\usepackage{subcaption}
\usepackage{epsf}
\usepackage{epsfig}
\usepackage{fancyhdr}
\usepackage{graphics}
\usepackage{graphicx}

\usepackage{psfrag}
\usepackage{fullpage}
\usepackage{pdfpages}
\usepackage{makecell} 

\usepackage{natbib}

\usepackage{url}

\usepackage[colorlinks,linkcolor=magenta,citecolor=blue, pagebackref=true]{hyperref}

\renewcommand*{\backrefalt}[4]{%
    \ifcase #1 \footnotesize{(Not cited.)}%
    \or        \footnotesize{(Cited on page~#2.)}%
    \else      \footnotesize{(Cited on pages~#2.)}%
    \fi}

\usepackage{color}

\usepackage{amsthm}
\usepackage{amsmath}
\usepackage{amssymb,bbm}
\usepackage{caption}
\usepackage{algorithmic}
\usepackage{algorithm}
\usepackage{textcomp}
\usepackage{siunitx}
\usepackage{wrapfig}
\usepackage{algorithmic}
\usepackage{algorithm}
\usepackage{multirow}
\usepackage{multicol}

\usepackage{bm}

\renewcommand{\arraystretch}{1.4}

\theoremstyle{plain}
\newtheorem{theorem}{Theorem}[section]
\newtheorem{proposition}[theorem]{Proposition}
\newtheorem{lemma}[theorem]{Lemma}

\theoremstyle{definition}
\newtheorem{definition}[theorem]{Definition}

\theoremstyle{remark}

\newcommand{\norm}[1]{\left\lVert#1\right\rVert}

\def\RR{\mathbb{R}}

\def\PP{\mathbb{P}}

\newcommand{\xbm}{{\bm x}}
\newcommand{\Xbm}{{\bm X}}

\newcommand{\ybm}{{\bm y}}

\newcommand{\Wbm}{\bm{W}}

\newcommand{\bbm}{{\bm b}}

\newcommand{\Abm}{{\bm A}}

\newcommand{\Kbm}{{\bm K}}

\newcommand{\vbm}{{\bm v}}
\newcommand{\Vbm}{{\bm V}}

\newcommand{\qbm}{{\bm q}}
\newcommand{\Qbm}{{\bm Q}}

\newcommand{\sbm}{{\bm s}}

\newcommand{\vtheta}{{\bm \theta}}

\newcommand{\veta}{{\bm \eta}}

\DeclareMathOperator*{\argmin}{arg\,min}

\newcommand{\bbE}{\mathbb{E}}
\newcommand{\var}{\mathrm{Var}}

\newcommand{\softmax}{\mathrm{Softmax}}

\newcommand{\gelu}{\mathrm{GELU}}
\newcommand{\relu}{\mathrm{ReLU}}

\newcommand{\dint}{\mathrm{d}}

\newcommand{\daijn}{\Delta \Abm_{ij}^{n}}
\newcommand{\dbijn}{\Delta \bbm_{ij}^{n}}
\newcommand{\deijn}{\Delta \veta_{ij}^{n}}

\newcommand{\boin}{\beta_{1i}^n}
\newcommand{\bzin}{\beta_{0i}^n}
\newcommand{\ain}{A_i^n}
\newcommand{\bin}{b_i^n}
\newcommand{\cin}{c_i^n}
\newcommand{\ein}{\eta_i^n}

\newcommand{\boj}{\beta_{1j}^*}
\newcommand{\bzj}{\beta_{0j}^*}
\newcommand{\cj}{c_j^*}
\newcommand{\aj}{\Abm_j^*}
\newcommand{\bj}{\bbm_j^*}
\newcommand{\ej}{\veta_j^*}

\newcommand{\zerod}{\mathbf{0}_d}

\newcommand{\zeroq}{\mathbf{0}_q}

\newcommand{\brj}{\bar{r}_j}

\newcommand{\brone}{\bar{r}_1}

\newcommand{\daione}{\Delta A_{i1}^{n}}
\newcommand{\dbione}{\Delta b_{i1}^{n}}

\newcommand{\normf}[1]{\|#1\|_{L_2(\mu)}}

\begin{document}

\begin{center}

{\bf{\LARGE{Quadratic Gating Mixture of Experts:\\ Statistical Insights into Self-Attention}}}
  
\vspace*{.2in}
{\large{
\begin{tabular}{cccc}
Pedram Akbarian$^{\dagger,\star}$  & Huy Nguyen$^{\dagger,\star}$ & Xing Han$^{\ddagger,\star}$& Nhat Ho$^{\dagger}$
\end{tabular}
}}

\vspace*{.2in}

\begin{tabular}{c}
The University of Texas at Austin$^{\dagger}$\\
Johns Hopkins University$^{\ddagger}$
\end{tabular}

\vspace*{.1in}
\today

\vspace*{.2in}

\begin{abstract}
Mixture of Experts (MoE) models are well known for effectively scaling model capacity while preserving computational overheads.
In this paper, we establish a rigorous relation between MoE and the self-attention mechanism, showing that each row of a self-attention matrix can be written as a quadratic gating mixture of linear experts.
Motivated by this connection, we conduct a comprehensive convergence analysis of MoE models with two different quadratic gating functions, namely the quadratic polynomial gate and the quadratic monomial gate, offering useful insights into the design of gating and experts for the MoE framework. First, our analysis indicates that the use of the quadratic monomial gate yields an improved sample efficiency for estimating parameters and experts compared to the quadratic polynomial gate. Second, parameter and expert estimation rates become significantly faster when employing non-linear experts in place of linear experts. Combining these theoretical insights with the above link between MoE and self-attention, we propose a novel \emph{active-attention} mechanism where we apply a non-linear activation function to the value matrix in the formula of self-attention. Finally, we demonstrate that the proposed active-attention outperforms the standard self-attention through several extensive experiments in various tasks, including image classification, language modeling, and multivariate time series forecasting.
\end{abstract}

\end{center}
\let\thefootnote\relax\footnotetext{$\star$ Equal contribution.}

\section{Introduction}
\label{sec:introduction}
Mixture of experts (MoE) is a statistical machine learning that aggregates the power of multiple specialized sub-models, known as experts, based on the principle of divide and conquer. These experts can be formulated as classifiers \cite{chen2022theory,nguyen2024general}, regression models \cite{faria2010regression,kwon_em_2020}, or feed-forward networks (FFNs) \cite{shazeer2017topk,dai2024deepseekmoe}, which we will denote by $\mathcal{E}_j:\mathbb{R}^d \to \mathbb{R}^{d^\prime}$, for $1\leq j\leq N$, where $N\in\mathbb{N}$ stands for the number of experts. Another key component of MoE is a gating network, which is in charge of allocating input-dependent softmax weight to each expert. In particular, given an input $\xbm \in\mathbb{R}^d$, the gating network first computes expert affinity scores that quantify the similarity between the expert specialization and the input domain, denoted by $\sbm(\xbm):=(s_1(\xbm), s_2(\xbm),\ldots,s_N(\xbm))\in\mathbb{R}^N$. Next, the gating values (mixture weights) are calculated as the softmax normalization of these affinity scores, that is, $\mathcal{G}(\xbm)_j=\softmax(\sbm(\xbm))_j:=\frac{\exp(s_j(\xbm))}{\sum_{i=1}^{N}\exp(s_i(\xbm))}$, for $1\leq j\leq N$. As a result, the gating network can be trained to assign higher weights to more relevant experts. This property helps MoE become more adaptive and obtain a greater expressive power than conventional mixture models \cite{Lindsay-1995}. Then, the final output of MoE is represented as an adaptively weighted sum of the expert outputs:
\begin{align}
    \label{eq:moe}
    \ybm = \sum_{j=1}^N \mathcal{G}(\xbm)_j \cdot \mathcal{E}_j(\xbm)= \sum_{j=1}^N \softmax(\sbm(\xbm))_j \cdot \mathcal{E}_j(\xbm).    
\end{align}\\

\noindent
Thanks to the adaptability of MoE, it has been leveraged in several applications, namely speech recognition \cite{You_Speech_MoE,You_Speech_MoE_2}, multi-task learning \cite{hazimeh_dselect_k_2021}, protein–protein interaction prediction \citep{Qi2007},
climate modeling \citep{Lu2006}, financial time‐series prediction \citep{Yumlu2003}, and audio classification \citep{Harb2004}.\\

\noindent
Subsequently, to upgrade the scalability of the standard MoE models, Shazeer et al. \cite{shazeer2017topk} proposed a sparse MoE variant that selectively activates a limited subset of specialized experts for each input instance. This selective activation strategy significantly expands the model's parameter capacity without a proportional increase in computational cost.
Consequently, sparse MoE architectures have attracted widespread interest in a variety of fields such as natural language processing \cite{jiang2024mixtral,Du_Glam_MoE,fedus2022switch,lepikhin_gshard_2021}, computer vision \cite{Riquelme2021scalingvision,liang_m3vit_2022}, multimodal learning \cite{han2024fusemoe,yun2024flex}, and domain adaptation \cite{chi_representation_2022,nguyen2025cosine}, where feed-forward layers in the Transformer architectures \cite{vaswani2017attention} are substituted with sparse MoE layers. \\

\noindent
Transformer models \cite{vaswani2017attention} have recently become dominant in several applications, including language modeling \cite{devlin2019bert,brown2020language,raffel2020exploring,touvron2023llama}, computer vision \cite{dosovitskiy2020image,carion2020end,radford2021learning,liu2021swin}, and reinforcement learning \cite{chen2021decision,kim2023preference,hu2024learning}. The core strength of the Transformers lies in their self-attention mechanism, which enables the models to dynamically focus on important parts of the input. Furthermore, self-attention also constructs context-sensitive representations for each token by weighting input elements according to their relevance, effectively capturing long-range and complex dependencies within the data. Mathematically, given an arbitrary input $\Xbm\in\mathbb{R}^{N \times d}$ of $N$ feature vectors of dimension $d$, the self-attention mechanism first calculates its three key components, namely the query matrix $\Qbm\in \mathbb{R}^{N \times d_{k}}$, the key matrix $\Kbm \in \mathbb{R}^{N \times d_{k}}$, and the value matrix $\Vbm \in \mathbb{R}^{N\times d_v}$ as follows: 
\begin{align*}
  \Qbm  =  \Xbm  \Wbm_q ,
  \quad
  \Kbm  =  \Xbm  \Wbm_k ,
  \quad
  \Vbm  =  \Xbm  \Wbm_v ,    
\end{align*}
where 
$
  \Wbm_q \in \mathbb{R}^{d \times d_{k}}, 
  \Wbm_k \in \mathbb{R}^{d \times d_{k}},
  \Wbm_v \in \mathbb{R}^{d \times d_v}
$
are trainable projection matrices. Next, attention scores, which measure the relevance of one element in an input sequence with respect to other elements, are computed as $\softmax\Bigl(\frac{\Qbm \Kbm^\top}{\sqrt{d_{k}}}\Bigr)$, where the softmax function acts row-wise on the matrix $\Qbm \Kbm^\top/\sqrt{d_k} \in \mathbb{R}^{N \times N}$. Then, the final output of the self-attention mechanism is given by
\begin{align}
\label{eq:attention}
  \mathrm{Attn}(\Xbm)
   = 
  \softmax\Bigl(\frac{\Qbm \Kbm^\top}{\sqrt{d_{k}}}\Bigr) \Vbm = \softmax\Bigl(\frac{\Xbm\Wbm_q \Wbm_k ^{\top}\Xbm^{\top}}{\sqrt{d_k}}\Bigr)\Xbm\Wbm_v .
\end{align}

\noindent
From equations~\eqref{eq:moe} and \eqref{eq:attention}, it can be seen that both MoE and self-attention are written as a softmax weighted sum of functions of the input. Thus, there have been some attempts to explore the connection between these two popular models in the literature. \\
 
\noindent
Recent work explores connections between self-attention and MoE frameworks by integrating MoE-inspired sparse routing into attention mechanisms and employing attention-based routing strategies within MoE models. \emph{SwitchHead} \citep{csordas2024switchhead} selectively routes queries to attention heads, significantly reducing computational and memory overhead, while \emph{Multi-Head MoE} \citep{wu2024multi} decouples experts from attention heads to enhance scalability. \emph{Mixture-of-Head (MoH)} \citep{jin2024moh} explicitly formalizes each head as an expert chosen by a learned router. Earlier approaches like \emph{Mixture of Attention Heads (MoA)} \citep{zhang2022moa} introduced per-token head mixtures as an alternative to uniform averaging, and \emph{Mixture of Softmax (MoS)} \citep{yang2018mos} aimed to overcome the expressiveness limitations of standard softmax attention distributions. \emph{MoEAtt} \citep{chen2023moeatt} further integrates attention-based routing with quadratic gating to provide greater flexibility in expert assignment. However, we notice that the aforementioned links between MoE and self-attention seem to be intuitive rather than rigorously derived. 

\begin{table*}[!ht]
\caption{Summary of expert estimation rates under MoE models equipped with the quadratic polynomial gate and the quadratic monomial gate. Below, $n$ stands for the sample size, while $\tau$ is some positive constant.}
\centering
\begin{tabular}{| c | c | c |} 
\hline
\textbf{Gates} & \textbf{FFN Experts} & \textbf{Linear Experts}\\
\hline 
{\small Quadratic Polynomial Gate} & $\widetilde{\mathcal{O}}_P(n^{-1/4})$ & $\widetilde{\mathcal{O}}_P(1/\log^{\tau}(n))$ \\
\hline
{\small Quadratic Monomial Gate} &$\widetilde{\mathcal{O}}_P(n^{-1/4})$ & $\widetilde{\mathcal{O}}_P(1/\log^{\tau}(n))$\\
\hline
\end{tabular}
\label{table:linear_nonlinear_experts}
\end{table*}

\noindent
\textbf{Contributions.} Therefore, in this paper, we aim to establish a stronger relation between these two models with theoretical guarantee. Furthermore, we also provide a comprehensive convergence analysis of the resulting MoE model from that relation to offer new insight into the design of self-attention. Our contributions are three-fold and can be summarized as follows:
\\

\noindent
\textbf{1. Relation between self-attention mechanism and quadratic gating MoE.}
Building on the similarity in the formulations of self-attention and MoE in equations~\eqref{eq:moe} and \eqref{eq:attention}, we unify these two popular concepts by showing that each row of a self-attention matrix can be represented as a quadratic gating MoE (see Proposition~\ref{prop:moe-attention}).
This relation motivates us to study the quadratic gating MoE in more detail.
\\

\noindent
\textbf{2. Theoretical analysis of the quadratic gating MoE.} We explore the effects of two quadratic gating functions, namely the quadratic polynomial gating function and the quadratic monomial gating function, on the convergence rates of parameter estimation and expert estimation under the MoE models. For the convergence analysis of each gate, we provide a corresponding strong identifiability condition (see Definitions~\ref{def:strong_identifiability} and \ref{def:mono_strong_identifiability}) to characterize the compatible structure of experts with that gate. These conditions indicate that approximating experts formulated as neural networks with non-linear activation functions require a much smaller amount of data than estimating linear experts (see Theorems~\ref{theorem:poly_strongly_identifiable} and \ref{theorem:poly_linear_experts}).
\\

\noindent
\textbf{3. Practical implications.} 
 There are two important practical implications of the convergence analysis of quadratic gating MoE. Firstly, this analysis encourages the usage of non-linear experts over linear experts in quadratic gating MoE models (see Table~\ref{table:linear_nonlinear_experts}). Given that insight, we propose a novel \emph{active-attention mechanism} in equation~\eqref{eq:novel_attention} by applying a non-linear activation to the value matrix in the self-attention mechanism. Through extensive empirical evaluation, we show the favorable performance of the proposed active-attention over standard attention in various tasks, namely image classification, language modeling, and multivariate time series forecasting.
 Secondly, the theoretical results indicate that quadratic monomial gate yields faster expert and parameter estimation rates than quadratic polynomial gate (see Table~\ref{table:parameter_rates}). This property confirms the benefits of the widely used linear embeddings of the keys and queries without bias terms in the self-attention mechanism in practice.  

\begin{table*}[!ht]
\caption{Summary of parameter estimation rates under the mixture of strongly identifiable experts equipped with the quadratic polynomial gate and the quadratic monomial gate. The function $\bar{r}$ below is defined in equation~\eqref{definition:polynomial_equation}, while Voronoi cells $\mathcal{V}_j$ are defined in equation~\eqref{eq:Voronoi_cells}.}
\centering
\begin{tabular}{| c | c | c | c | c |} 
\hline
\textbf{Gates}  & $\exp(c^*_j)$ & $\Abm^*_j$ & $\bbm^*_j$ & $\veta^*_j$\\
\hline 
{\small Quadratic Polynomial Gate} &  $\widetilde{\mathcal{O}}_P(n^{-1/2})$ & $\widetilde{\mathcal{O}}_P(n^{-1/\bar{r}(|\mathcal{V}_{j}|)})$ & $\widetilde{\mathcal{O}}_P(n^{-1/2\bar{r}(|\mathcal{V}_{j}|)})$ & $\widetilde{\mathcal{O}}_P(n^{-1/4})$ \\
\hline
{\small Quadratic Monomial Gate} &  $\widetilde{\mathcal{O}}_P(n^{-1/2})$ & $\widetilde{\mathcal{O}}_P(n^{-1/4})$ & $\widetilde{\mathcal{O}}_P(n^{-1/4})$ & $\widetilde{\mathcal{O}}_P(n^{-1/4})$\\
\hline
\end{tabular}
\label{table:parameter_rates}
\end{table*} 

\noindent
\textbf{Organization.} The paper proceeds as follows. In Section~\ref{sec:background}, we establish a relation between the self-attention mechanism and MoE, and then introduce two quadratic gating functions of our interest. Next, we investigate the convergence behavior of parameter estimation and expert estimation under MoE models with these quadratic gates, following by two important practical implications of that analysis in Section~\ref{sec:theory}. Then, we discuss related works to our paper in Section~\ref{sec:related_work}. After that, we conduct extensive experiments to empirically justify the theoretical results and favorable performance of the active-attention mechanism in Section~\ref{sec:experiments}. Then, we discuss related works to our paper in Section~\ref{sec:related_work}. Finally, we conclude the paper, highlight some limitations of our work, and list a few potential future directions in Section~\ref{sec:discussion}. Additional results and theoretical proofs are deferred to the Appendices. \\

\noindent
\textbf{Notation.} We let $[n]:=\{1,2,\ldots,n\}$ for any  $n\in\mathbb{N}$. Next, for any set $S$, we denote $|S|$ as its cardinality. For any vector $\vbm \in \mathbb{R}^{d}$ and $\bm{\alpha}:=(\alpha_1,\alpha_2,\ldots,\alpha_d)\in\mathbb{N}^d$, we let $\vbm^{\bm \alpha}=v_{1}^{\alpha_{1}}v_{2}^{\alpha_{2}}\ldots v_{d}^{\alpha_{d}}$, $|\vbm|:=v_1+v_2+\ldots+v_d$ and $\bm{\alpha}!:=\alpha_{1}!\alpha_{2}!\ldots \alpha_{d}!$, while $\|\vbm\|$ stands for its $\ell_2$-norm value. 
Additionally, for any two positive sequences $(a_n)_{n\geq 1}$ and $(b_n)_{n\geq 1}$, we write $a_n = \mathcal{O}(b_n)$ or $a_{n} \lesssim b_{n}$ if $a_n \leq C b_n$ for all $ n\in\mathbb{N}$, where $C > 0$ is some universal constant. For a sequence $(A_n)_{n\geq 1}$ of positive random variables, the notation $A_{n} = \mathcal{O}_{P}(b_{n})$ indicates that $A_{n}/b_{n}$ is stochastically bounded, that is, for any $\epsilon>0$, there exists an $M>0$ such that $\mathbb{P}( A_{n}/b_{n} > M) < \epsilon $ for all sufficiently large $n$. Lastly, we write $A_n = \widetilde{\mathcal{O}}_{P}(b_n)$ when $A_n = \mathcal{O}_{P}(b_n \log^c(b_n))$, for some $c > 0$. 

\section{Preliminaries}
\label{sec:background}
In this section, we illustrate the link between self-attention mechanism and MoE in Section~\ref{sec:moe-attention}, and then provide background on quadratic gating MoE in Section~\ref{sec:quadratic_MoE}.
\subsection{Connection between MoE and Self-Attention}
\label{sec:moe-attention}
We start with establishing a rigorous relation between the MoE framework and the self-attention mechanism in Proposition~\ref{prop:moe-attention}.

\begin{proposition}
    \label{prop:moe-attention}
    Each row of the self-attention matrix in equation~\eqref{eq:attention}  can be written as an MoE with a quadratic affinity score function, which we refer to as a quadratic gating MoE.
\end{proposition}
\begin{proof}
 Firstly, recall that the self-attention matrix in equation~\eqref{eq:attention} is defined as
\begin{align}
  \mathrm{Attn}(\Xbm) = \softmax\Bigl(\frac{\Xbm\Wbm_q \Wbm_k ^{\top}\Xbm^{\top}}{\sqrt{d_k}}\Bigr) \Xbm \Wbm_v \in \mathbb{R}^{N\times d_v},
\end{align}
where $\Xbm \in \mathbb{R}^{N\times d}$, $\Wbm_q \in \mathbb{R}^{d \times d_{k}}$, $\Wbm_k \in \mathbb{R}^{d \times d_{k}}$, and $\Wbm_v \in \mathbb{R}^{d \times d_v}$. Let $\xbm_i^\top \in \mathbb{R}^{d}$ denote the $i$-th row of $\Xbm$, and define $\bm{M} := \frac{\Wbm_q \Wbm_k ^{\top}}{\sqrt{d_k}} \in \mathbb{R}^{d \times d}$. Note that the softmax function is applied row-wise to the matrix $\Xbm \bm{M}\Xbm^{\top} \in \mathbb{R}^{N \times N}$. Then, the $i$-th row of the self-attention score matrix can be expressed as:
\begin{align*}
[\softmax(\Xbm \bm{M} \Xbm^{\top})]_{i,:} = \softmax(\xbm_i^\top \bm{M} \Xbm^{\top}).
\end{align*}
Thus, the $i$-th row of the self-attention output can be written as:
\begin{align}
    \label{eq:attention_row}
    [\mathrm{Attn}(\Xbm)]_{i,:} 
    &= \sum_{j=1}^N \softmax(\xbm_i^\top \bm{M} \Xbm^{\top})_j \cdot \xbm_j^\top \Wbm_v\\
    &= \sum_{j=1}^N \frac{\exp(\xbm_i^\top \bm{M} \xbm_j)}{\sum_{k=1}^N\exp(\xbm_i^\top \bm{M} \xbm_k)} \cdot \xbm_j^\top \Wbm_v.
\end{align}
Now, let $\widetilde{\xbm} = \mathrm{vec}(\Xbm) = (\xbm_1^\top, \ldots, \xbm_N^\top)^\top \in \mathbb{R}^{Nd}$ denote the vectorization of $\Xbm$. For each token index $i \in {1,\ldots,N}$, let $\bm{e}_i \in \mathbb{R}^N$ denote the $i$-th standard basis vector, and define the corresponding row-selector matrix as $\bm{P}_i := \bm{e}_i^\top \otimes \bm{I}_d$, where $\otimes$ denotes the Kronecker product, that is,
\begin{align*}
\bm{P}_i
= \bm{e}_i^\top \otimes \bm{I}_d = 
\begin{bmatrix}
\bm{0}_{d\times d} & \cdots & \bm{0}_{d\times d} & \underbrace{\bm{I}_d}_{i\text{-th block}} & \bm{0}_{d\times d} & \cdots & \bm{0}_{d\times d}
\end{bmatrix} \in \mathbb{R}^{d \times Nd},
\end{align*}
In particular, \(\bm{P}_i\) extracts the \(i\)th row of \(\Xbm\), that is, \(\bm{P}_i\,\widetilde{\xbm}=\xbm_i\). Then, the $(i,j)$ entry of the logit matrix $\Xbm \bm{M} \Xbm^\top$ can be written in a quadratic form
\begin{align}
    [\Xbm \bm{M} \Xbm^\top]_{i,j} = \xbm_i^\top \bm{M} \xbm_j = \widetilde{\xbm}^\top \underbrace{\bm{P}_i^\top \bm{M} \bm{P}_j}_{=: \bm{B}_{ij}} \widetilde{\xbm}.
\end{align}
Therefore, the $i$-th row of the self-attention matrix in equation~\eqref{eq:attention_row} becomes:
\begin{align*}
    [\mathrm{Attn}(\Xbm)]_{i,:} &= \sum_{j=1}^N \frac{\exp(\widetilde{\xbm}^\top \bm{B}_{ij} \widetilde{\xbm})}{\sum_{k = 1}^{N}\exp(\widetilde{\xbm}^\top \bm{B}_{ik} \widetilde{\xbm})} \cdot \widetilde{\xbm}^\top \bm{P}_j^\top \Wbm_v \\
    &= \sum_{j=1}^N \frac{\exp(\widetilde{\xbm}^\top \bm{B}_{ij} \widetilde{\xbm})}{\sum_{k = 1}^{N}\exp(\widetilde{\xbm}^\top \bm{B}_{ik} \widetilde{\xbm})} \cdot \widetilde{\xbm}^\top \bm{E}_j,
\end{align*}
where $\bm{B}_{ij} := \bm{P}_i^\top \bm{M} \bm{P}_j = \frac{\bm{P}_i^\top \Wbm_q \Wbm_k^\top \bm{P}_j}{\sqrt{d_k}} \in \mathbb{R}^{Nd \times Nd}$ and $\bm{E}_j := \bm{P}_j^\top \Wbm_v \in \mathbb{R}^{Nd \times d_v}$. Hence, each row of the self-attention matrix can be represented as an MoE with quadratic affinity scores.
\end{proof}

\subsection{Quadratic Gating MoE}
\label{sec:quadratic_MoE}
Inspired by connection in Proposition~\ref{prop:moe-attention}, we consider a class of \emph{quadratic} gating MoE~\cite{liao2007quadratically} models as an alternative to the conventional linear gating MoE. 
By allowing for more flexible decision boundaries through a quadratic scoring function, these models aim to provide better adaptability in expert selection. In the sequel, we particularly focus on two following quadratic gating functions: \\

\emph{(1) Quadratic polynomial gating function}, which is inspired by the full linear embeddings with biases of keys and queries in the attention, takes the form: 
\begin{align}
    \label{eq:quadratic_gate}
    \xbm \in\mathbb{R}^d\mapsto\mathcal{G}(\xbm)_i = \frac{\exp\left(\xbm^\top \Abm_i \xbm + \bbm_i^\top \xbm + c_i\right)}{\sum_{j=1}^N \exp\left(\xbm^\top \Abm_j \xbm + \bbm_j^\top \xbm + c_j\right)}
    ,
\end{align}
where $(\Abm_i, \bbm_i, c_i) \in \RR^{d \times d} \times \RR^d \times \RR$, for $i\in[N]$, are learnable parameters; \\

\emph{(2) Quadratic monomial gating function}, which is motivated by the widely used linear embeddings without biases of keys and queries in the attention, is given by
\begin{align}   
\label{eq:quadratic_mononial_gate}
    \xbm\in\mathbb{R}^d\mapsto\mathcal{G}(\xbm)_i = \frac{\exp\left(\xbm^\top \Abm_i \xbm + c_i\right)}{\sum_{j=1}^N \exp\left(\xbm^\top \Abm_j \xbm + c_j\right)}.
\end{align}
Subsequently, we will present a comprehensive convergence analysis for the quadratic polynomial gating of MoE in Section~\ref{sec:theory}, while that for quadratic monomial gating MoE is in Section~\ref{sec:quadratic_monomial}.\\

\noindent
\textbf{Parameter count overheads.} Introducing quadratic gating significantly increases the number of model parameters due to the additional quadratic terms in the gating network, which can lead to higher computational and memory demands. To mitigate this overhead, we can employ low-rank embeddings for the quadratic terms. 
This approach retains the advantages of quadratic gating while minimizing the overhead, making it a practical enhancement for MoE models. Appendix~\ref{appendix:overhead} offers a more thorough discussion on this topic.

\section{Convergence Analysis of Quadratic Gating MoE}
\label{sec:theory}
Motivated by the link between the self-attention mechanism and quadratic gating MoE in Section~\ref{sec:moe-attention}, we will study the impacts of two quadratic gating functions in equations~\eqref{eq:quadratic_gate} and \eqref{eq:quadratic_mononial_gate} on the convergence behavior of least squares expert estimators. 
We will examine the quadratic polynomial gate in Section~\ref{sec:quadratic_polynomial}, while we defer the analysis for the quadratic monomial gate to Section~\ref{sec:quadratic_monomial}. Lastly, we exhibit two important practical implications on the design of self-attention and MoE models in Section~\ref{sec:practical_implications}.

\subsection{Quadratic Polynomial Gate}
\label{sec:quadratic_polynomial}
To begin with, let us formally present the regression framework used for our analysis of the quadratic polynomial gate. Assume that the data $(\Xbm_1,Y_1), (\Xbm_2,Y_2),\ldots,(\Xbm_n, Y_n)\in\mathcal{X}\times\mathcal{Y}\subseteq\mathbb{R}^d\times\mathbb{R}$ are i.i.d. sampled from the following model:
\begin{align}
    \label{eq:regression_framework}
    Y_i=f_{G_*}(\Xbm_i)+\varepsilon_i, \quad i=1,2,\ldots,n,
\end{align}
where $\varepsilon_1,\varepsilon_2,\ldots,\varepsilon_n$ are independent Gaussian noise variables such that $\bbE[{\varepsilon_{i}}|\Xbm_i] = 0$ and $\var(\varepsilon_{i}|\Xbm_i) = \sigma^2$, for all $i\in[n]$. Additionally, we assume that $\Xbm_{1},\Xbm_{2}, \ldots,\Xbm_{n}$ are i.i.d. samples from some known probability distribution $\mu$. Above, the regression function $f_{G_{*}}(\cdot)$ admits the form of a quadratic polynomial gating MoE model with $N^*$ experts, namely
\begin{align}
    f_{G_{*}}(\xbm) :=
    \sum_{i=1}^{N^*} \frac{\exp(\xbm^{\top}\Abm^*_i\xbm+(\bbm^*_i)^{\top}\xbm+c^*_i)}{\sum_{j=1}^{N^*}\exp(\xbm^{\top}\Abm^*_j\xbm+(\bbm^*_j)^{\top}\xbm+c^*_j)}\cdot \mathcal{E}(\xbm,\veta^*_i),
\end{align}
where $(\Abm^*_i,\bbm^*_i,c^*_i,\veta^*_i)_{i=1}^{N^*}$ are unknown ground-truth parameters in $\mathbb{R}^{d\times d}\times\mathbb{R}^d\times\mathbb{R}\times\mathbb{R}^q$, and $G_{*} := \sum_{i = 1}^{N_{*}} \exp(c^{*}_i) \delta_{(\Abm_{i}^{*},\bbm^*_i,\veta^*_i)}$ denotes the associated \emph{mixing measure}, that is, a weighted sum of Dirac measures $\delta$. Meanwhile, the function $\xbm\in\mathcal{X}\mapsto\mathcal{E}(\xbm;\veta)\in\mathbb{R}$ is referred to as {\it the expert function,} which we assume to be of parametric form.
\\

\noindent
\textbf{Least squares estimation:} We estimate the unknown parameters $(\Abm^*_i,\bbm^*_i,c^*_i,\veta^*_i)_{i=1}^{N^*}$ through estimating the ground-truth mixing measure $G_*$ by deploying the least squares method~\cite{vandeGeer-00} as follows:
\begin{align}
    \label{eq:least_squared_estimator}
    \widehat{G}_n:=\argmin_{G\in\mathcal{M}_{N}(\Theta)}\sum_{i=1}^{n}\Big(y_i-f_{G}(\xbm_i)\Big)^2,
\end{align}
where $\mathcal{M}_{N}(\Theta):=\{G=\sum_{i=1}^{N'}\exp(c_i)\delta_{(\Abm_i,\bbm_i,\veta_{i})}:1\leq N'\leq N, \  (\Abm_i,\bbm_i,c_i,\veta_{i})\in\Theta\}$ is the set of all mixing measures with at most $N$ components, where $N>N^*$. The goal of this paper is to explore the convergence properties of the estimator $\widehat{G}_n$ in a fixed-dimensional setting.\\

\noindent
\textbf{Universal assumptions.} For the sake of theory, we make the following three mild assumptions on the input and the ground-truth parameters throughout the paper. 

\emph{(A.1) The input space $\mathcal{X}$ is bounded, and the parameter space $\Theta$ is compact with fixed dimension. Furthermore, the expert function $\xbm\mapsto\mathcal{E}(\xbm,\veta)$ is bounded and Lipschitz continuous with respect to $\veta$.}

\emph{(A.2) The last tuple of gating parameters equals to zero, that is, $\Abm^*_{N^*}=\boldsymbol{0}_{d\times d}$, $\bbm^*_{N^*}=\boldsymbol{0}_d$, and $c^*_{N^*}=0$. Additionally, the set of parameters $\{ (\Abm^*_{i},\bbm^*_{i}): i\in[N^*]\}$, has at least one non-zero element.} 

\emph{(A.3) The expert parameters $\veta^*_1,\veta^*_2,\ldots,\veta^*_{N^*}$ have distinct values.} \\

\noindent
Here, the assumption (A.1) helps guarantee the convergence of the least squares estimator $\widehat{G}_n$ to its ground-truth counterpart $G_*$. Next, the first part of assumption (A.2) allows us to keep the quadratic gating MoE identifiable despite the invariance to translation of the softmax function. Meanwhile, the second part of assumption (A.2) is to maintain the dependence of the gating on the input. The assumption (A.3) is necessary to ensure that the expert values are distinct. It should be noted that these assumptions have been used in previous works \cite{chen2022theory,nguyen2024leastsquare}. Given the above assumptions, we demonstrate in Theorem~\ref{theorem:regression_rate} that the convergence rate of regression estimation is parametric on the sample size.
\begin{theorem}[Regression Estimation Rate]
    \label{theorem:regression_rate}
     Equipped with a least squares estimator $\widehat{G}_n$ given in equation~\eqref{eq:least_squared_estimator}, the model estimation $f_{\widehat{G}_n}$ converges to the true model $f_{G_*}$ at the following rate:
    \begin{align}
        \label{eq:model_bound}
        \normf{f_{\widehat{G}_n}-f_{G_*}}=\mathcal{O}_{P}(\sqrt{\log(n)/n}).
    \end{align}
\end{theorem}

\noindent
The proof of Theorem~\ref{theorem:regression_rate} is in Appendix~\ref{appendix:regression_rate}. From the result of this theorem, it can be seen that if we are able to establish the lower bound $\normf{f_{\widehat{G}_n}-f_{G_*}} \gtrsim \mathcal{L}(\widehat{G}_n,G_*)$ where  $\mathcal{L}$ is some loss function among parameters, then we obtain the parameter estimation rate $\mathcal{L}(\widehat{G}_n,G_*) =\mathcal{\widetilde{O}}_{P}(n^{-1/2})$. This approach plays an vital role in establishing the rates for estimating individual parameters as well as experts in the sequel.\\

\noindent
Turning to the parameter and expert estimation problem. A key step to establish the parameter and expert estimation rates is to decompose the discrepancy $f_{\widehat{G}_n}(\xbm)-f_{G_*}(\xbm)$ into a combination of linearly independent terms via Taylor expansions to the function $F(\xbm;\Abm,\bbm,\veta):=\exp(\xbm^{\top}\Abm \xbm + \bbm^{\top}\xbm)\mathcal{E}(\xbm,\veta)$. However, we notice that there is an interaction among gating parameters $\Abm$ and $\bbm$ expressed by the following partial differential equation (PDE):
\begin{align}
    \label{eq:PDE}
    \dfrac{\partial F}{\partial \Abm}(\xbm;\Abm,\bbm,\veta)=\dfrac{\partial^2F}{\partial \bbm \partial \bbm^{\top}}(\xbm;\Abm,\bbm,\veta).
\end{align}
Technically, such parameter interaction induces plenty of linearly dependent derivative terms in the decomposition of $f_{\widehat{G}_n}(\xbm)-f_{G_*}(\xbm)$, which is undesirable. To capture this interaction, we need to consider a system of polynomial equations as described below to construct a loss function among parameters used for the parameter estimation problem.\\

\noindent
\textbf{System of polynomial equations.} Let $\bar{r}(m)$ be the smallest natural number $r$ such that the following system of polynomial equations does not admit any non-trivial solutions for the unknown variables $(p_l,\gamma_{1l},\gamma_{2l})_{l=1}^{m}\subseteq\mathbb{R}^3$:
\begin{align}
    \label{definition:polynomial_equation}
    \sum_{l=1}^{m}\sum_{\substack{n_1,n_2\in\mathbb{N}\\ n_1+2n_2=\alpha}}\dfrac{p_l^2~\gamma_{1l}^{n_1}~\gamma_{2l}^{n_2}}{n_1!~n_2!}=0, \quad \alpha=1,2,\ldots,r.
\end{align}
A solution to the above system is regarded as non-trivial if all variables $p_{l}$ are non-zero, whereas at least one of the $\gamma_{1l}$ is different from zero. As shown in [Proposition 2.1, \cite{Ho-Nguyen-Ann-16}], we have $\bar{r}(2)=4$, $\bar{r}(3)=6$ and $\Bar{r}(m)\geq 7$ when $m\geq 4$.\\

\noindent
Next, we introduce a condition called \emph{poly-strong identifiability} to characterize the types of expert functions that admit faster estimation rates than others. From a technical view, the purpose of the poly-strong identifiability condition is to eliminate all potential interactions among parameters expressed in the language of PDEs as in equation~\eqref{eq:PDE}.
\begin{definition}[Poly-strong identifiability]
    \label{def:strong_identifiability}
    We say that an expert function $\xbm \mapsto \mathcal{E}(\xbm,\veta)$ is strongly identifiable if it is twice differentiable w.r.t its parameter $\veta$, and if for any $N\geq 1$ and distinct parameters $\veta_1,\veta_2,\ldots,\veta_N$, the following set of functions in $\xbm$
    \begin{multline*}
        \Big\{\xbm^{\nu}\cdot\frac{\partial^{|\gamma|}\mathcal{E}}{\partial\veta^{\gamma}}(\xbm;\veta_j): j\in[N], \nu\in\mathbb{N}^d,\gamma\in\mathbb{N}^q, \ 0\leq|\gamma|\leq r_j, \ 0\leq|\nu|\leq 2(r_j-|\gamma|)\Big\},
    \end{multline*}
    is linearly independent for almost every $\xbm$ for any $r_j\leq \bar{r}(N-N^*+1)$.
\end{definition}

\noindent
\textbf{Examples.} It can be verified that the poly-strong identifiability condition holds for experts formulated as two-layer FFNs with non-linear activation functions such as $\relu$ and $\gelu$, that is, $\mathcal{E}(\xbm,(\boldsymbol{\beta}_2,\boldsymbol{\beta}_1,\beta_0))=\boldsymbol{\beta}_2\varphi(\boldsymbol{\beta}_1^{\top}x+\beta_0)$. However, linear experts of the form $\mathcal{E}(\xbm,(\boldsymbol{\beta}_1,\beta_0))=\boldsymbol{\beta}_1^{\top}\xbm+\beta_0$ fail to satisfy this condition as the two terms $ \frac{\partial\mathcal{E}}{\partial\boldsymbol{\beta}_1^{(u)}}(\xbm,(\boldsymbol{\beta}_1,\beta_0))$ and $\xbm^{(u)}\cdot\frac{\partial\mathcal{E}}{\partial\beta_0}(\xbm,(\boldsymbol{\beta}_1,\beta_0))$ are linearly dependent for all $u\in[d]$, that is,
\begin{align}
    \label{eq:violate_PDE}
    \frac{\partial\mathcal{E}}{\partial\boldsymbol{\beta}_1^{(u)}}(\xbm,(\boldsymbol{\beta}_1,\beta_0))=\xbm^{(u)}\cdot\frac{\partial\mathcal{E}}{\partial\beta_0}(\xbm,(\boldsymbol{\beta}_1,\beta_0)).
\end{align}
In the sequel, we determine the parameter and expert estimation rates when using poly-strongly identifiable experts and linear experts in Section~\ref{sec:poly-strongly-experts} and Section~\ref{sec:linear_expert}, respectively. 

\subsubsection{Poly-strongly Identifiable Experts}
\label{sec:poly-strongly-experts}
To characterize the convergence behavior of strongly identifiable experts, let us construct a loss function among parameters based on a notion of Voronoi cells \cite{manole22refined}. Given an arbitrary mixing measure $G$ with $N'\leq N$ components, we distribute its components to the following Voronoi cells, which are generated by the components of $G_*$:
\begin{align}
    \label{eq:Voronoi_cells}
    \mathcal{V}_j\equiv\mathcal{V}_j(G) & := 
    \{i\in[N']:\|\boldsymbol{\omega}_i-\boldsymbol{\omega}^*_j\|\leq\|\boldsymbol{\omega}_i-\boldsymbol{\omega}^*_{\ell}\|,\forall \ell\neq j\},
\end{align}
where we denote $\boldsymbol{\omega}_i:=(\Abm_{i},\bbm_i,\veta_i)$ and $\boldsymbol{\omega}^*_j:=(\Abm^*_j,\bbm^*_j,\veta^*_j)$ for any $j\in[N^*]$. Notably, the cardinality of Voronoi cell $\mathcal{V}_j$ is exactly the number of fitted components that approximates $\omega^*_j$. Then, the Voronoi loss function used for our analysis is given by:
\begin{align}
\label{eq:D1_loss}
\mathcal{L}_1(G,G_*):=\sum_{j=1}^{N^*}\Big|\sum_{i\in\mathcal{V}_j}\exp(c_i)&-\exp(\cj)\Big| + \sum_{j:|\mathcal{V}_j|=1}\sum_{i\in\mathcal{V}_j}\exp(c_i)[\|\Delta \Abm_{ij}\|+\|\Delta\bbm_{ij}\|+\|\Delta\eta_{ij}\|] \nonumber \\
    & +\sum_{j:|\mathcal{V}_j|>1}\sum_{i\in\mathcal{V}_j}\exp(c_i) \Big[\|\Delta \Abm_{ij}\|^{\frac{\bar{r}(|\mathcal{V}_j|)}{2}}+\|\Delta\bbm_{ij}\|^{\bar{r}(|\mathcal{V}_j|)}+\|\Delta\eta_{ij}\|^{2}\Big],
\end{align}
where we denote $\Delta \Abm_{ij}:=\Abm_{i}-\Abm^*_{j}$, $\Delta \bbm_{ij}:=\bbm_i-\bbm^*_{j}$, and $\Delta \veta_{ij}:=\veta_i-\veta^*_j$. 
\\

\noindent
Equipped with the above Voronoi loss $\mathcal{L}_1$, we are now ready to capture the parameter estimation rates in Theorem~\ref{theorem:poly_strongly_identifiable}, whose proof can be found in Appendix~\ref{appendix:poly_strongly_identifiable}.
\begin{theorem}
    \label{theorem:poly_strongly_identifiable} 
    Suppose that the expert function $\xbm \mapsto \mathcal{E}(\xbm,\veta)$ is poly-strongly identifiable, then we achieve the following lower bound for any $G\in\mathcal{G}_N(\Theta)$:
    \begin{align*}
         \normf{f_{G}-f_{G_*}}\gtrsim\mathcal{L}_1(G,G_*),
    \end{align*}
    which together with Theorem~\ref{theorem:regression_rate} indicates that $\mathcal{L}_1(\widehat{G}_n,G_*)=\mathcal{O}_{P}(\sqrt{\log(n)/n})$.
\end{theorem}

\noindent
There are two main implications from the results of Theorem~\ref{theorem:poly_strongly_identifiable}. First, it follows from the formulation of the loss function $\mathcal{L}_1$ that exact-specified parameters $\Abm^*_{j}, \bbm^*_{j},\veta^*_{j}$, i.e. $j\in[N^*]:|\mathcal{V}_j(\widehat{G}_n)|=1$, share the same estimation rate of order $\widetilde{\mathcal{O}}_P(n^{-1/2})$. Note that as the expert $\mathcal{E}(\xbm,\veta)$ is a Lipschitz function in $\veta$, then by denoting $\widehat{G}_n:=\sum_{i=1}^{\widehat{N}_n}\exp(\widehat{c}^n_{i})\delta_{(\widehat{\Abm}^n_{i},\widehat{\bbm}^n_i,\widehat{\veta}^n_i)}$, we get 
\begin{align}
    \label{eq:expert_rate}
    \sup_{\xbm} |\mathcal{E}(\xbm,\widehat{\veta}^n_i)-\mathcal{E}(\xbm,\veta^*_j)| 
    \lesssim \|\widehat{\veta}^n_i-\veta^*_j\|= \widetilde{\mathcal{O}}_P(n^{-1/2}),     
\end{align}
for any $i\in\mathcal{V}_j(\widehat{G}_n)$.
The above bound indicates that if the poly-strongly identifiable expert $\mathcal{E}(\xbm,\veta^*_j)$ is fitted by exactly one expert, it has an estimation rate of order $\widetilde{\mathcal{O}}_P(n^{-1/2})$. Second, for over-specified parameters $\Abm^*_{j}, \bbm^*_{j}, \veta^*_{j}$, where $j\in[N^*]:|\mathcal{V}_j(\widehat{G}_n)|>1$, the rates for estimating them are substantially slower. In particular, the estimation rates for $\Abm^*_j$ and $\bbm^*_j$ are of orders $\widetilde{\mathcal{O}}_P(n^{-1/\bar{r}(|\mathcal{V}_j(\widehat{G}_n)|)})$ and $\widetilde{\mathcal{O}}_P(n^{-1/2\bar{r}(|\mathcal{V}_j(\widehat{G}_n)|)})$, respectively, which are determined by the solvability of the system~\eqref{definition:polynomial_equation}. For instance, when those parameters are fitted by three components, the previous rates become $\widetilde{\mathcal{O}}_P(n^{-1/6})$ and $\widetilde{\mathcal{O}}_P(n^{-1/12})$. Meanwhile, parameters $\veta^*_j$ enjoy an estimation rate of order $\widetilde{\mathcal{O}}_P(n^{-1/4})$. By arguing in a similar fashion to equation~\eqref{eq:expert_rate}, the rates for estimating the experts $\mathcal{E}(\xbm,\veta^*_j)$ are also $\widetilde{\mathcal{O}}_P(n^{-1/4})$.

\subsubsection{Linear Experts}
\label{sec:linear_expert}
Subsequently, we provide the convergence rates for parameter estimation and expert estimation under MoE with the quadratic polynomial gate. \\
 
\noindent
It is worth noting that the linear expert function $\mathcal{E}(\xbm,({\bm \beta}_{1},\beta_{0}))={\bm \beta}_{1}^{\top}\xbm+\beta_{0i}$, where $({\bm \beta}_{1},\beta_{0})\in\mathbb{R}^d\times\mathbb{R}$, violates the poly-strong identifiability condition, leading to an interaction among parameters via the following PDE:
\begin{align}
    \label{eq:PDE_linear}
    \frac{\partial^2 F}{\partial \bbm\partial\beta_0}(\xbm;\Abm^*_{i},\bbm^*_{i},{\bm \beta}^*_{1i},\beta^*_{0i})
     =\frac{\partial F}{\partial {\bm \beta}_1}(\xbm;\Abm^*_{i},\bbm^*_{i},{\bm \beta}^*_{1i},\beta^*_{0i}),
\end{align}
where we denote $F(\xbm;\Abm,\bbm,{\bm \beta}_1,\beta_0):=\exp(\xbm^{\top}\Abm\xbm+\bbm^{\top}\xbm)({\bm \beta}_{1}^{\top}\xbm+\beta_{0})$. 
To capture the effects of such parameter interaction on the convergence of parameter estimation, let us design another Voronoi loss tailored to this setting. More specifically, we define for any $r\geq 1$ that
\begin{align}
    \label{eq:poly_linear_experts_loss}
    \mathcal{L}_{2,r}(G,G_*):=\sum_{j=1}^{N^*}\sum_{i\in\mathcal{V}_j}\exp(c_i)\Big[\|\Delta \Abm_{ij}\|^r+\|\Delta \bbm_{ij}\|^r&+\|\Delta {\bm \beta}_{1ij}\|^r+|\Delta \beta_{0ij}|^r\Big]\nonumber\\
    &+\sum_{j=1}^{N^*}\Big|\sum_{i\in\mathcal{V}_j}\exp(c_i)-\exp(c^*_j)\Big|.
\end{align}
Given the above loss function, we demonstrate in the following theorem that the parameter and expert estimation rates are seriously affected by the parameter interaction in equation~\eqref{eq:PDE_linear}.
\begin{theorem}
    \label{theorem:poly_linear_experts}
    Assume that the experts take the linear form $\mathcal{E}(\xbm,(\boldsymbol{\beta}_1,\beta_0))={\bm \beta}_{1}^{\top}\xbm+\beta_{0}$, then we achieve the following minimax lower bound of estimating $G_*$:
    \begin{align*}
        \inf_{\overline{G}_n\in\mathcal{M}_{N}(\Theta)}\sup_{G\in\mathcal{M}_{N}(\Theta)\setminus\mathcal{M}_{N^*-1}(\Theta)}\bbE_{f_{G}}[\mathcal{L}_{2,r}(\overline{G}_n,G)]\gtrsim n^{-1/2},
    \end{align*}
    for any $r\geq 1$, where $\bbE_{f_{G}}$ indicates the expectation taken w.r.t the product measure with $f^n_{G}$.
\end{theorem}

\noindent
The proof of Theorem~\ref{theorem:poly_linear_experts} is in Appendix~\ref{appendix:poly_linear_experts}. A few remarks on the result of this theorem are in order. First, Theorem~\ref{theorem:poly_linear_experts} reveals that using linear experts make the estimation rates for all the parameters $\Abm^*_{i}$, $\bbm^*_{i}$, ${\bm \beta}^*_{1i}$, and $\beta^*_{0i}$ are slower than polynomial rates $\widetilde{\mathcal{O}}_P(n^{-1/2r})$ for any $r\geq 1$, and, thus, could be as slow as $\mathcal{O}_{P}(1/\log^{\tau}(n))$, for some $\tau>0$, owing to the parameter interaction in equation~\eqref{eq:PDE_linear}. Second, we have that
\begin{align*}
     \sup_{x} \Big|((\widehat{{\bm \beta}}^n_{1i})^{\top}\xbm+\widehat{\beta}^n_{0i})-(({\bm \beta}^*_{1j})^{\top}\xbm+\beta^*_{0j})\Big|\leq \sup_{\xbm} \|\widehat{{\bm \beta}}^n_{1i}-{\bm \beta}^*_{1j}\|\cdot\|\xbm\|+|\widehat{\beta}^n_{0i}-\beta^*_{0j}|.
\end{align*}
Since the input space $\mathcal{X}$ is bounded, the rates for estimating linear experts $({\bm \beta}^*_{1j})^{\top}\xbm+\beta^*_{0j}$ could also be of order $\mathcal{O}_{P}(1/\log^{\tau}(n))$. Hence, combining with the results in Theorem~\ref{theorem:poly_strongly_identifiable}, we deduce that the performance of a mixture of linear experts cannot compare to that of a mixture of non-linear experts in terms of the expert estimation problem. This observation is totally in line with the findings in \cite{chen2022theory}.

\subsection{Quadratic Monomial Gate}
\label{sec:quadratic_monomial}
In this section, we proceed to streamline the analysis of the quadratic monomial gate based on the regression framework in equation~\eqref{eq:regression_framework}. Due to the change in the gating function, the corresponding regression function is reformulated as follows:
\begin{align}
    \tilde{f}_{G_{*}}(\xbm) := \sum_{i=1}^{N^*} \frac{\exp(\xbm^{\top}\Abm^*_i\xbm+c^*_i)}{\sum_{j=1}^{N^*}\exp(\xbm^{\top}\Abm^*_j\xbm+c^*_j)}\cdot \mathcal{E}(\xbm,\veta^*_i).
\end{align}
In comparison with the quadratic polynomial gating, the first-degree monomial term $\bbm^{\top}\xbm$ has been removed from the affinity scores. As a consequence, the least squares estimator under this setting also changes accordingly to
\begin{align}
    \label{eq:mono_least_squared_estimator}
    \widetilde{G}_n:=\argmin_{G\in\mathcal{M}_{N}(\Theta)}\sum_{i=1}^{n}\Big(y_i-\tilde{f}_{G}(\xbm_i)\Big)^2.
\end{align}
Given the above estimator, we provide in Theorem~\ref{theorem:mono_regression_rate} the convergence rate of regression estimation $\tilde{f}_{\widetilde{G}_n}$ to the ground-truth regression function $\tilde{f}_{G_*}$.
\begin{theorem}[Regression Estimation Rate]
    \label{theorem:mono_regression_rate}
     Equipped with a least squares estimator $\widetilde{G}_n$ given in equation~\eqref{eq:mono_least_squared_estimator}, the model estimation $\tilde{f}_{\widetilde{G}_n}$ converges to the true model $\tilde{f}_{G_*}$ at the following rate:
    \begin{align}
        \label{eq:mono_model_bound}
        \normf{\tilde{f}_{\widetilde{G}_n}-\tilde{f}_{G_*}}=\mathcal{O}_{P}(\sqrt{\log(n)/n}).
    \end{align}
\end{theorem}

\noindent
See Appendix~\ref{appendix:mono_regression_rate} for the proof of Theorem~\ref{theorem:mono_regression_rate}. It follows from the bound~\eqref{eq:mono_model_bound} that the regression estimation rate still remains parametric on the sample size, which matches that in Theorem~\ref{theorem:regression_rate} where we use the quadratic polynomial gate in the MoE-type regression function.\\

\noindent
Analogous to Section~\ref{sec:quadratic_polynomial}, we also derive a \emph{mono-strong identifiability} condition in Definition~\ref{def:mono_strong_identifiability} to determine which expert functions will have faster estimation rates than others. 
\begin{definition}[Mono-strong identifiability]
    \label{def:mono_strong_identifiability}
    We say that an expert function $\xbm \mapsto \mathcal{E}(\xbm,\veta)$ is mono-strongly identifiable if it is twice differentiable w.r.t its parameter $\veta$, and if for any $k\geq 1$ and distinct $\veta_1,\veta_2,\ldots,\veta_k$, the following set of functions in $\xbm$
    \begin{align*}
    \left\{\xbm^{\nu}\cdot\frac{\partial^{|\gamma|} \mathcal{E}}{\partial\veta^{\gamma}}(\xbm;\veta_j):j\in[k], \nu\in\mathbb{N}^d,\gamma\in\mathbb{N}^q, \ |\nu|\in\{0,2,4\}, \ 0\leq|\gamma|\leq 2-\frac{|\nu|}{2}\right\},
\end{align*}
    is linearly independent for almost every $\xbm$.
\end{definition}

\noindent
\textbf{Examples.} We can validate that two-layer FFN experts with non-linear activation functions such as $\relu$ and $\gelu$ satisfy the mono-strong identifiability condition. On the other hand, experts of the linear form $\mathcal{E}(\xbm,(\boldsymbol{\beta}_1,\beta_0))=\boldsymbol{\beta}_1^{\top}\xbm+\beta_0$ violate this condition.\\

\noindent
\textbf{Voronoi loss.} Now, we aim to establish the convergence rate of parameter and expert estimation under the mixture of strongly identifiable experts model with the quadratic monomial gating function. For that sake, let us design a new Voronoi loss function among parameters defined as below. 
\begin{align}
\label{eq:D2_loss}
\mathcal{L}_3(G,G_*):=\sum_{j=1}^{N^*}\Big|\sum_{i\in\mathcal{V}_j}\exp(c_i)-\exp(\cj)\Big|&+\sum_{j:|\mathcal{V}_j|=1}\sum_{i\in\mathcal{V}_j}\exp(c_i)\Big[\|\Delta \Abm_{ij}\|+\|\Delta \veta_{ij}\|\Big] \nonumber\\
    &+\sum_{j:|\mathcal{V}_j|>1}\sum_{i\in\mathcal{V}_j}\exp(c_i)\Big[\|\Delta \Abm_{ij}\|^{2}+\|\Delta \veta_{ij}\|^2\Big]. 
\end{align}
Now, we are ready to capture the parameter and expert estimation rates in Theorem~\ref{theorem:mono_strongly_identifiable}.

\begin{theorem}
    \label{theorem:mono_strongly_identifiable} 
    Assume that the expert function $\xbm\mapsto\mathcal{E}(\xbm,\veta)$ is mono-strongly identifiable, then we achieve the following lower bound for any $G\in\mathcal{M}_N(\Theta)$:
    \begin{align*}
         \normf{\tilde{f}_{G}-\tilde{f}_{G_*}}\gtrsim\mathcal{L}_3(G,G_*),
    \end{align*}
    which together with Theorem~\ref{theorem:regression_rate} indicates that $\mathcal{L}_3(\widetilde{G}_n,G_*)=\mathcal{O}_{P}(\sqrt{\log(n)/n})$.
\end{theorem}

\noindent
The proof of Theorem~\ref{theorem:mono_strongly_identifiable} is in Appendix~\ref{appendix:mono_strongly_identifiable}. A few comments regarding this theorem are in order: \\

\noindent
(i) The rates for estimating gating parameters $\Abm^*_j$ fitted by more than one atom, i.e. $|\mathcal{V}_{j}(\widetilde{G}_n)|>1$, are significantly improved to be of order $\widetilde{\mathcal{O}}_{P}(n^{-1/4})$. Those rates are much faster than their counterparts when using the quadratic polynomial gate, which stand at order $\widetilde{\mathcal{O}}_{P}(n^{-1/\bar{r}(|\mathcal{V}_j|)})$ (cf. Theorem~\ref{theorem:poly_strongly_identifiable}). This rate acceleration occurs due to the disappearance of the interaction among gating parameters in equation~\eqref{eq:PDE}. Meanwhile, the estimation rates for expert parameters $\veta^*_j$ remained unchanged at the order of $\widetilde{\mathcal{O}}_{P}(n^{-1/4})$; \\

\noindent
(ii) Model parameters $\Abm^*_j,\veta^*_j$  fitted by exactly one atom, i.e. $|\mathcal{V}_{j}(\widetilde{G}_n)|=1$, enjoy the parametric estimation rates of order $\widetilde{\mathcal{O}}_{P}(n^{-1/2})$, which are comparable to their counterparts in Theorem~\ref{theorem:poly_strongly_identifiable}.\\

\noindent
\textbf{Linear Experts.} Similar to the case of quadratic polynomial gate, the linear expert function $\mathcal{E}(\xbm,({\bm \beta}_{1},\beta_{0}))={\bm \beta}_{1}^{\top}\xbm+\beta_{0}$ in the quadratic monomial gate setting does not satisfy the mono-strong idenfiability condition. That violation yields the parameter and expert estimation rates of order
$\mathcal{O}(1/\log^{\tau}(n))$, for some $\tau>0$. The proof for this result is similar to that of Theorem~\ref{theorem:poly_linear_experts} in Section~\ref{sec:quadratic_polynomial}; therefore, it is omitted here.

\subsection{Practical Implications}
\label{sec:practical_implications}
In this section, we provide two practical implications from the convergence analysis of parameter estimation and expert estimation under the MoE models with the quadratic gating functions in Section~\ref{sec:quadratic_polynomial} and Section~\ref{sec:quadratic_monomial}.\\

\noindent
\textbf{1. New attention mechanism.}  Both the poly-strong identifiability and mono-strong identifiability conditions shed light on the design of new attention mechanism in practice. In particular, we may avoid linear experts as these experts do not satisfy these identifiability conditions and lead to considerably slow rates of parameter and expert estimations. The linear experts correspond to the linear value matrix in the attention mechanism~\eqref{eq:attention}. The poly-strong identifiability and mono-strong identifiability conditions suggest the usage of non-linear experts, which corresponds to the following new attention mechanism:
\begin{align}
    \texttt{Act-Att}(\qbm, \Kbm, \Vbm) = \softmax\left(\frac{\qbm^\top \Kbm^\top}{\sqrt{d_k}}\right) \varphi(\Vbm), \label{eq:novel_attention}
\end{align}
where $\varphi$ is a non-linear activation function. We name the new attention mechanism~\eqref{eq:novel_attention} as \emph{active-attention}. Our experiments with the active-attention in Figure~\ref{fig:image_text} and Table~\ref{tab:time_series} for both classification and time series forecasting tasks with a wide range of non-linear functions demonstrate the favorable performance of active-attention over the standard attention mechanism. \\

\noindent
\textbf{2. Benefits of quadratic monomial gating~\eqref{eq:quadratic_mononial_gate} over quadratic polynomial gating~\eqref{eq:quadratic_gate}.} The remarks after Theorems~\ref{theorem:poly_strongly_identifiable} and~\ref{theorem:mono_strongly_identifiable} indicate that the estimation rates of the  gating parameters are independent of the amount of over-specification of the number of experts and much better than those of the polynomial gating parameters, which become very slow even when we only overspecify the model by a few experts. That theoretical advantage of the monomial gating over the polynomial gating confirms the benefits of the widely used linear embeddings of the keys and queries without bias terms in the attention in practice.

\section{Related Work}
\label{sec:related_work}
\textbf{Quadratic Gating MoE.} Xu et al. \citep{xu1994alternative} introduced an alternative model for MoE formulation, which incorporates a gating network with a distinct parametric structure. Their gating function is defined by
\begin{align}
\label{eq:alternative_gate}
\mathcal{G}(\xbm)_i
&= \frac{\pi_i p(\xbm \mid \vtheta_i)}{\sum_{j=1}^N \pi_j p(\xbm \mid \vtheta_j)},
\quad i = 1, \ldots, N,
\end{align}
where $(\pi_1, \ldots, \pi_N)^{\top}$ belongs to the simplex $\Delta^{N-1}$, and ${(\pi_i, \vtheta_i), i = 1, \ldots, N}$ are learnable parameters. Here, each density $p(\xbm \mid \vtheta_i)$ is from the exponential family. When specifically choosing Gaussian densities with mean $\bm{\mu}_i \in \RR^d$ and covariance matrix $\bm{\Sigma}_i \in \RR^{d\times d}$, the gating function naturally takes a quadratic form:
\begin{align*}
    \mathcal{G}(\xbm)_i = \frac{\exp\left(-\frac12 (\xbm - \bm{\mu}_i)^\top \bm{\Sigma}_{i}^{-1}(\xbm - \bm{\mu}_i) + \log \pi_i\right)}{\sum_{j=1}^N \exp\left(-\frac12 (\xbm - \bm{\mu}_j)^\top \bm{\Sigma}_{j}^{-1}(\xbm - \bm{\mu}_j) + \log \pi_j\right)}.
\end{align*}
This quadratic gating approach enables explicit analytic parameter updates within the Expectation Maximization (EM) algorithm, eliminating the need for iterative nonlinear optimization during training. As a result, the authors introduced a simplified, single-loop EM algorithm, significantly reducing computational complexity and improving convergence speed compared to traditional double-loop EM methods.\\

\noindent
Xu et al. \citep{xu1994alternative} demonstrated their approach on applications such as piecewise polynomial regression and classifier combination, reporting faster convergence and competitive or improved accuracy relative to earlier MoE methods. Their work provides an early theoretical foundation for quadratic gating, directly connecting to modern studies exploring quadratic MoE structures in contemporary deep learning frameworks.\\

\noindent
The gating function $\mathcal{G}(\xbm)_i$ essentially models the posterior probability $\PP(\zeta = i \mid \xbm)$, indicating the likelihood that $\xbm$ is assigned to the partition associated with the $i$-th expert. 
Here, $\zeta \in \{1, \ldots, N\}$ is a latent gating variable that selects a particular expert. 
More precisely, the gating function defined in Equation~\eqref{eq:alternative_gate} interprets this posterior probability when $\zeta$ follows a categorical distribution with parameters $(\pi_1, \ldots, \pi_N)^{\top} \in \Delta^{N-1}$, and conditioned on the event that $\zeta$ selects the $i$-th expert, the distribution of $\xbm$ is modeled by a specific parametric distribution $p(~\cdot \mid \vtheta_i)$.
Motivated by this interpretation of gating function, Liao et al.~\citep{liao2007quadratically} introduced Quadratically Gated Mixture of Experts (QGME), a statistical model for multi-class, nonlinear classification when input features are partially missing.\\

\noindent
\textbf{Theoretical results for MoE models.} Previous works have theoretically investigated MoE models in two main aspects, including the statistical convergence behavior of MoE, and the applications of MoE in deep learning.   
Regarding the statistical convergence behavior of MoE, Mendes et al. \cite{mendes2011convergence} took into account a mixture of polynomial regression experts and studied the maximum likelihood estimator (MLE) of that model. In particular, they established the convergence rate of the Kullback-Leibler (KL) divergence between the estimated density and the true density as the sample size increased, and then offered useful insight into the trade-off between the model complexity and the number of experts based on these results. Next, Ho et al. \cite{ho2022gaussian} and Nguyen et al. \cite{nguyen2023demystifying} continued to investigate the MLE convergence behavior but for Gaussian MoE models with covariate-free and covariate-dependent gating functions, respectively, where the response variable conditioning on the covaritate variable followed from a mixture of Gaussian distributions. Their results indicated that the convergence rates of parameter estimations might be substantially slowed down due to parameter interactions expressed via some partial differential equations (PDEs). Unlike these works, Nguyen et al. \cite{nguyen2024leastsquare} then considered a regression framework where the regression function admitted the form of a softmax gating MoE, and then analyzed the convergence of least squares estimators of parameters and experts. They observed that feed-forward expert networks have much faster estimation rates than polynomial experts.\\

\noindent
Regarding the applications of MoE in deep learning, Chen et al. \cite{chen2022theory} demonstrated that a single expert was incapable of learning a classification problem with intrinsic cluster structures. On the other hand, this problem would be addressed by using a sparsely gated MoE with each expert formulated as a two-layer non-linear convolutional neural networks. Next, MoE was shown to effectively solve the catastrophic forgetting issue in continual learning through overparameterized linear regression tasks \cite{li2025cl,le2024mixture}. More specifically, they illustrated that experts can be trained to specialize in various tasks, while the router can choose appropriate experts for each task. Lastly, MoE and its hierarchical variant with Laplace gating function were also theoretically verified to perform well in multimodal learning where the data included different data modalities, namely text, images, time series, and tabular data in \cite{han2024fusemoe} and \cite{nguyen2024hmoe}, respectively. 

\section{Experiments}
\label{sec:experiments}
In this section, we conduct several experiments to verify: (i) the theoretical results presented in Section~\ref{sec:theory}, (ii) the empirical benefits of quadratic gating over standard linear gating in language modeling, and (iii) the favorable performance of the proposed active-attention mechanism~\eqref{eq:novel_attention} over standard attention mechanism. Furthermore, we also perform an ablation study on the number of heads when using the proposed active-attenion mechanism, which is deferred to Appendix~\ref{appendix:ablation_study} due to the space limit. Meanwhile, all the experimental details are provided in Appendix~\ref{appendix:experimental-details}.

\subsection{Numerical Experiments} We generate synthetic data based on the model described in equation~\eqref{eq:regression_framework}. \\

\noindent
\textbf{Model details.} We now provide the details for the model parameters in model. The variance of Gaussian noise is specified as $\sigma^2 = 0.049$. The true parameters for the gating network, $(\Abm^*_i, \bbm^*_i, c^*_i) \in  \mathbb{R}^{d\times d} \times \mathbb{R}^d \times \mathbb{R}$, are drawn independently of an isotropic Gaussian distribution with zero mean and variance $\sigma_r^2 = 0.01/d$ for $1 \le i \le 7$, and otherwise are set to zero. Similarly, the true parameters of the experts, $(\bm{\beta}_{1i}^*, \bm{\beta}_{0i}^*) \in \mathbb{R}^d \times \mathbb{R}$, are drawn independently of an isotropic Gaussian distribution with zero mean and variance $\sigma_e^2 = 1/d$ for all experts. These parameters remain unchanged for all experiments.
\\

\noindent
\textbf{Training procedure.} For each sample size $n$, spanning from $10^3$ to $10^5$, we perform $20$ experiments. In every experiment, the parameters initialization for the gate's and experts' parameters are adjusted to be near the true parameters, minimizing potential instabilities from the optimization process. Subsequently, we execute gradient descent for $10$ epochs, employing a learning rate of $\eta = 0.1$ to fit a model to the synthetic data. All the numerical experiments are conducted on a MacBook Air equipped with an M1 chip CPU.\\

\noindent
We evaluate the empirical convergence rates of parameter estimation for (1) quadratic polynomial gate and (2) quadratic monomial gate involving linear and $\relu$ experts in an over-specified setting. 
Data for each experiment are produced following equation~\eqref{eq:regression_framework}, based on the true model for each case.
For each experiment, we compute the respective Voronoi losses for each model and present the average values for different sample sizes in Figure~\ref{fig:emp_rates}. 
Error bars representing two standard deviations are also shown. 
Figure~\ref{fig:monomial_plot} investigates the empirical convergence rates of linear and $\relu$ experts within a quadratic monomial gate setting. \\

\noindent
\textbf{Results.} As shown in Figure~\ref{fig:emp_rates}, the empirical convergence rates of the Voronoi losses follow the trends predicted by our theory in Section~\ref{sec:theory}. For the quadratic monomial gate (Figure~\ref{fig:emp_rates}a), the loss for ReLU experts decreases at a rate close to $\mathcal{O}(n^{-0.52})$, while the loss for linear experts shows a much slower rate of around $\mathcal{O}(n^{-0.09})$. For the quadratic polynomial gate (Figure~\ref{fig:emp_rates}b), we again observe a rate of about $\mathcal{O}(n^{-0.49})$ for ReLU experts and a slower rate of $\mathcal{O}(n^{-0.06})$ for linear experts. These results are consistent with our theoretical analysis.

\begin{figure*}[ht]
    \centering
    \begin{subfigure}{.47\textwidth}
        \centering
        \includegraphics[scale = .468]{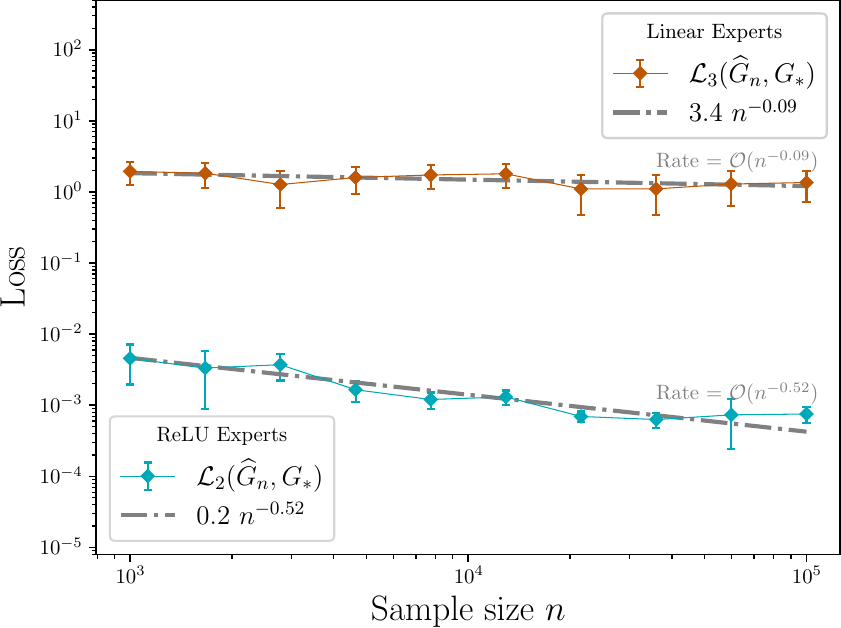}
        \caption{Quadratic monomial gate}
        \label{fig:monomial_plot}
    \end{subfigure}
    \hspace{0.3cm}
    \begin{subfigure}{.47\textwidth}
        \centering
        \includegraphics[scale = .468]{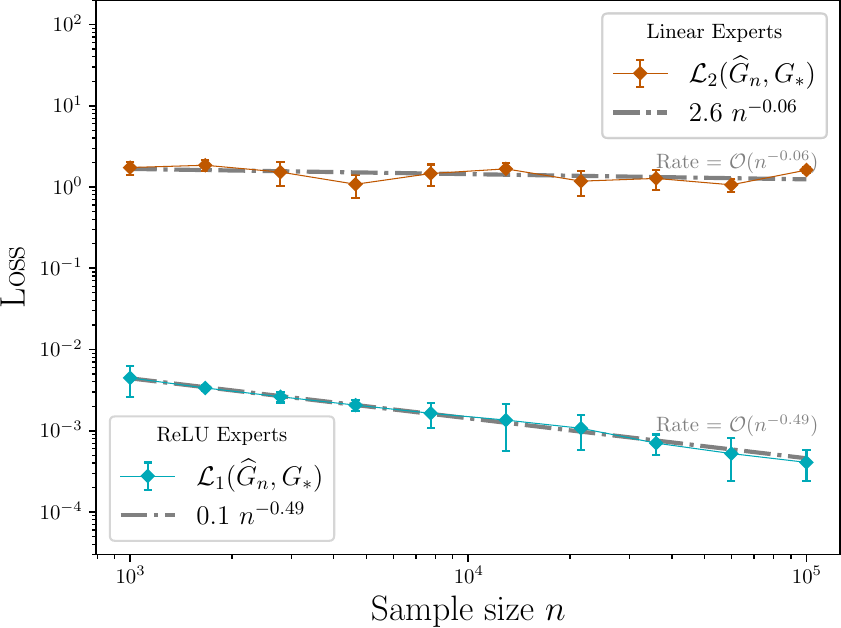}
        \caption{Quadratic polynomial gate}	
        \label{fig:polynomial_plot}
    \end{subfigure}
    \caption{\small Logarithmic plots displaying empirical convergence rates. Subfigures~\ref{fig:monomial_plot} and \ref{fig:polynomial_plot} depict the empirical averages of the corresponding Voronoi losses for the quadratic polynomial and quadratic monomial settings, respectively.
    The orange lines and blue lines respectively depict the Voronoi loss associated with the linear experts and the $\relu$ experts. The gray dash-dotted lines are used to illustrate the fitted lines to indicate the empirical convergence rate.}
    \label{fig:emp_rates}
\end{figure*}

\subsection{Quadratic Gating versus Linear Gating}

To assess the impact of gating function expressiveness on downstream performance, we conduct a set of experiments comparing quadratic and linear gating in GPT2-based MoE models.\\

\noindent
\textbf{Experimental setup.}
We evaluate the impact of gating function design in MoE models using the GPT2 (124M) architecture~\citep{radford2019language}. All MoE models use 8 experts with Top2 routing. We compare three configurations: a standard dense GPT2 model, a GPT2-MoE model with a linear gating network, and a GPT2-MoE model with a quadratic gating network. To ensure fairness, we control for the total number of activated parameters across models by adjusting the hidden sizes of the experts accordingly. All models are trained on a 10 billion token subset of the FineWeb-Edu dataset~\citep{penedo2024fineweb}. 
\\

\noindent
\textbf{Results.}
Table~\ref{tab:quadratic_moe_results} summarizes the performance of each model on validation loss and the HellaSwag~\citep{zellers2019hellaswag} benchmark. The MoE model with a quadratic gating function consistently achieves the best results across both metrics. It shows improved generalization compared to the dense GPT2 baseline and demonstrates more effective expert selection than the linear gating variant. While the linear MoE slightly improves over the dense model in terms of training loss, it underperforms on HellaSwag, suggesting limited transferability. These results highlight the advantage of using a more expressive gating function in downstream tasks.

\begin{table}[ht]
    \caption{Comparison of GPT2-MoE performance (using linear and quadratic gating) against baseline models on the HellaSwag benchmark.}
    \centering
    \begin{tabular}{c|c|c}
        \Xhline{4\arrayrulewidth}
        Model & Val. Loss & HellaSwag (\%)\\
        \hline
        Dense & 3.0211 & 30.74\% \\
        MoE (linear) & 2.9928 & 29.95\% \\
        MoE (quadratic) & \textbf{2.9712} & \textbf{32.11\%} \\
        \Xhline{4\arrayrulewidth}
    \end{tabular}
    \label{tab:quadratic_moe_results}
\end{table}

\subsection{Performance of Active-Attention Mechanism} 

We empirically demonstrate the effects of employing nonlinear activation functions in the proposed active-attention mechanism~\eqref{eq:novel_attention}. Our experiments span large-scale image classification tasks on CIFAR-10 \citep{krizhevsky2009learning} and ImageNet \citep{russakovsky2015imagenet}, language modeling on WikiText-103 \citep{merity2016pointer}, and multivariate time series forecasting across 8 benchmarks. We evaluate the impact of 5 commonly used activation functions, including ReLU \citep{agarap2018deep}, GELU \citep{hendrycks2016gaussian}, SiLU \citep{elfwing2018sigmoid}, Sigmoid, and Tanh \citep{nwankpa2018activation}. 
Figures \ref{fig:image_text} (a) and (b) present the results of image classification using different activation functions in self-attention, with ViT \citep{dosovitskiy2020image} and CaiT \citep{touvron2021going} as the base models. For CIFAR-10, we employed the ViT-Tiny and CaiT-Tiny models, while for ImageNet, we utilized the ViT-Base and CaiT-Medium models. Figure \ref{fig:image_text} (c) displays the results of the language modeling task on WikiText-103, using the standard multi-head self-attention transformer \citep{vaswani2017attention}. Additionally, we tested various activation functions on the Performer model \citep{choromanski2020rethinking} as another backbone. Our findings show that the GELU and ReLU activation functions greatly improve performance compared to linear activation functions. This aligns with prior research, suggesting that these two activation functions are preferred in large-scale deep networks due to their ability to support more efficient and stable training.
\\

\begin{table*}[t]
\caption{\small We further assess the effectiveness of nonlinear activation functions on transformer-based time-series forecasting models across eight forecasting tasks. The results show the averaged mean squared error across five random experiments, with the best results highlighted in \textbf{bold} and the second-best results \underline{underlined}. The results indicate that in most situations, Tanh and Sigmoid functions outperformed other activation functions in these tasks.}
\centering
\renewcommand\arraystretch{1.3}
\resizebox{\textwidth}{!}{%
\begin{tabular}{c|c|c|c|c|c|c|c|c|c} 
\Xhline{4\arrayrulewidth}
\multicolumn{2}{c}{Model $\backslash$ Dataset} & Weather & Traffic & Electricity & Illness & ETTh1 & ETTh2 & ETTm1 & ETTm2 \\ \hline
\multirow{7}{*}{ PatchTST } & \textit{Linear} & \textit{0.197} & \textit{0.383} & \textit{0.152} & \textit{1.474} & \textit{0.414} & \textit{0.338} & \textit{0.331} & \textit{0.220} \\
& GELU & 0.195 & 0.382 & 0.149 & \underline{1.520} & 0.413 & 0.337 & 0.332 & 0.221 \\
& ReLU & 0.196 & \underline{0.380} & 0.150 & 1.551 & 0.413 & 0.336 & 0.331 & 0.218 \\
& Sigmoid & \underline{0.192} & 0.386 & \underline{0.146} & 1.613 & \underline{0.411} & \textbf{0.325} & \underline{0.328} & \underline{0.216} \\
& SiLU & 0.196 & 0.381 & 0.149 & 1.559 & 0.413 & 0.337 & 0.333 & 0.221 \\
& Tanh & \textbf{0.187} & \textbf{0.375} & \textbf{0.141} & \textbf{1.447} & \textbf{0.410} & \underline{0.329} & \textbf{0.325} & \textbf{0.212} \\ \hline
\multirow{7}{*}{ Transformer } & \textit{Linear} & \textit{0.835} & \textit{0.748} & \textit{0.296} & \textit{4.832} & \textit{1.328} & \textit{1.152} & \textit{1.138} & \textit{1.389} \\
& GELU & \underline{0.804} & 0.726 & 0.302 & \textbf{4.129} & \underline{1.269} & 1.134 & 1.134 & \textbf{1.353} \\
& ReLU & 0.839 & 0.735 & \underline{0.272} & 4.224 & 1.314 & \underline{1.102} & \textbf{1.116} & 1.382 \\
& Sigmoid & 0.811 & \textbf{0.714} & 0.278 & 4.972 & 1.285 & \textbf{1.086} & 1.132 & \underline{1.357} \\
& SiLU & 0.823 & 0.756 & 0.293 & 4.535 & 1.334 & 1.157 & 1.153 & 1.379 \\
& Tanh & \textbf{0.797} & \underline{0.721} & \textbf{0.269} & \underline{4.216} & \textbf{1.255} & 1.114 & \underline{1.125} & 1.364 \\ 
\Xhline{4\arrayrulewidth}
\end{tabular}%
}
\label{tab:time_series}
\end{table*}

\noindent
Table \ref{tab:time_series} further evaluates the impact of different activation functions on transformer-based time-series forecasting models across 8 forecasting tasks. In this experiment, we employ the state-of-the-art PatchTST model \citep{Yuqietal-2023-PatchTST} and the standard self-attention transformer as the backbone. Unlike the results observed in Figure \ref{fig:image_text}, in addition to GELU, we find that Tanh and Sigmoid functions also show prominent advantages over linear activation function. We speculate that this is due to the smoothing gradients provided by Tanh and Sigmoid, which may help in capturing subtle patterns in time-series data. Additionally, since these tasks tend to be smaller and more prone to overfitting, the saturation effects of Tanh and Sigmoid could serve as a regularization mechanism by limiting output ranges and avoiding extreme activations.

\begin{figure*}[ht]
  \centering
  \resizebox{\textwidth}{!}{%
    \begin{tabular}{ccc}
      \begin{tabular}{c}
        \includegraphics[width=.33\textwidth]{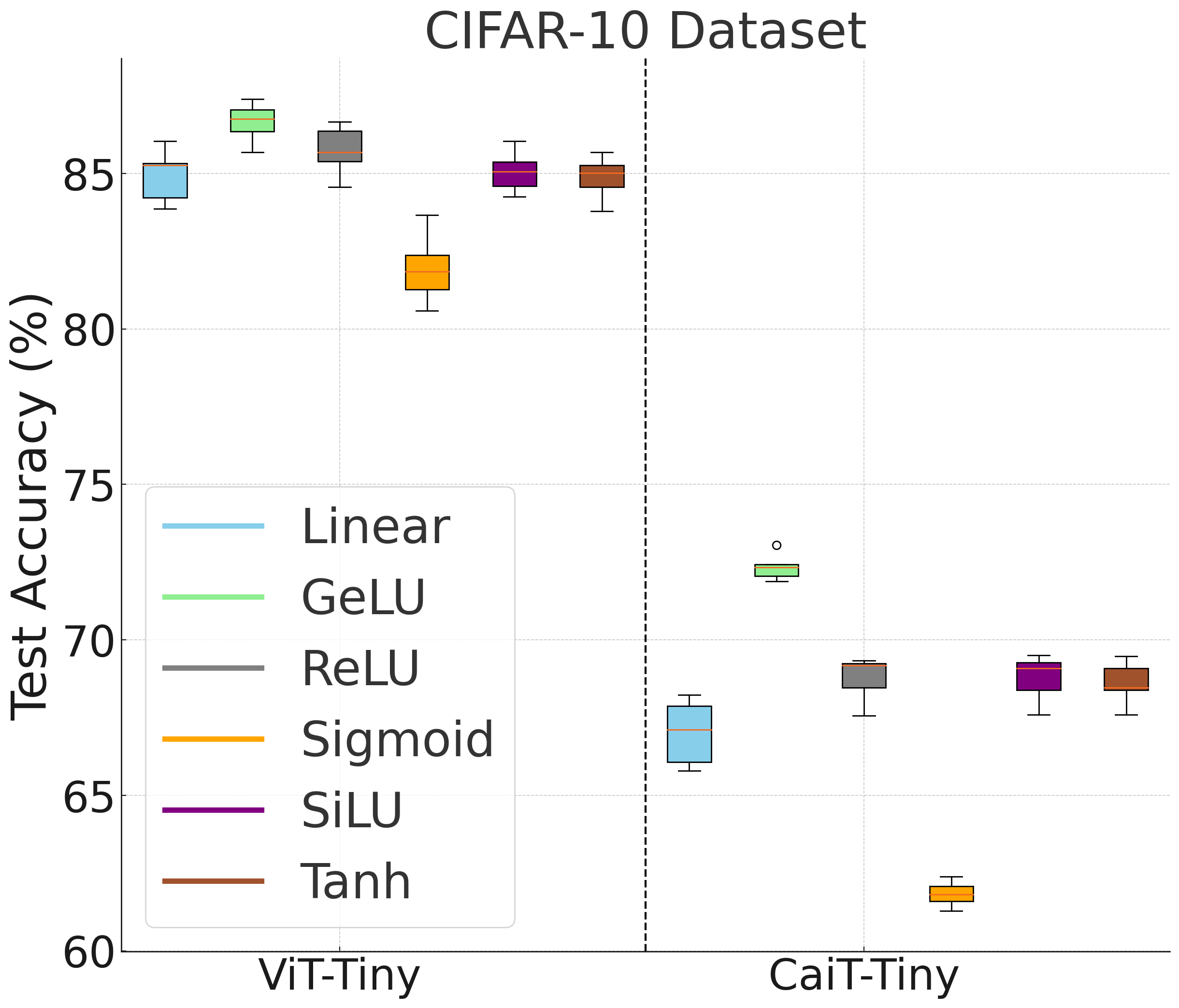}\\
        {\small (a)}
      \end{tabular} &
      \begin{tabular}{c}
        \includegraphics[width=.33\textwidth]{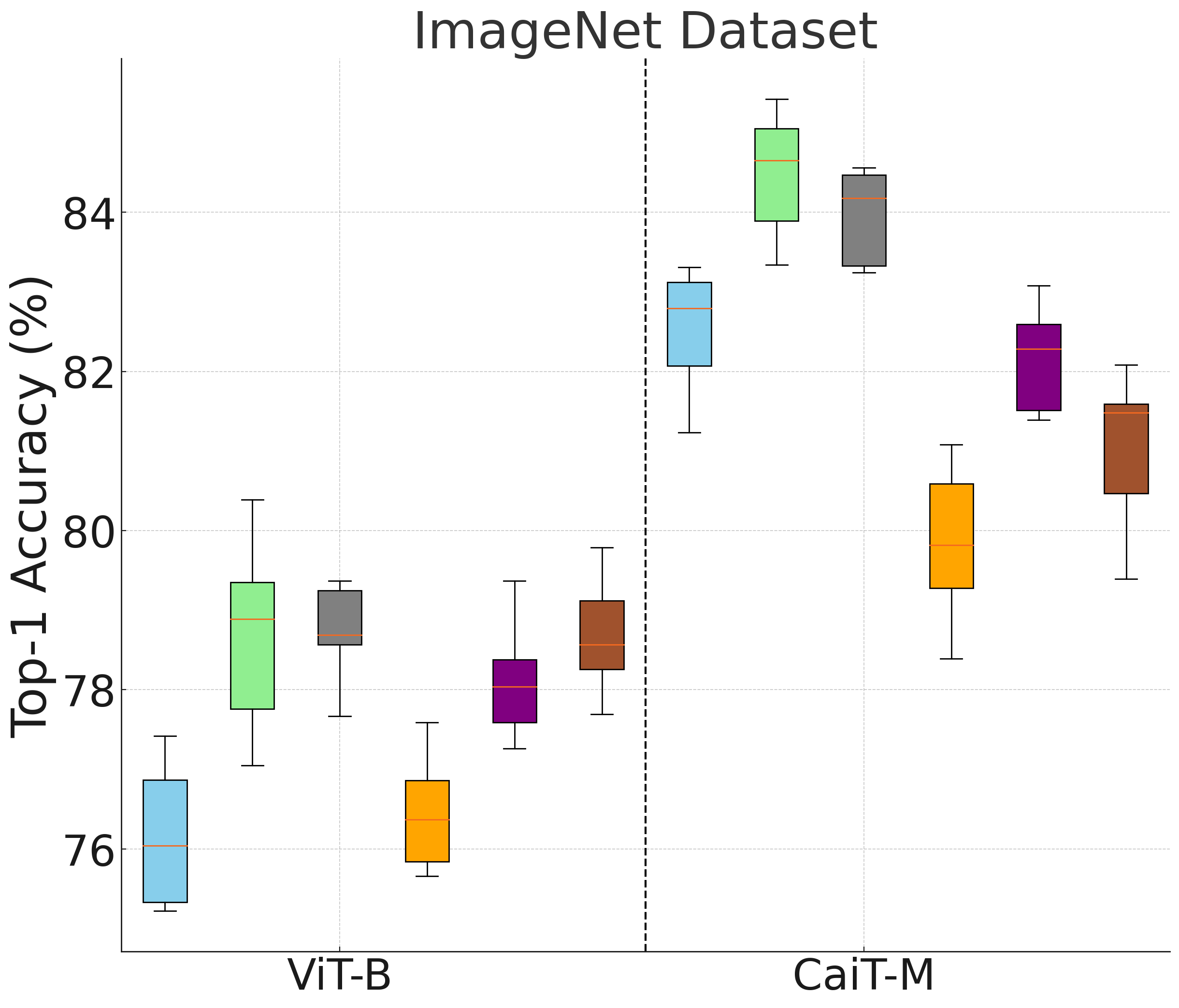}\\
        {\small (b)}
      \end{tabular} &
      \begin{tabular}{c}
        \includegraphics[width=.33\textwidth]{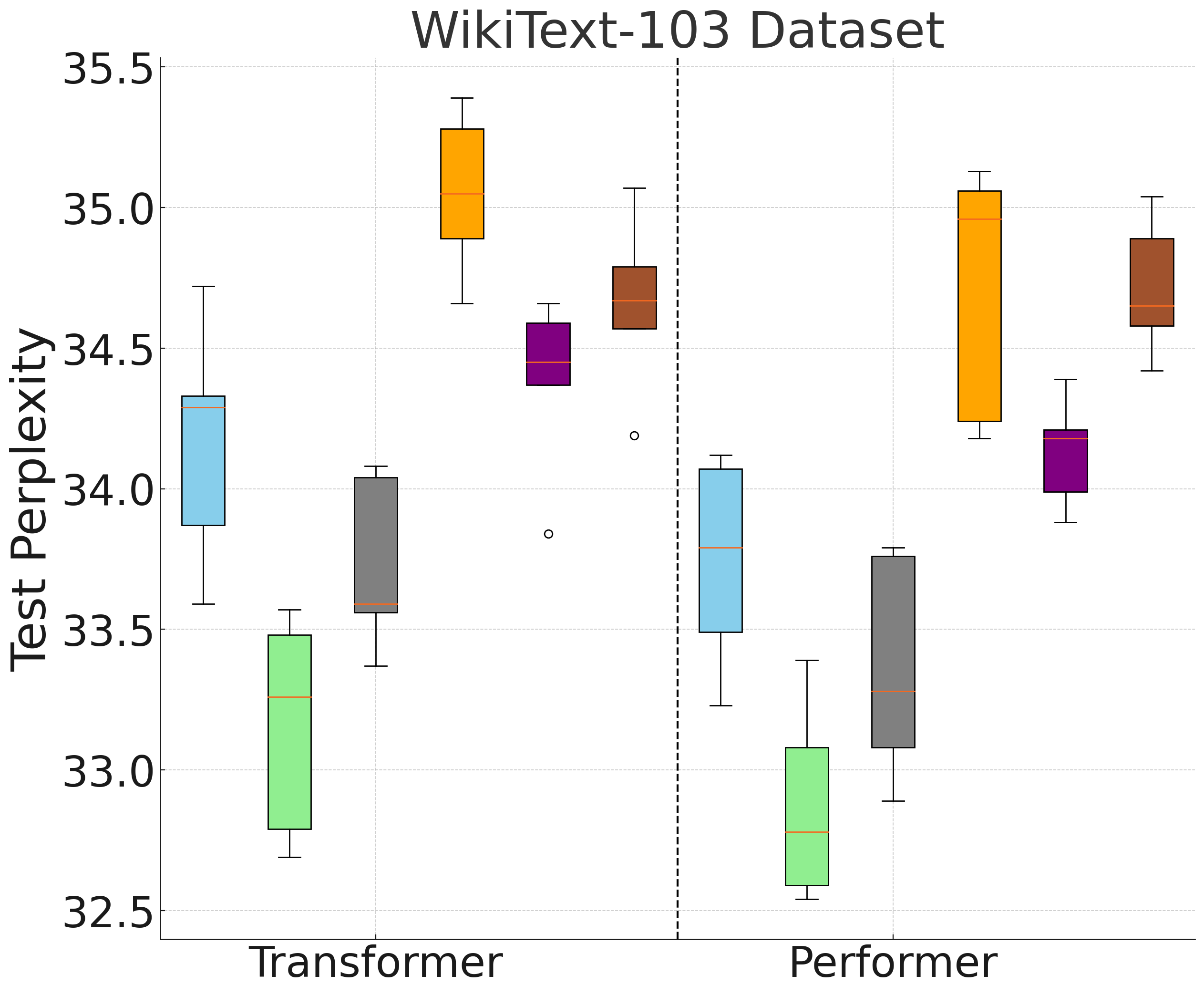}\\
        {\small (c)}
      \end{tabular}
    \end{tabular}%
  }
  \caption{\small We evaluate the effect of employing five different nonlinear activation functions in the self-attention block and compare them with standard linear activation functions on (a) CIFAR-10, (b) ImageNet, and (c) WikiText-103 datasets. All results are averaged across five random experiments. The results demonstrate that using GELU and ReLU activation functions in various transformer backbones noticeably improves performance compared to the linear activation function.}
  \label{fig:image_text}
\end{figure*}

\section{Discussion}
\label{sec:discussion}
In this paper, we first establish a link between the MoE framework and the attention mechanism by showing that each row of a self-attention matrix can be interpreted as a quadratic gating MoE. Inspired by this relation, we carry out the convergence analysis of parameter estimation and expert estimation under MoE models with the quadratic polynomial gate and the quadratic monomial gate, respectively. Our theories indicate that experts formulated as neural networks with non-linear activation functions have significantly faster estimation rates than linear experts. Based on this theoretical insight on MoE and the above connection between MoE and self-attention, we propose a new attention mechanism termed active-attention, where we apply a non-linear activation function to the value matrix. Through extensive real-world experiments, we empirically demonstrate that the proposed active-attention enjoys preferable performance to the standard self-attention in various tasks, including image classification, language modeling, and multivariate time series forecasting. \\

\noindent
\textbf{Limitations.} 
There are some limitations of our work. Firstly, in the convergence analysis of quadratic gating MoE, we take into account a single MoE layer rather than multiple layers as in practice \cite{dai2024deepseekmoe,jiang2024mixtral} for the sake of theory. 
Secondly, as we keep ground-truth parameters in that analysis independent of the sample size, the parameter and expert estimation rates are pointwise but not uniform. However, our results can be naturally extended using the techniques of deriving minimax optimal rates for parameter estimation in MoE from \cite{yan2025contaminated}. 
As these two problems are not the main focus of our work, we leave them for future development.\\

\noindent
\textbf{Future directions.} A potential direction of extending our theories is to investigate quadratic gating MoE in the context of a classification problem as considered in \cite{chen2022theory,nguyen2024general}. Meanwhile, it is of great importance to demystify the expressive power of quadratic gating MoE \cite{wang2025expressivepower}, which has remained elusive in the literature.


\appendix
\begin{center}
\textbf{\Large Supplement to
``Quadratic Gating Mixture of Experts: Statistical Insights into Self-Attention''}
\end{center}

\noindent
In this supplementary material, we first provide rigorous proofs for the theoretical results of Section~\ref{sec:theory} in Appendix~\ref{appendix:proofs}. Next, we study the identifiability of the quadratic gating MoE in Appendix~\ref{appendix:auxiliary_results}. In Appendix~\ref{appendix:experimental-details}, we specify the details of experiments exhibited in Section~\ref{sec:experiments} before discussing the computational overheads of the quadratic gating functions in Appendix~\ref{appendix:overhead}. Lastly, we carry out an empirical ablation study on the number of attention heads when using the proposed active-attention mechanism in Appendix~\ref{appendix:ablation_study}. 

\section{Proof of Theoretical Results}
\label{appendix:proofs}
In this appendix, we present the detailed proofs for the theoretical results introduced in the paper.

\subsection{Proof of Theorem~\ref{theorem:regression_rate}}
\label{appendix:regression_rate}
For the proof of the theorem, we first introduce some notation. Firstly, we denote by $\mathcal{F}_N(\Theta)$ the set of conditional densities of all mixing measures in $\mathcal{M}_N(\Theta)$, that is, $\mathcal{F}_N(\Theta):=\{f_{G}(\xbm):G\in\mathcal{M}_{N}(\Theta)\}$.
Additionally, for each $\delta>0$, the $L_{2}(\mu)$ ball centered around the regression function $f_{G_*}(\xbm)$ and intersected with the set $ \mathcal{F}_N(\Theta)$ is defined as
\begin{align*}   \mathcal{F}_N(\Theta,\delta):=\left\{f \in \mathcal{F}_N(\Theta): \|f -f_{G_*}\|_{L^2(\mu)} \leq\delta\right\}.
\end{align*}
In order to measure the size of the above set, Geer et al. \cite{vandeGeer-00} suggest using the following quantity:
\begin{align}
    \label{eq:bracket_size}
    \mathcal{J}_B(\delta, \mathcal{F}_N(\Theta,\delta)):=\int_{\delta^2/2^{13}}^{\delta}H_B^{1/2}(t, \mathcal{F}_N(\Theta,t),\|\cdot\|_{L^2(\mu)})~\dint t\vee \delta,
\end{align}
where $H_B(t, \mathcal{F}_N(\Theta,t),\|\cdot\|_{L^2(\mu))}$ stands for the bracketing entropy \cite{vandeGeer-00} of $ \mathcal{F}_N(\Theta,u)$ under the $L^{2}$-norm, and $t\vee\delta:=\max\{t,\delta\}$. By using the similar proof argument of Theorem 7.4 and Theorem 9.2 in \cite{vandeGeer-00} with notations being adapted to this work, we obtain the following lemma:
\begin{lemma}
    \label{lemma:density_rate}
    Take $\Psi(\delta)\geq \mathcal{J}_B(\delta, \mathcal{F}_N(\Theta,\delta))$ that satisfies $\Psi(\delta)/\delta^2$ is a non-increasing function of $\delta$. Then, for some universal constant $c$ and for some sequence $(\delta_n)$ such that $\sqrt{n}\delta^2_n\geq c\Psi(\delta_n)$, we achieve that
    \begin{align*}
        \mathbb{P}\Big(\|f_{\widehat{G}_n} - f_{G_*}\|_{L^2(\mu)} > \delta\Big)\leq c \exp\left(-\frac{n\delta^2}{c^2}\right),
    \end{align*}
    for all $\delta\geq \delta_n$.
\end{lemma}

\noindent
We now demonstrate that when the expert functions are Lipschitz continuous, the following bound holds:
\begin{align}    
H_B(\varepsilon,\mathcal{F}_N(\Theta),\|.\|_{L^{2}(\mu)}) \lesssim \log(1/\varepsilon), \label{eq:bracket_entropy_bound}
\end{align}
for any $0 < \varepsilon \leq 1/2$. Indeed, for any function $f_{G} \in \mathcal{F}_N(\Theta)$, since the expert functions are bounded, we obtain that $f_{G}(\xbm) \leq M$ for almost everywhere $\xbm$, where $M>0$ is some bounded constant of the expert functions. Let $\tau\leq\varepsilon$ and $\{\xi_1,\ldots,\xi_k\}$ be the $\tau$-cover under the $L^{\infty}$ norm of the set $\mathcal{F}_N(\Theta)$ where $k:=N(\tau,\mathcal{F}_N(\Theta),\|\cdot\|_{L^{\infty}})$ is the $\tau$-covering number of the metric space $(\mathcal{F}_N(\Theta),\|\cdot\|_{L^{\infty}})$. Then, we construct the brackets of the form $[L_i(\xbm),U_i(\xbm)]$ for all $i\in[k]$ as follows:
    \begin{align*}
        L_i(\xbm)&:=\max\{\xi_i(\xbm)-\tau,0\},\\
        U_i(\xbm)&:=\max\{\xi_i(\xbm)+\tau, M \}.
    \end{align*}
From the above construction, we can validate that $\mathcal{F}_{N}(\Theta)\subset\cup_{i=1}^{k}[L_i(\xbm),U_i(\xbm)]$ and $U_i(\xbm)-L_i(\xbm)\leq \min\{2\tau,M\}$. Therefore, it follows that 
\begin{align*}
    \normf{U_i-L_i}^2=\int(U_i-L_i)^2\dint\mu(\xbm)\leq\int 4\tau^2\dint\mu(\xbm)=4\tau^2,
\end{align*}
which implies that $\normf{U_i-L_i}\leq 2\tau$. By definition of the bracketing entropy, we deduce that
\begin{align}
    \label{eq:bracketing_covering}
    H_B(2\tau,\mathcal{F}_{N}(\Theta),\normf{\cdot})\leq\log k=\log N(\tau,\mathcal{F}_N(\Theta),\|\cdot\|_{L^{\infty}}).
\end{align}
Therefore, we need to provide an upper bound for the covering number $N(\tau,\mathcal{F}_N(\Theta),\|\cdot\|_{L^{\infty}})$. In particular, we denote $\Delta:=\{\Abm,\bbm,c)\in\mathbb{R}^{d\times d}\times\mathbb{R}^d\times \mathbb{R}:(\Abm,\bbm,c,\veta)\in\Theta\}$ and $\Omega:=\{\veta\in\mathbb{R}^q:(\Abm,\bbm,c,\veta)\in\Theta\}$. Since $\Theta$ is a compact set, $\Delta$ and $\Omega$ are also compact. Therefore, we can find $\tau$-covers $\Delta_{\tau}$ and ${\Omega}_{\tau}$ for $\Delta$ and $\Omega$, respectively. We can check that 
\begin{align*}
    |\Delta_{\tau}|\leq \mathcal{O}(\tau^{-(d^2+d+1)N}), \quad |\Omega_{\tau}|\lesssim \mathcal{O}(\tau^{-qN}).
\end{align*}
For each mixing measure $G=\sum_{i=1}^{N}\exp(c_i)\delta_{(\Abm_i,\bbm_i,\veta_i)}\in\mathcal{M}_N(\Theta)$, we consider other two mixing measures:
\begin{align*}
    \check{G}:=\sum_{i=1}^N\exp(c_i)\delta_{(\Abm_i,\bbm_i,\overline{\veta}_i)}, \qquad \overline{G}:=\sum_{i=1}^N\exp(\overline{c}_{i})\delta_{(\overline{\Abm}_{i},\overline{\bbm}_i,\overline{\veta}_i)}.
\end{align*}
Here, $\overline{\veta}_i\in{\Omega}_{\tau}$ such that $\overline{\veta}_i$ is the closest to $\veta_i$ in that set, while $(\overline{\Abm}_{i},\overline{\bbm}_{i},\overline{c}_i)\in\Delta_{\tau}$ is the closest to $(\Abm_i,\bbm_i,c_i)$ in that set. From the above formulations, we get that
\begin{align*}
    \|f_{G}-f_{\check{G}}\|_{L^{\infty}}&=\sup_{\xbm\in\mathcal{X}}\Bigg|\sum_{i=1}^{N}\frac{\exp(\xbm^{\top}\Abm_i\xbm+(\bbm_i)^{\top}\xbm+c_i)}{\sum_{j=1}^{N}\exp(\xbm^{\top}\Abm_j\xbm+(\bbm_j)^{\top}\xbm+c_{j})}\cdot[\mathcal{E}(\xbm,\veta_i)-\mathcal{E}(\xbm,\overline{\veta}_i)]\Bigg|\\
    &\leq\sum_{i=1}^{N} \sup_{\xbm\in\mathcal{X}}\Bigg|\frac{\exp(\xbm^{\top}\Abm_i\xbm+(\bbm_i)^{\top}\xbm+c_i)}{\sum_{j=1}^{N}\exp(\xbm^{\top}\Abm_j\xbm+(\bbm_j)^{\top}\xbm+c_{j})}\cdot[\mathcal{E}(\xbm,\veta_i)-\mathcal{E}(\xbm,\overline{\veta}_i)]\Bigg|\\
    &\leq\sum_{i=1}^{N} \sup_{\xbm\in\mathcal{X}}|\mathcal{E}(\xbm,\veta_i)-\mathcal{E}(\xbm,\overline{\veta}_i)|\\
    &\lesssim\sum_{i=1}^{N}\|\veta_i-\overline{\veta}_i\|\lesssim\tau.
\end{align*}
Here, the first inequality is according to the triangle inequality, the second inequality occurs as the softmax weight is bounded by 1, and the third inequality follows from the fact that the expert $\mathcal{E}(\xbm,\cdot)$ is a Lipschitz function. Next, we have
\begin{align*}
    &\|f_{\check{G}}-f_{\overline{G}}\|_{L^{\infty}}\\
    &\hspace{-0.2cm}=\sup_{\xbm\in\mathcal{X}} \Bigg|\sum_{i=1}^{N}\Bigg[\frac{\exp(\xbm^{\top}\Abm_i \xbm+(\bbm_i)^{\top}\xbm+c_i)}{\sum_{j=1}^{N}\exp(\xbm^{\top}\Abm_j \xbm+(\bbm_j)^{\top}\xbm+c_{j})}-\frac{\exp(\xbm^{\top}\overline{\Abm}_i \xbm+(\overline{\bbm}_i)^{\top}\xbm+\overline{c}_i)}{\sum_{j=1}^{N}\exp(\xbm^{\top}\overline{\Abm}_j \xbm+(\overline{\bbm}_{j})^{\top}\xbm+\overline{c}_{j})}\Bigg] \mathcal{E}(\xbm,\overline{\veta}_i)\Bigg|\\
    &\hspace{-0.2cm}\leq\sum_{i=1}^{N}\sup_{\xbm\in\mathcal{X}} \Bigg|\Bigg[\frac{\exp(\xbm^{\top}\Abm_i \xbm+(\bbm_i)^{\top}\xbm+c_i)}{\sum_{j=1}^{N}\exp(\xbm^{\top}\Abm_j \xbm+(\bbm_j)^{\top}\xbm+c_{j})}-\frac{\exp(\xbm^{\top}\overline{\Abm}_i \xbm+(\overline{\bbm}_i)^{\top}\xbm+\overline{c}_i)}{\sum_{j=1}^{N}\exp(\xbm^{\top}\overline{\Abm}_j \xbm+(\overline{\bbm}_{j})^{\top}\xbm+\overline{c}_{j})}\Bigg]\mathcal{E}(\xbm,\overline{\veta}_i)\Bigg|\\
    &\hspace{-0.2cm}\leq\sum_{i=1}^{N}\sup_{\xbm\in\mathcal{X}} \Bigg|\frac{\exp(\xbm^{\top}\Abm_i \xbm+(\bbm_i)^{\top}\xbm+c_i)}{\sum_{j=1}^{N}\exp(\xbm^{\top}\Abm_j \xbm+(\bbm_j)^{\top}\xbm+c_{j})}-\frac{\exp(\xbm^{\top}\overline{\Abm}_i \xbm+(\overline{\bbm}_i)^{\top}\xbm+\overline{c}_i)}{\sum_{j=1}^{N}\exp(\xbm^{\top}\overline{\Abm}_j \xbm+(\overline{\bbm}_{j})^{\top}\xbm+\overline{c}_{j})}\Bigg|\\
    &\hspace{-0.2cm}\lesssim\sum_{i=1}^{N}\sup_{\xbm\in\mathcal{X}} \Big[\|\Abm_{i}-\overline{\Abm}_{i}\|\cdot\|\xbm\|^2+\|\bbm_i-\overline{\bbm}_i\|\cdot\|\xbm\|+|c_i-\overline{c}_{i}|\Big]\\
    &\hspace{-0.2cm}\leq\sum_{i=1}^{N}(\tau B^2+\tau B+\tau)\lesssim\tau.
\end{align*}
Above, the first inequality is due to the triangle inequality, the second inequality happens as the expert function is bounded, the third inequality follows from the fact that the softmax function is Lipschitz, and the fourth inequality occurs as the input space is bounded, that is, $\|\xbm\|\leq B$ for some constant $B>0$.
According to the triangle inequality, we have
\begin{align*}
    \|f_{G}-f_{\overline{G}}\|_{L^{\infty}}\leq \|f_{G}-f_{\check{G}}\|_{L^{\infty}}+\|f_{\check{G}}-f_{\overline{G}}\|_{L^{\infty}}\lesssim\tau.
\end{align*}
By definition of the covering number, we deduce that
\begin{align}
    \label{eq:covering_bound}
    {N}(\tau,\mathcal{F}_N(\Theta),\|\cdot\|_{L^{\infty}})\leq |\Delta_{\tau}|\times|\Omega_{\tau}|\leq \mathcal{O}_{P}(n^{-(d^2+d+1)N})\times\mathcal{O}(n^{-qN})\leq\mathcal{O}(n^{-(d^2+d+1+q)N}).
\end{align}
Combine equations~\eqref{eq:bracketing_covering} and \eqref{eq:covering_bound}, we achieve that
\begin{align*}
    H_B(2\tau,\mathcal{F}_{N}(\Theta),\normf{\cdot})\lesssim \log(1/\tau).
\end{align*}
Let $\tau=\varepsilon/2$, then we obtain that 
\begin{align*}
    H_B(\varepsilon,\mathcal{F}_N(\Theta),\|.\|_{L^{2}(\mu)}) \lesssim \log(1/\varepsilon).
\end{align*}
As a result, it follows that 
\begin{align}
    \label{eq:bracketing_integral}
    \mathcal{J}_B(\delta, \mathcal{F}_N(\Theta,\delta))= \int_{\delta^2/2^{13}}^{\delta}H_B^{1/2}(t, \mathcal{F}_N(\Theta,t),\normf{\cdot})~\dint t\vee \delta\lesssim \int_{\delta^2/2^{13}}^{\delta}\log(1/t)dt\vee\delta.
\end{align}
Let $\Psi(\delta)=\delta\cdot[\log(1/\delta)]^{1/2}$, then $\Psi(\delta)/\delta^2$ is a non-increasing function of $\delta$. Furthermore, equation~\eqref{eq:bracketing_integral} indicates that $\Psi(\delta)\geq \mathcal{J}_B(\delta,\mathcal{F}_N(\Theta,\delta))$. In addition, let $\delta_n=\sqrt{\log(n)/n}$, then we get that $\sqrt{n}\delta^2_n\geq c\Psi(\delta_n)$ for some universal constant $c$. Finally, by applying Lemma~\ref{lemma:density_rate}, we achieve the desired conclusion of the theorem.

\subsection{Proof of Theorem~\ref{theorem:poly_strongly_identifiable}}
\label{appendix:poly_strongly_identifiable}
In this proof, we aim to establish the following inequality:
\begin{align}
    \label{eq:general_universal_inequality}
    \inf_{G\in\mathcal{M}_{N}(\Theta)}\normf{f_{G}-f_{G_*}}/\mathcal{L}_1(G,G_*)>0.
\end{align}
For that purpose, we divide the proof of the above inequality into local and global parts in the sequel.
\\

\noindent
\textbf{Local part:} In this part, we demonstrate that
\begin{align}
    \label{eq:general_local_inequality}
    \lim_{\varepsilon\to0}\inf_{G\in\mathcal{M}_{N}(\Theta):\mathcal{L}_1(G,G_*)\leq\varepsilon}\normf{f_{G}-f_{G_*}}/\mathcal{L}_1(G,G_*)>0.
\end{align}
Assume by contrary that the above inequality does not hold true, then there exists a sequence of mixing measures $G_n=\sum_{i=1}^{N^*}\exp(\cin)\delta_{(\Abm^n_i, \bbm^n_i,\veta^n_i)}$ in $\mathcal{M}_{N}(\Theta)$ such that $\mathcal{L}_{1n}:=\mathcal{L}_1(G_n,G_*)\to0$ and
\begin{align}
    \label{eq:general_ratio_limit}
    \normf{f_{G_n}-f_{G_*}}/\mathcal{L}_{1n}\to0,
\end{align}
as $n\to\infty$. Let us denote by $\mathcal{V}^n_j:=\mathcal{V}_j(G_n)$ a Voronoi cell of $G_n$ generated by the $j$-th components of $G_*$. Since our arguments are asymptotic, we may assume that those Voronoi cells do not depend on the sample size, i.e. $\mathcal{V}_j=\mathcal{V}^n_j$. Thus, the Voronoi loss $\mathcal{L}_{1n}$ can be represented as
\begin{align}
    \label{eq:loss_n1}
   &\mathcal{L}_{1n}:=\sum_{j:|\mathcal{V}_j|>1}\sum_{i\in\mathcal{V}_j}\exp(\cin)\Big[\|\daijn\|^{\frac{\bar{r}(|\mathcal{V}_j|)}{2}}+\|\dbijn\|^{\bar{r}(|\mathcal{V}_j|)}+\|\deijn\|^2\Big]\nonumber\\
    &+\sum_{j:|\mathcal{V}_j|=1}\sum_{i\in\mathcal{V}_j}\exp(\cin)\Big[\|\daijn\|+\|\dbijn\|+\|\deijn\|\Big]+\sum_{j=1}^{N^*}\Big|\sum_{i\in\mathcal{V}_j}\exp(\cin)-\exp(\cj)\Big|,
\end{align}
where we denote $\daijn:=A^n_{i}-\Abm^*_{j}$, $\dbijn:=\bbm^n_i-\bbm^*_j$ and $\deijn:=\veta^n_i-\veta^*_j$.
\\

\noindent
Since $\mathcal{L}_{1n}\to0$, we get that $(\ain,\bin,\ein)\to(\aj,\bj,\ej)$ and $\sum_{i\in\mathcal{V}_j}\exp(\cin)\to\exp(c^*_j)$ as $n\to\infty$ for any $i\in\mathcal{V}_j$ and $j\in[N^*]$. Now, we divide the proof of local part into three steps as follows:
\\

\noindent
\textbf{Step 1 - Taylor expansion.} In this step, we decompose the term 
\begin{align}
    \label{eq:Qn_formulation}
    Q_n(\xbm):=\Big[\sum_{j=1}^{N^*}\exp(\xbm^{\top}\aj \xbm+(\bj)^{\top}\xbm+\cj)\Big]\cdot[f_{G_n}(\xbm)-f_{G_*}(\xbm)]
\end{align}
into a combination of linearly independent elements using Taylor expansion. In particular, 
we have
\begin{align}
    \label{eq:general_Q_n}
    Q_n(\xbm)&=\sum_{j=1}^{N^*}\sum_{i\in\mathcal{V}_j}\exp(\cin)\Big[\exp(\xbm^{\top}\ain \xbm+(\bin)^{\top}\xbm)\mathcal{E}(\xbm;\ein)-\exp(\xbm^{\top}\aj \xbm+(\bj)^{\top}\xbm)\mathcal{E}(\xbm;\ej)\Big]\nonumber\\
    &-\sum_{j=1}^{N^*}\sum_{i\in\mathcal{V}_j}\exp(\cin)\Big[\exp(\xbm^{\top}\ain \xbm+(\bin)^{\top}\xbm)-\exp(\xbm^{\top}\aj \xbm+(\bj)^{\top}\xbm)\Big]f_{G_n}(\xbm)\nonumber\\
    &+\sum_{j=1}^{N^*}\Big(\sum_{i\in\mathcal{V}_j}\exp(\cin)-\exp(c^*_j)\Big)\Big[\exp(\xbm^{\top}\aj \xbm+(\bj)^{\top}\xbm)\mathcal{E}(\xbm;\ej)\nonumber\\
    &\hspace{6cm}-\exp(\xbm^{\top}\aj \xbm+(\bj)^{\top}\xbm)f_{G_n}(\xbm)\Big]\nonumber\\
    &:=A_n(\xbm)-B_n(\xbm)+C_{n}(\xbm).
\end{align}
\textbf{Decomposition of $A_n(\xbm)$.} Next, we continue to separate the term $A_n(\xbm)$ into two parts as follows:
\begin{align*}
    &A_n(\xbm)\\
    &:=\sum_{j:|\mathcal{V}_j|=1}\sum_{i\in\mathcal{V}_j}\exp(\cin)\Big[\exp(\xbm^{\top}\ain \xbm+(\bin)^{\top}\xbm)\mathcal{E}(\xbm;\ein)-\exp(\xbm^{\top}\aj \xbm+(\bj)^{\top}\xbm)\mathcal{E}(\xbm;\ej)\Big]\\
    &+\sum_{j:|\mathcal{V}_j|>1}\sum_{i\in\mathcal{V}_j}\exp(\cin)\Big[\exp(\xbm^{\top}\ain \xbm+(\bin)^{\top}\xbm)\mathcal{E}(\xbm;\ein)-\exp(\xbm^{\top}\aj \xbm+(\bj)^{\top}\xbm)\mathcal{E}(\xbm;\ej)\Big]\\
    &:=A_{n,1}(\xbm)+A_{n,2}(\xbm).
\end{align*}

\noindent
Let $E(\xbm;A,b):=\exp(\xbm^{\top}A\xbm+b^{\top}\xbm)$. By means of the first-order Taylor expansion, we have
\begin{align*}
    A_{n,1}(\xbm)&=\sum_{j:|\mathcal{V}_j|=1}\sum_{i\in\mathcal{V}_j}\frac{\exp(\cin)}{\alpha!}\sum_{|\alpha|=1}(\daijn)^{\alpha_1}(\dbijn)^{\alpha_2}(\deijn)^{\alpha_3}\\
    &\hspace{5cm}\times\frac{\partial^{|\alpha_1|+|\alpha_2|}E}{\partial A^{\alpha_1}\partial b^{\alpha_2}}(\xbm;\aj,\bj)\frac{\partial^{|\alpha_3|}\mathcal{E}}{\partial\veta^{\alpha_3}}(\xbm;\ej)
    +R_{n,1}(\xbm),
\end{align*}
where $R_{n,1}(\xbm)$ is a Taylor remainder such that $R_{n,1}(\xbm)/\mathcal{L}_{1n}\to0$ as $n\to\infty$. Note that
\begin{align*}
   \frac{\partial^{|\alpha_1|+|\alpha_2|}E}{\partial A^{\alpha_1}\partial b^{\alpha_2}}(\xbm;\aj,\bj)=\frac{\partial^{2|\alpha_1|+|\alpha_2|}E}{\partial b^{\tau(\alpha_1,\alpha_2)}}(\xbm;\aj,\bj),
\end{align*}
where $\tau(\alpha_1,\alpha_2):=\Big(\sum_{u=1}^{d}(\alpha_1^{(uv)}+\alpha_1^{(vu)})+\alpha_2^{(v)}\Big)_{v=1}^{d}=\Big(2\sum_{u=1}^{d}\alpha_1^{(uv)}+\alpha_2^{(v)}\Big)_{v=1}^{d}\in\mathbb{N}^d$. Then, $A_{n,1}(\xbm)$ can be rewritten as
\begin{align*}
    A_{n,1}(\xbm)&=\sum_{j:|\mathcal{V}_j|=1}\sum_{|\alpha_3|=0}^{1}\sum_{|\ell_1|=0\vee 1-|\alpha_3|}^{2(1-|\alpha_3|)}\sum_{i\in\mathcal{V}_j}\sum_{\tau(\alpha_1,\alpha_2)=\ell_1}\frac{\exp(\cin)}{\alpha!}(\daijn)^{\alpha_1}(\dbijn)^{\alpha_2}(\deijn)^{\alpha_3}\\
    &\hspace{5cm}\times\frac{\partial^{2|\alpha_1|+|\alpha_2|}E}{\partial b^{\tau(\alpha_1,\alpha_2)}}(\xbm;\aj,\bj)\frac{\partial^{|\alpha_3|}\mathcal{E}}{\partial\veta^{\alpha_3}}(\xbm;\ej)+R_{n,1}(\xbm)\\
    &=\sum_{j:|\mathcal{V}_j|=1}\sum_{|\alpha_3|=0}^{1}\sum_{|\ell_1|=0\vee 1-|\alpha_3|}^{2(1-|\alpha_3|)}S_{n,j,\alpha_3,\ell_1}\cdot\frac{\partial^{|\ell_1|}E}{\partial b^{\ell_1}}(\xbm;\aj,\bj)\frac{\partial^{|\alpha_3|}\mathcal{E}}{\partial\veta^{\alpha_3}}(\xbm;\ej)+R_{n,1}(\xbm),
\end{align*}
where we denote $$S_{n,j,\alpha_3,\ell_1}:=\sum_{i\in\mathcal{V}_j}\sum_{\tau(\alpha_1,\alpha_2)=\ell_1}\frac{\exp(\cin)}{\alpha!}(\daijn)^{\alpha_1}(\dbijn)^{\alpha_2}(\deijn)^{\alpha_3}$$ for any $j\in[N^*]$ and $(\alpha_3,\ell_1)\neq(\zerod,\zerod)$.
\\

\noindent
Analogously, by applying the Taylor expansion of order $\brj:=\bar{r}(|\mathcal{V}_j|)$, we can represent the term $A_{n,2}(\xbm)$ as
\begin{align*}
    A_{n,2}(\xbm)&=\sum_{j:|\mathcal{V}_j|>1}\sum_{|\alpha_3|=0}^{\brj}\sum_{|\ell_1|=0\vee 1-|\alpha_3|}^{2(\brj-|\alpha_3|)}S_{n,j,\alpha_3,\ell_1}\cdot\frac{\partial^{|\ell_1|}E}{\partial b^{\ell_1}}(\xbm;\aj,\bj)\frac{\partial^{|\alpha_3|}\mathcal{E}}{\partial\veta^{\alpha_3}}(\xbm;\ej)+R_{n,2}(\xbm),
\end{align*}
where $R_{n,2}(\xbm)$ is a Taylor remainder such that $R_{n,2}(\xbm)/\mathcal{L}_{1n}\to0$ as $n\to\infty$.
\\

\noindent
\textbf{Decomposition of $B_n(\xbm)$.} Note that $B_n(\xbm)$ can be rewritten as
\begin{align*}
    B_n(\xbm)&=\sum_{j:|\mathcal{V}_j|=1}\sum_{i\in\mathcal{V}_j}\exp(\cin)\Big[E(\xbm;\ain,\bin)-E(\xbm;\aj,\bj)\Big]f_{G_n}(\xbm)\\
    &+\sum_{j:|\mathcal{V}_j|>1}\sum_{i\in\mathcal{V}_j}\exp(\cin)\Big[E(\xbm;\ain,\bin)-E(\xbm;\aj,\bj)\Big]f_{G_n}(\xbm)\\
    &:=B_{n,1}(\xbm) + B_{n,2}(\xbm).
\end{align*}
By reusing the above techniques, we can decompose $B_{n,1}(\xbm)$ as
\begin{align*}
    B_{n,1}(\xbm)&=\sum_{j:|\mathcal{V}_j|=1}\sum_{i\in\mathcal{V}_j}\frac{\exp(\cin)}{\alpha!}\sum_{|\alpha|=1}(\daijn)^{\alpha_1}(\dbijn)^{\alpha_2}\cdot\frac{\partial^{2|\alpha_1|+|\alpha_2|}E}{\partial b^{\tau(\alpha_1,\alpha_2)}}(\xbm;\aj,\bj)f_{G_n}(\xbm)\\
    &\hspace{11cm}+R_{n,3}(\xbm)\\
    &=\sum_{j:|\mathcal{V}_j|=1}\sum_{|\ell_2|=1}^{2}T_{n,j,\ell_2}\cdot\frac{\partial^{|\ell_2|}E}{\partial b^{\ell_2}}(\xbm;\aj,\bj)f_{G_n}(\xbm)+R_{n,3}(\xbm),
\end{align*}
where we denote $$T_{n,j,\ell_2}:=\sum_{i\in\mathcal{V}_j}\sum_{\tau(\alpha_1,\alpha_2)=\ell_2}\frac{\exp(\cin)}{\alpha!}(\daijn)^{\alpha_1}(\dbijn)^{\alpha_2}$$ for any $j\in[N^*]$ and $\ell_2\neq\zerod$. Meanwhile, $R_{n,3}(\xbm)$ is a Taylor remainder such that $R_{n,3}(\xbm)/\mathcal{L}_{1n}\to0$ as $n\to\infty$. Similarly, we also have that
\begin{align*}
    B_{n,2}(\xbm)=\sum_{j:|\mathcal{V}_j|>1}\sum_{|\ell_2|=1}^{2\brj}T_{n,j,\ell_2}\cdot\frac{\partial^{|\ell_2|}E}{\partial b^{\ell_2}}(\xbm;\aj,\bj)f_{G_n}(\xbm)+R_{n,4}(\xbm),
\end{align*}
where $R_{n,4}(\xbm)$ is a Taylor remainder such that $R_{n,4}(\xbm)/\mathcal{L}_{1n}\to0$ as $n\to\infty$.
\\

\noindent
Putting the above results together, we can decompose the term $Q_n(\xbm)$ as
\begin{align}
    \label{eq:Qn_decomposition}
    Q_n(\xbm)&=\sum_{j=1}^{N^*}\sum_{|\alpha_3|=0}^{\brj}\sum_{|\ell_1|=0}^{2(\brj-|\alpha_3|)}S_{n,j,\alpha_3,\ell_1}\cdot\frac{\partial^{|\ell_1|}E}{\partial b^{\ell_1}}(\xbm;\aj,\bj)\frac{\partial^{|\alpha_3|}\mathcal{E}}{\partial\veta^{\alpha_3}}(\xbm;\ej)\nonumber\\
    &-\sum_{j=1}^{N^*}\sum_{|\ell_2|=0}^{2\brj}T_{n,j,\ell_2}\cdot\frac{\partial^{|\ell_2|}E}{\partial b^{\ell_2}}(\xbm;\aj,\bj)f_{G_n}(\xbm)+\sum_{i=1}^{4}R_{n,i}(\xbm),
\end{align}
where we define $S_{n,j,\zeroq,\zerod}=T_{n,j,\zerod}=\sum_{i\in\mathcal{V}_j}\exp(\cin)-\exp(\cj)$ for any $j\in[N^*]$.
\\

\noindent
\textbf{Step 2 - Non-vanishing coefficients.} In this step, we prove by contradiction that at least one among ratios of the forms $S_{n,j,\alpha_3,\ell_1}/\mathcal{L}_{1n}$ and $T_{n,j,\ell_2}/\mathcal{L}_{1n}$ goes to zero as $n$ tends to infinity. Assume that
\begin{align*}
    \frac{S_{n,j,\alpha_3,\ell_1}}{\mathcal{L}_{1n}}\to0, \qquad \frac{T_{n,j,\ell_2}}{\mathcal{L}_{1n}}\to0,
\end{align*}
for any $j\in[N^*]$, $0\leq|\alpha_3|\leq \brj$, $0\leq|\ell_1|\leq2(\brj-|\alpha_3|)$ and $0\leq|\ell_2|\leq 2\brj$.
\\

\noindent
First of all, it is worth noting that as $n\to\infty$,
\begin{align}
    \label{eq:weight_limit}
    \frac{1}{\mathcal{L}_{1n}}\sum_{j=1}^{N^*}\Big|\sum_{i\in\mathcal{V}_j}\exp(\cin)-\exp(\cj)\Big|=\sum_{j=1}^{N^*}\Big|\frac{S_{n,j,\zeroq,\zerod}}{\mathcal{L}_{1n}}\Big|\to0.
\end{align}
Now, let us consider indices $j\in[N^*]$ such that its corresponding Voronoi cell has only one element, i.e. $|\mathcal{V}_j|=1$.
\begin{itemize}
    \item When $\alpha_3=e_{q,u}:=(0,\ldots,0,\underbrace{1}_{\textit{u-th}},0,\ldots,0)\in\mathbb{N}^{q}$ and $\ell_1=\zerod$, we have 
    \begin{align*}
        \frac{1}{\mathcal{L}_{1n}}\cdot\sum_{i\in\mathcal{V}_j}\exp(\cin)|(\deijn)^{(u)}|=|S_{n,j,\alpha_3,\ell_1}|/\mathcal{L}_{1n}\to0 \quad\text{as}\quad n\to\infty.
    \end{align*}
    By taking the summation of the previous term with $u\in[q]$, we achieve that 
    \begin{align*}
        \frac{1}{\mathcal{L}_{1n}}\sum_{i\in\mathcal{V}_j}\exp(\cin)\|\deijn\|_1\to0.
    \end{align*}
    Owing to the topological equivalence between norm-1 and norm-2, it follows that
    \begin{align}
        \label{eq:eta_limit_1}
        \frac{1}{\mathcal{L}_{1n}}\sum_{i\in\mathcal{V}_j}\exp(\cin)\|\deijn\|\to0.
    \end{align}
    \item When $\alpha_3=\zeroq$ and $\ell_1=e_{d,u}:=(0,\ldots,0,\underbrace{1}_{\textit{u-th}},0,\ldots,0)\in\mathbb{N}^{d}$, by using the above arguments, we get that
    \begin{align}
        \label{eq:b_limit_1}
        \frac{1}{\mathcal{L}_{1n}}\sum_{i\in\mathcal{V}_j}\exp(\cin)\|\dbijn\|\to0.
    \end{align}
    \item When $\alpha_3=\mathbf{0}_q$ and $\ell_1=2e_{d,u}$, it follows that
    \begin{align}
        \label{eq:a_limit_1}
        \frac{1}{\mathcal{L}_{1n}}\sum_{i\in\mathcal{V}_j}\exp(\cin)\|\daijn\|\to0.
    \end{align}
\end{itemize}
Combine the limits in equations~\eqref{eq:eta_limit_1}, \eqref{eq:b_limit_1} and \eqref{eq:a_limit_1}, we obtain that 
\begin{align}
    \label{eq:order_1_limit}
    \frac{1}{\mathcal{L}_{1n}}\sum_{j:|\mathcal{V}_j|=1}\sum_{i\in\mathcal{V}_j}\exp(\cin)[\|\daijn\|+\|\dbijn\|+\|\deijn\|]\to0,
\end{align}
as $n\to\infty$. 
\\

\noindent
Next, we consider indices $j\in[N^*]$ such that its corresponding Voronoi cell has more than one element, i.e. $|\mathcal{V}_j|>1$. When $\alpha_3=2e_{q,u}$ and $\ell_1=\zerod$, we get $\frac{1}{\mathcal{L}_{1n}}\cdot\sum_{i\in\mathcal{V}_j}\exp(\cin)|(\deijn)^{(u)}|^2=|2S_{n,j,\alpha_3,\ell_1}|/\mathcal{L}_{1n}\to0$ as $n\to\infty$. By taking the summation of the previous term with $u\in[q]$, we achieve that $\frac{1}{\mathcal{L}_{1n}}\sum_{i\in\mathcal{V}_j}\exp(\cin)\|\deijn\|^2\to0$. This result indicates that
    \begin{align}
        \label{eq:eta_limit_2}
        \frac{1}{\mathcal{L}_{1n}}\sum_{j:|\mathcal{V}_j|>1}\sum_{i\in\mathcal{V}_j}\exp(\cin)\|\deijn\|^2\to0,
    \end{align}
    as $n\to\infty$. It follows from the limits in equations~\eqref{eq:weight_limit}, \eqref{eq:order_1_limit}, \eqref{eq:eta_limit_2} and the formulation of $\mathcal{L}_{1n}$ in equation~\eqref{eq:loss_n1} that 
    \begin{align*}
        \frac{1}{\mathcal{L}_{1n}}\sum_{j:|\mathcal{V}_j|>1}\sum_{i\in\mathcal{V}_j}\exp(\cin)[\|\daijn\|^{\brj/2}+\|\dbijn\|^{\brj}]\to1.
    \end{align*}
    The above limit suggests that there exists an index $j^{\prime}:|\mathcal{V}_{j^{\prime}}|>1$ such that 
    \begin{align}
        \label{eq:non_zero_limit}
         \frac{1}{\mathcal{L}_{1n}}\sum_{i\in\mathcal{V}_{j^{\prime}}}\exp(\cin)[\|\Delta A^n_{ij^{\prime}}\|^{\bar{r}_{j^{\prime}}/2}+\|\Delta b^n_{ij^{\prime}}\|^{\bar{r}_{j^{\prime}}}]\not\to0.
    \end{align}
    Without loss of generality, we may assume that $j^{\prime}=1$.
    \\

\noindent    
    \textbf{Case 1.} $\frac{1}{\mathcal{L}_{1n}}\sum_{i\in\mathcal{V}_{1}}\exp(\cin)[\|((\daione)^{(uu)})_{u=1}^{d}\|^{\brone/2}+\|\dbione\|^{\brone}]\not\to0$.
    \\

    In this case, there exists some $u^{\prime}\in[d]$ such that
    \begin{align*}
        \frac{1}{\mathcal{L}_{1n}}\sum_{i\in\mathcal{V}_{1}}\exp(\cin)[|(\daione)^{(u^{\prime}u^{\prime})}|^{\brone/2}+|(\dbione)^{(u^{\prime})}|^{\brone}]\not\to0.
    \end{align*}
    Again, we may assume WLOG that $u^{\prime}=1$ throughout Case 1, i.e.
    \begin{align}
        \label{eq:non_vanishing_term}
         \frac{1}{\mathcal{L}_{1n}}\sum_{i\in\mathcal{V}_{1}}\exp(\cin)[|(\daione)^{(11)}|^{\brone/2}+|(\dbione)^{(1)}|^{\brone}]\not\to0.
    \end{align}
    Next, let us consider the term 
    \begin{align}
        \label{eq:coefficient_0_ell}
        S_{n,1,\zeroq,\ell_1}=\sum_{i\in\mathcal{V}_1}\sum_{\tau(\alpha_1,\alpha_2)=\ell_1}\frac{\exp(\cin)}{\alpha_1!~\alpha_2!}(\daijn)^{\alpha_1}(\dbijn)^{\alpha_2},
    \end{align}
    where $\ell_1\in\mathcal{N}^d$ such that $\ell_1^{(u)}=0$ for any $u=2,3,\ldots,d$. Then, the constraint $\tau(\alpha_1,\alpha_2)=\ell_1$ holds iff $\alpha_1^{(u1)}=\alpha_1^{(1v)}=\alpha_1^{(uv)}=\alpha_2^{(u)}$ for all $u,v=2,3,\ldots,d$. Thus, by assumption, we get
    \begin{align}
        \label{eq:zero_limit}
        \frac{1}{\mathcal{L}_{1n}}\sum_{i\in\mathcal{V}_1}\sum_{2\alpha_1^{(11)}+\alpha_2^{(1)}=\ell_1^{(1)}}\frac{\exp(\cin)}{\alpha_1^{(11)}!~\alpha_2^{(1)}!}(\daijn)^{\alpha_1^{(11)}}(\dbijn)^{\alpha_2^{(1)}}=\frac{S_{n,1,\zeroq,\ell_1}}{\mathcal{L}_{1n}}\to0.
    \end{align}
        

    \noindent
    By dividing the left hand side of equation~\eqref{eq:zero_limit} by that of equation~\eqref{eq:non_vanishing_term}, we get
    \begin{align}
        \label{eq:fraction_zero}
        \dfrac{\sum_{i\in\mathcal{V}_1}\sum_{2\alpha_1^{(11)}+\alpha_2^{(1)}=\ell_1^{(1)}}\frac{\exp(\cin)}{\alpha_1^{(11)}!~\alpha_2^{(1)}!}(\daijn)^{\alpha_1^{(11)}}(\dbijn)^{\alpha_2^{(1)}}}{\sum_{i\in\mathcal{V}_{1}}\exp(\cin)[|(\daione)^{(11)}|^{\brone/2}+|(\dbione)^{(1)}|^{\brone}]}\to0.
    \end{align}
Subsequently, we define 
\begin{align*}
    {M}_n&:=\max\{|(\daione)^{(11)}|^{1/2},|(\dbione)^{(1)}|:i\in\mathcal{V}_1\},\\
    {\pi}_n&:=\max_{i\in\mathcal{V}_1}\exp(\cin).
\end{align*}
For any $i\in\mathcal{V}_1$, it is clear that the sequence of positive real numbers $(\exp(\cin)/{\pi}_n)$ is bounded, therefore, we can replace it by its subsequence that admits a non-negative limit denoted by $p_i^2=\lim_{n\to\infty}\exp(\cin)/{\pi}_n$. In addition, let us denote $(\daione)^{(11)}/{M}_n^2\to \gamma_{1i}$ and $(\dbione)^{(1)}/{M}_n\to \gamma_{2i}$. Since $\exp(\cin)\geq \beta$ for some $\beta>0$, the real numbers $p_i$ will not vanish, and at least one of them is equal to $1$. Analogously, at least one of the terms $\gamma_{1i}$ and $\gamma_{2i}$ is equal to either $1$ or $-1$. 
\\

\noindent
Note that $\sum_{i\in\mathcal{V}_1}\exp(\cin)\Big(|(\daione)^{(11)}|^{\brone/2}+|(\dbione)^{(1)}|^{\brone}\Big)/({\pi}_n{M}_n^{\ell_1^{(1)}})\not\to 0$ for all $\ell_1^{(1)}\in[\brone]$. Thus, we are able to divide both the numerator and the denominator in equation~\eqref{eq:fraction_zero} by ${\pi}_n{M}_n^{\ell_1^{(1)}}$ and let $n\to\infty$ in order to achieve the following system of polynomial equations:
\begin{align*}
    \sum_{i\in\mathcal{V}_1}\sum_{2\alpha_1^{(11)}+\alpha_2^{(1)}=\ell_1^{(1)}}\dfrac{p_i^2\gamma_{1i}^{\alpha_1^{(11)}}\gamma_{2i}^{\alpha_2^{(1)}}}{\alpha_1^{(11)}!~\alpha_2^{(1)}!}=0, \quad   \ell_1^{(1)}\in[\brone].
\end{align*}
However, by the definition of $\brone$, the above system cannot admit any non-trivial solutions, which is a contradiction. Thus, Case 1 cannot happen.
\\

\noindent
\textbf{Case 2.} $\frac{1}{\mathcal{L}_{1n}}\sum_{i\in\mathcal{V}_{1}}\exp(\cin)\|((\daione)^{(uv)})_{1\leq u\neq v\leq d}\|^{\brone/2}\not\to0$.
\\

\noindent
In this case, there exist some indices $u^{\prime}, v^{\prime}$ such that $u^{\prime}\neq v^{\prime}$ and 
\begin{align*}
    \dfrac{1}{\mathcal{L}_{1n}}\cdot\sum_{i\in\mathcal{V}_1}\exp(\cin)|(\daione)^{(u^{\prime}v^{\prime})}|^{\brone/2}\not\to 0.
\end{align*}
Recall that $|\mathcal{V}_1|>1$, or equivalently, $|\mathcal{V}_1|\geq 2$, we have that $\brone\geq 4$. Therefore, the above equation leads to
\begin{align}
     \label{eq:gamma_not_vanish}
    \dfrac{1}{\mathcal{L}_{1n}}\cdot\sum_{i\in\mathcal{V}_1}\exp(\cin)|(\daione)^{(u^{\prime}v^{\prime})}|^{2}\not\to 0.
\end{align}
WLOG, we assume that $u^{\prime}=1$ and $v^{\prime}=2$ throughout Case 2. We continue to consider the coefficient $S_{n,1,\zeroq,\ell_1}$ in equation~\eqref{eq:coefficient_0_ell} with $\ell_1=(2,2,0,\ldots,0)\in\mathbb{N}^d$. By assumption, we have ${S_{n,1,\zeroq,\ell_1}}/{\mathcal{L}_{1n}}\to 0$, which together with equation~\eqref{eq:gamma_not_vanish} imply that
\begin{align}
    \label{eq:case_12_ratio_vanish}
   \dfrac{\sum_{i\in\mathcal{V}_1}\sum_{ \tau(\alpha_1,\alpha_2)=\ell_1}\dfrac{\exp(\cin)}{\alpha_1!\alpha_2!}(\daione)^{\alpha_1}(\dbione)^{\alpha_2}}{\sum_{i\in\mathcal{V}_1}\exp(\cin)|(\daione)^{(12)}|^{2}}\to 0.
\end{align}
Similarly, by combining the fact that case 1.1 does not hold and the result in equation~\eqref{eq:gamma_not_vanish}, we get
\begin{align*}
    \dfrac{\sum_{i\in\mathcal{V}_1}\exp(\cin)\Big(\norm{((\daione)^{(uu)})_{u=1}^d}^{\brone/2}+\norm{\dbione}^{\brone}\Big)}{\sum_{i\in\mathcal{V}_1}\exp(\cin)|(\daione)^{(12)}|^{2}}\to 0.
\end{align*}
Since $\brone\geq 4$, the above limit indicates that any terms in equation~\eqref{eq:case_12_ratio_vanish} with $\alpha_1^{(uu)}>0$ and $\alpha_2^{(u)}>0$ for $u\in\{1,2\}$ will vanish. Consequently, we deduce from equation~\eqref{eq:case_12_ratio_vanish} that
\begin{align*}
    1=\dfrac{\sum_{i\in\mathcal{V}_1}\exp(\cin)|(\daione)^{(12)}|^{2}}{\sum_{i\in\mathcal{V}_1}\exp(\cin)|(\daione)^{(12)}|^{2}}\to 0,
\end{align*}
which is a contradiction. Thus, Case 2 cannot happen.\\

\noindent
Collect the results from Case 1 and Case 2, we can conclude that the claim in equation~\eqref{eq:non_zero_limit}, which is a contradiction. Therefore, at least one among ratios of the forms $S_{n,j,\alpha_3,\ell_1}/\mathcal{L}_{1n}$ and $T_{n,j,\ell_2}/\mathcal{L}_{1n}$ goes to zero as $n\to\infty$.\\

\noindent
\textbf{Step 3. Application of Fatou's lemma.} In this step, we show that all the ratios $S_{n,j,\alpha_3,\ell_1}/\mathcal{L}_{1n}$ and $T_{n,j,\ell_2}/\mathcal{L}_{1n}$ go to zero as $n\to\infty$, which contradicts to the conclusion in Step 2. In particular, by denoting $m_n$ as the maximum of the absolute values of those ratios. From the result of Step 2, it follows that $1/m_n\not\to\infty$. \\

\noindent
Recall from the hypothesis in equation~\eqref{eq:general_ratio_limit} that $\normf{f_{G_n}-f_{G_*}}/\mathcal{L}_{1n}\to0$ as $n\to\infty$, which indicates that $\|f_{G_n}-f_{G_*}\|_{L^1(\mu)}/\mathcal{L}_{1n}\to0$. Therefore, by applying the Fatou's lemma, we get that
\begin{align*}
    0=\lim_{n\to\infty}\frac{\|f_{G_n}-f_{G_*}\|_{L^1(\mu)}}{m_n\mathcal{L}_{1n}}\geq \int \liminf_{n\to\infty}\frac{|f_{G_n}(\xbm)-f_{G_*}(\xbm)|}{m_n\mathcal{L}_{1n}}\dint\mu(\xbm)\geq 0.
\end{align*}
This result implies that $\frac{1}{m_n\mathcal{L}_{1n}}\cdot[f_{G_n}(\xbm)-f_{G_*}(\xbm)]\to0$ as $n\to\infty$ for $\mu$-almost surely $x$. Looking at the formulation of $Q_n(\xbm)$ in equation~\eqref{eq:Qn_formulation}, since the term $\Big[\sum_{j=1}^{N^*}\exp(\xbm^{\top}\aj \xbm+(\bj)^{\top}\xbm+\cj)\Big]$ is bounded, we deduce that the term $\frac{1}{m_n\mathcal{L}_{1n}}\cdot Q_n(\xbm)\to0$ for $\mu$-almost surely $x$.
\\

\noindent
Let us denote
\begin{align*}
    \frac{S_{n,j,\alpha_3,\ell_1}}{m_n\mathcal{L}_{1n}}\to \phi_{j,\alpha_3,\ell_1}, \qquad \frac{T_{n,j,\ell_2}}{m_n\mathcal{L}_{1n}}\to\varphi_{j,\ell_2},
\end{align*}
with a note that at least one among them is non-zero. Then, from the decomposition of $Q_n(\xbm)$ in equation~\eqref{eq:Qn_decomposition}, we have
\begin{align*}
    \sum_{j=1}^{N^*}\sum_{|\alpha_3|=0}^{\brj}\sum_{|\ell_1|=0}^{2(\brj-|\alpha_3|)}\phi_{j,\alpha_3,\ell_1}\cdot&\frac{\partial^{|\ell_1|}E}{\partial b^{\ell_1}}(\xbm;\aj,\bj)\frac{\partial^{|\alpha_3|}\mathcal{E}}{\partial\veta^{\alpha_3}}(\xbm;\ej)\nonumber\\
    &-\sum_{j=1}^{N^*}\sum_{|\ell_2|=0}^{2\brj}\varphi_{j,\ell_2}\cdot\frac{\partial^{|\ell_2|}E}{\partial b^{\ell_2}}(\xbm;\aj,\bj)f_{G_*}(\xbm)=0,
\end{align*}
for $\mu$-almost surely $x$. Since the expert function $h$ satisifes the condition in Definition~\ref{def:strong_identifiability}, we obtain that $\phi_{j,\alpha_3,\ell_1}=\varphi_{j,\ell_2}=0$ for all $j\in[N^*]$, $0\leq|\alpha_3|\leq \brj$, $0\leq|\ell_1|\leq 2(\brj-|\alpha_3|)$ and $0\leq|\ell_2|\leq 2\brj$. This result turns out to contradict the fact that at least one among them is different from zero. Hence, we achieve the inequality in equation~\eqref{eq:general_local_inequality}.
\\

\noindent
\textbf{Global part.} It is worth noting that the inequality~\eqref{eq:general_local_inequality} suggests that there exists a positive constant $\varepsilon'$ such that
\begin{align*}
    \inf_{G\in\mathcal{M}_{N}(\Theta):\mathcal{L}_1(G,G_*)\leq\varepsilon'}\normf{f_{G}-f_{G_*}}/\mathcal{L}_1(G,G_*)>0.
\end{align*}
Therefore, it is sufficient to prove that
\begin{align}
    \label{eq:general_global_inequality}
    \inf_{G\in\mathcal{M}_{N}(\Theta):\mathcal{L}_1(G,G_*)>\varepsilon'}\normf{f_{G}-f_{G_*}}/\mathcal{L}_1(G,G_*)>0.
\end{align}
Assume by contrary that the inequality~\eqref{eq:general_global_inequality} does not hold true, then we can find a sequence of mixing measures $G'_n\in\mathcal{M}_{N}(\Theta)$ such that $\mathcal{L}_1(G'_n,G_*)>\varepsilon'$ and
\begin{align*}
    \lim_{n\to\infty}\frac{\normf{f_{G'_n}-f_{G_*}}}{\mathcal{L}_1(G'_n,G_*)}=0,
\end{align*}
which indicates that $\normf{f_{G'_n}-f_{G_*}}\to0$ as $n\to\infty$. Recall that $\Theta$ is a compact set, therefore, we can replace the sequence $G'_n$ by one of its subsequences that converges to a mixing measure $G'\in\mathcal{M}_{N}(\Omega)$. Since $\mathcal{L}_1(G'_n,G_*)>\varepsilon'$, we deduce that $\mathcal{L}_1(G',G_*)>\varepsilon'$. \\

\noindent
Next, by invoking the Fatou's lemma, we have that
\begin{align*}
    0=\lim_{n\to\infty}\normf{f_{G'_n}-f_{G_*}}^2\geq \int\liminf_{n\to\infty}\Big|f_{G'_n}(\xbm)-f_{G_*}(\xbm)\Big|^2~\dint\mu(\xbm).
\end{align*}
Thus, we get that $f_{G'}(\xbm)=f_{G_*}(\xbm)$ for $\mu$-almost surely $x$. From Proposition~\ref{prop:general_identifiability}, we deduce that $G'\equiv G_*$. Consequently, it follows that $\mathcal{L}_1(G',G_*)=0$, contradicting the fact that $\mathcal{L}_1(G',G_*)>\varepsilon'>0$. Hence, the proof is completed.

\subsection{Proof of Theorem~\ref{theorem:poly_linear_experts}}
\label{appendix:poly_linear_experts}
In this proof, we first introduce the following lemma which will be used for our subsequent main proof of Theorem~\ref{theorem:poly_linear_experts}.
\begin{lemma}
\label{lemma:log_n}
    Suppose that the following holds for any $r\geq 1$:
    \begin{align}
        \label{eq:assumption}
         \lim_{\varepsilon\to0}\inf_{G\in\mathcal{M}_{N}(\Theta):\mathcal{L}_{2,r}(G,G_*)\leq\varepsilon}\frac{\normf{f_{G}-f_{G_*}}}{\mathcal{L}_{2,r}(G,G_*)}=0.
    \end{align}
    Then, we achieve that  for any $r\geq 1$:
    \begin{align}
        \label{eq:desired_minimax_bound}
         \inf_{\overline{G}_n\in\mathcal{M}_{N}(\Theta)}\sup_{G\in\mathcal{M}_{N}(\Theta)\setminus\mathcal{M}_{N^*-1}(\Theta)}\bbE_{f_{G}}[\mathcal{L}_{2,r}(\overline{G}_n,G)]\gtrsim n^{-1/2},
    \end{align}
   where where $\bbE_{f_{G}}$ indicates the expectation taken w.r.t the product measure with $f^n_{G}$.
\end{lemma}
\begin{proof}[Proof of Lemma~\ref{lemma:log_n}]
\noindent
Firstly, note that from the Gaussian assumption on the noise variables, we obtain that $Y_{i}|X_{i} \sim \mathcal{N}(f_{G_{*}}(\xbm_{i}), \sigma^2)$ for all $i \in [n]$. Next, it follows from the assumption in equation~\eqref{eq:assumption} that for sufficiently small $\varepsilon>0$ and a fixed constant $C_1>0$ which we will choose later, there exists a mixing measure $G'_* \in \mathcal{M}_{N}(\Theta)$ such that $\mathcal{L}_{2,r}(G'_*,G_*)=2 \varepsilon$ and $\|f_{G'_*} - f_{G_*}\|_{L_2(\mu)} \leq C_1\varepsilon$. According to Le Cam's lemma~\cite{yu97lecam}, since the Voronoi loss function $\mathcal{L}_{2,r}$ satisfies the weak triangle inequality, we get that
\begin{align}
    \inf_{\overline{G}_n\in\mathcal{M}_{N}(\Theta)}&\sup_{G\in\mathcal{M}_{N}(\Theta)\setminus\mathcal{M}_{N^*-1}(\Theta)}\bbE_{f_{G}}[\mathcal{L}_{2,r}(\overline{G}_n,G)]\nonumber \\
    & \gtrsim \frac{\mathcal{L}_{2,r}(G'_*,G_*)}{8} \text{exp}(- n \mathbb{E}_{X \sim \mu}[\text{KL}(\mathcal{N}(f_{G'_{*}}(\xbm), \sigma^2),\mathcal{N}(f_{G_{*}}(\xbm), \sigma^2))]) \nonumber \\
    & \gtrsim \varepsilon \cdot \text{exp}(-n \|f_{G'_*} - f_{G_*}\|_{L_2(\mu)}^2), \nonumber \\
    & \gtrsim \varepsilon \cdot \text{exp}(-C_{1} n \varepsilon^2), \label{eq:LeCam_inequality}
\end{align}
where the second inequality is due to the fact that
\begin{align*}
    \text{KL}(\mathcal{N}(f_{G'_{*}}(\xbm), \sigma^2),\mathcal{N}(f_{G_{*}}(\xbm), \sigma^2)) = \dfrac{(f_{G'_*}(\xbm) - f_{G_*}(\xbm))^2}{2 \sigma^2}.
\end{align*}
By choosing $\varepsilon=n^{-1/2}$, we obtain that $\varepsilon \cdot \text{exp}(-C_{1} n \varepsilon^2)=n^{-1/2}\exp(-C_1)$. As a consequence, we achieve the desired minimax lower bound in equation~\eqref{eq:desired_minimax_bound}.
\end{proof}

\noindent
Given the result of Lemma~\ref{lemma:log_n}, it suffices to prove that the following limit holds true for any $r\geq 1$:
\begin{align}
    \label{eq:ratio_zero_limit}
    \lim_{\varepsilon\to0}\inf_{G\in\mathcal{M}_{N}(\Theta):\mathcal{L}_{2,r}(G,G_*)\leq\varepsilon}\frac{\normf{f_{G}-f_{G_*}}}{\mathcal{L}_{2,r}(G,G_*)}=0.
\end{align}
To this end, we will construct a sequence of mixing measures $(G_n)$ such that both $\mathcal{L}_{2,r}(G_n,G_*)\to0$ and 
\begin{align*}
    \frac{\normf{f_{G_n}-f_{G_*}}}{\mathcal{L}_{2,r}(G_n,G_*)}\to0,
\end{align*}
as $n\to\infty$. In particular, we consider the sequence  $G_n=\sum_{i=1}^{N^*+1}\exp(\cin)\delta_{(\ain,\bin,\boin,\bzin)}$, where 
\begin{itemize}
    \item $\exp(c^n_{1})=\exp(c^n_{2})=\frac{1}{2}\exp(c^*_{1})+\frac{1}{2n^{r+1}}$ and  $\exp(c^n_{i})=\exp(c^n_{i-1})$ for any $3\leq i\leq N^*+1$;
    \item $A^n_{1}=A^n_{2}=A^*_{1}$ and  $A^n_{i}=A^*_{i-1}$ for any $3\leq i\leq N^*+1$;
    \item $b^n_1=b^n_2=\bbm^*_1$ and $\bbm^n_i=\bbm^*_{i-1}$ for any $3\leq i\leq N^*+1$;
    \item ${\bm \beta}^n_{11}={\bm \beta}^n_{12}={\bm \beta}^*_{11}$ and ${\bm \beta}^n_{1i}={\bm \beta}^*_{1(i-1)}$ for any $3\leq i\leq N^*+1$;
    \item $\beta^n_{01}=\beta^*_{01}+\frac{1}{n}$, $\beta^n_{02}=\beta^*_{01}-\frac{1}{n}$ and  $\beta^n_{0i}=\beta^*_{0(i-1)}$ for any $3\leq i\leq N^*+1$.
\end{itemize}
Consequently, the loss function $\mathcal{L}_{2,r}(G_n,G_*)$ turns into
\begin{align}
    \label{eq:D_r_formulation}
    \mathcal{L}_{2,r}(G_n,G_*)=\frac{1}{n^{r+1}}+\Big[\exp(c^*_{1})+\frac{1}{n^{r+1}}\Big]\cdot\frac{1}{n^r}=\mathcal{O}(n^{-r}).
\end{align}
which suggests that $\mathcal{L}_{2,r}(G_n,G_*)\to0$ as $n\to\infty$. \\

\noindent
Now, we prove that $\normf{f_{G_n}-f_{G_*}}/\mathcal{L}_{2,r}(G_n,G_*)\to0$. For that purpose, let us consider  $$Q_n(\xbm):=\Big[\sum_{j=1}^{N^*}\exp(\xbm^{\top}\aj \xbm+(\bj)^{\top}\xbm)\Big]\cdot[f_{G_n}(\xbm)-f_{G_*}(\xbm)].$$ Then, we decompose $Q_n(\xbm)$ as $Q_n(\xbm)=A_n(\xbm)-B_n(\xbm)+C_{n}(\xbm)$ where we define
\begin{align*}
    A_n(\xbm)&=\sum_{j=1}^{N^*}\sum_{i\in\mathcal{V}_j}\exp(\cin)\Big[\exp(\xbm^{\top}\ain \xbm+(\bin)^{\top}\xbm)((\boin)^{\top}\xbm+\bzin)\\
    &\hspace{4cm}-\exp(\xbm^{\top}\aj \xbm+(\bj)^{\top}\xbm)((\boj)^{\top}\xbm+\bzj)\Big],\nonumber\\
    B_n(\xbm)&=\sum_{j=1}^{N^*}\sum_{i\in\mathcal{V}_j}\exp(\cin)\Big[\exp(\xbm^{\top}\ain \xbm+(\bin)^{\top}\xbm)-\exp(\xbm^{\top}\aj \xbm+(\bj)^{\top}\xbm)\Big]f_{G_n}(\xbm),\nonumber\\
    C_n(\xbm)&=\sum_{j=1}^{N^*}\Big(\sum_{i\in\mathcal{V}_j}\exp(\cin)-\exp(c^*_j)\Big)\Big[\exp(\xbm^{\top}\aj \xbm+(\bj)^{\top}\xbm)((\boj)^{\top}\xbm+\bzj)\\
    &\hspace{6cm}-\exp(\xbm^{\top}\aj \xbm+(\bj)^{\top}\xbm)f_{G_n}(\xbm)\Big].
\end{align*}
From the definitions of $A^n_{i}$, $\bbm^n_i$, ${\bm \beta}^n_{1i}$ and $\beta^n_{0i}$, we can rewrite $A_n(\xbm)$ as follows:
\begin{align*}
    A_n(\xbm)&=\frac{1}{2}\exp(c^n_{1})\exp(\xbm^{\top}A^*_{1}\xbm+(\bbm^*_1)^{\top}\xbm)[(\beta^n_{01}-\beta^*_{01})+(\beta^n_{02}-\beta^*_{01})]\\
    &=\frac{1}{2}\exp(c^n_{1})\exp(\xbm^{\top}A^*_{1}\xbm+(\bbm^*_1)^{\top}\xbm)\Big[\frac{1}{n}-\frac{1}{n}\Big]=0.
\end{align*}
Moreover, we can verify that $B_n(\xbm)=0$. Next, we have
\begin{align*}
    C_n(\xbm)&=\Big(\sum_{i=1}^{2}\exp(c^n_i)-\exp(c^*_1)\Big)\exp(\xbm^{\top}A^*_1 \xbm+(\bbm^*_1)^{\top}\xbm)\Big[(({\bm \beta}^*_{11})^{\top}\xbm+\beta^*_{01})-f_{G_n}(\xbm)\Big]\\
    &=\frac{1}{n^{r+1}}\cdot\exp(\xbm^{\top}A^*_1 \xbm+(\bbm^*_1)^{\top}\xbm)\Big[(({\bm \beta}^*_{11})^{\top}\xbm+\beta^*_{01})-f_{G_n}(\xbm)\Big]\\
    &\leq\mathcal{O}(n^{-(r+1)}),
\end{align*}
which leads to the fact that $C_n(\xbm)/\mathcal{L}_{2,r}(G_n,G_*)\to0$ as $n\to\infty$. \\

\noindent
As a result, $Q_n(\xbm)/\mathcal{L}_{2,r}(G_n,G_*)\to0$ as $n\to\infty$ for $\mu$-almost surely $x$. Since the term $\sum_{j=1}^{N^*}\exp(\xbm^{\top}\aj \xbm+(\bj)^{\top}\xbm)$ is bounded, we deduce that $[f_{G_n}(\xbm)-f_{G_*}(\xbm)]/\mathcal{L}_{2,r}(G_n,G_*)\to0$ for $\mu$-almost surely $x$. This result indicates that $\normf{f_{G_n}-f_{G_*}}/\mathcal{L}_{2,r}(G_n,G_*)\to0$ as $n\to\infty$. Hence, we achieve the claim~\eqref{eq:ratio_zero_limit} and complete the proof.

\subsection{Proof of Theorem~\ref{theorem:mono_regression_rate}}
\label{appendix:mono_regression_rate}
The proof of Theorem~\ref{theorem:mono_regression_rate} can be done in a similar fashion to that of Theorem~\ref{theorem:regression_rate} in Appendix~\ref{appendix:regression_rate}.

\subsection{Proof of Theorem~\ref{theorem:mono_strongly_identifiable}}
\label{appendix:mono_strongly_identifiable}
Our goal is also to demonstrate the following inequality:
\begin{align}
    \label{eq:mono_general_universal_inequality}
    \inf_{G\in\mathcal{M}_{N}(\Theta)}\normf{\tilde{f}_{G}-\tilde{f}_{G_*}}/\mathcal{L}_3(G,G_*)>0.
\end{align}
For that purpose, we divide the proof of the above inequality into local and global parts in the sequel. Here, we only present the proof of the local part, while that of the global part can be done using the same arguments as in Appendix~\ref{appendix:poly_strongly_identifiable}.
\\

\noindent
\textbf{Local part:} In this part, we demonstrate that
\begin{align}
    \label{eq:mono_general_local_inequality}
    \lim_{\varepsilon\to0}\inf_{G\in\mathcal{M}_{N}(\Theta):\mathcal{L}_3(G,G_*)\leq\varepsilon}\normf{\tilde{f}_{G}-\tilde{f}_{G_*}}/\mathcal{L}_3(G,G_*)>0.
\end{align}
Assume by contrary that the above claim is not true true, then there exists a sequence of mixing measures $G_n=\sum_{i=1}^{N^*}\exp(\cin)\delta_{(\Abm^n_i, \veta^n_i)}$ in $\mathcal{M}_{N}(\Theta)$ such that $\mathcal{L}_{3n}:=\mathcal{L}_3(G_n,G_*)\to0$ and
\begin{align}
    \label{eq:mono_general_ratio_limit}
    \normf{\tilde{f}_{G_n}-\tilde{f}_{G_*}}/\mathcal{L}_{1n}\to0,
\end{align}
as $n\to\infty$. Let us denote by $\mathcal{V}^n_j:=\mathcal{V}_j(G_n)$ a Voronoi cell of $G_n$ generated by the $j$-th components of $G_*$. Since our arguments are asymptotic, we may assume that those Voronoi cells do not depend on the sample size, i.e. $\mathcal{V}_j=\mathcal{V}^n_j$. Thus, the Voronoi loss $\mathcal{L}_{3n}$ can be represented as
\begin{align}
    \label{eq:loss_n2}
   &\mathcal{L}_{3n}:=\sum_{j:|\mathcal{V}_j|>1}\sum_{i\in\mathcal{V}_j}\exp(\cin)\Big[\|\daijn\|^{2}+\|\deijn\|^2\Big]\nonumber\\
    &+\sum_{j:|\mathcal{V}_j|=1}\sum_{i\in\mathcal{V}_j}\exp(\cin)\Big[\|\daijn\|+\|\deijn\|\Big]+\sum_{j=1}^{N^*}\Big|\sum_{i\in\mathcal{V}_j}\exp(\cin)-\exp(\cj)\Big|,
\end{align}
where we denote $\daijn:=A^n_{i}-\Abm^*_{j}$ and $\deijn:=\veta^n_i-\veta^*_j$.
\\

\noindent
Since $\mathcal{L}_{3n}\to0$, we get that $(\ain,\ein)\to(\aj,\ej)$ and $\sum_{i\in\mathcal{V}_j}\exp(\cin)\to\exp(c^*_j)$ as $n\to\infty$ for any $i\in\mathcal{V}_j$ and $j\in[N^*]$. Now, we divide the proof of local part into three steps as follows:
\\

\noindent
\textbf{Step 1 - Taylor expansion.} By abuse of notations, we sometimes tailor notations defined in Appendix~\ref{appendix:poly_strongly_identifiable} to the setting of this proof. In this step, we would like to decompose the quantity
\begin{align}
    \label{eq:mono_Qn_formulation}
    Q_n(\xbm):=\Big[\sum_{j=1}^{N^*}\exp(\xbm^{\top}\aj \xbm+\cj)\Big]\cdot[\tilde{f}_{G_n}(\xbm)-\tilde{f}_{G_*}(\xbm)]
\end{align}
into a combination of linearly independent elements using Taylor expansion. By using the same arguments for deriving equation~\eqref{eq:general_Q_n}, we get that $Q_n(\xbm)=A_n(\xbm)-B_n(\xbm)+C_n(\xbm)$, where
\begin{align*}
    A_n(\xbm)&:=\sum_{j=1}^{N^*}\sum_{i\in\mathcal{V}_j}\exp(\cin)\Big[\exp(\xbm^{\top}\ain \xbm)\mathcal{E}(\xbm;\ein)-\exp(\xbm^{\top}\aj \xbm)\mathcal{E}(\xbm;\ej)\Big],\\
    B_n(\xbm)&:=\sum_{j=1}^{N^*}\sum_{i\in\mathcal{V}_j}\exp(\cin)\Big[\exp(\xbm^{\top}\ain \xbm)-\exp(\xbm^{\top}\aj \xbm)\Big]\tilde{f}_{G_n}(\xbm),\\
    C_n(\xbm)&:=\sum_{j=1}^{N^*}\Big(\sum_{i\in\mathcal{V}_j}\exp(\cin)-\exp(c^*_j)\Big)\Big[\exp(\xbm^{\top}\aj \xbm)\mathcal{E}(\xbm;\ej)-\exp(\xbm^{\top}\aj \xbm)\tilde{f}_{G_n}(\xbm)\Big].
\end{align*}
\textbf{Decomposition of $A_n(\xbm)$.} Let us denote $E(\xbm;A):=\exp(\xbm^{\top}A\xbm)$, then $A_n(\xbm)$ can be separated into two terms as follows:
\begin{align*}
    A_n(\xbm)&:=\sum_{j:|\mathcal{V}_j|=1}\sum_{i\in\mathcal{V}_j}\exp(\cin)\Big[E(\xbm;\ain)\mathcal{E}(\xbm;\ein)-E(\xbm;\aj)\mathcal{E}(\xbm;\ej)\Big]\\
    &+\sum_{j:|\mathcal{V}_j|>1}\sum_{i\in\mathcal{V}_j}\exp(\cin)\Big[E(\xbm;\ain)\mathcal{E}(\xbm;\ein)-E(\xbm;\aj)\mathcal{E}(\xbm;\ej)\Big]\\
    &:=A_{n,1}(\xbm)+A_{n,2}(\xbm).
\end{align*}
By means of the first-order Taylor expansion, we have
\begin{align*}
    A_{n,1}(\xbm)&=\sum_{j:|\mathcal{V}_j|=1}\sum_{i\in\mathcal{V}_j}\frac{\exp(\cin)}{\alpha!}\sum_{|\alpha|=1}(\daijn)^{\alpha_1}(\deijn)^{\alpha_2}\frac{\partial^{|\alpha_1|}E}{\partial A^{\alpha_1}}(\xbm;\aj)\frac{\partial^{|\alpha_2|}\mathcal{E}}{\partial\veta^{\alpha_2}}(\xbm;\ej)
    +R_{n,1}(\xbm)\\
    &=\sum_{j:|\mathcal{V}_j|=1}\sum_{|\alpha_1|+|\alpha_2|=1}S_{n,j,\alpha_1,\alpha_2}\frac{\partial^{|\alpha_1|}E}{\partial A^{\alpha_1}}(\xbm;\aj)\frac{\partial^{|\alpha_2|}\mathcal{E}}{\partial\veta^{\alpha_2}}(\xbm;\ej)
    +R_{n,1}(\xbm),
\end{align*}
where $R_{n,1}(\xbm)$ is a Taylor remainder such that $R_{n,1}(\xbm)/\mathcal{L}_{3n}\to0$ as $n\to\infty$, and
\begin{align*}
    S_{n,j,\alpha_1,\alpha_2}:=\sum_{i\in\mathcal{V}_j}\frac{\exp(\cin)}{\alpha!}(\daijn)^{\alpha_1}(\deijn)^{\alpha_2}.
\end{align*}
On the other hand, by applying the second-order Taylor expansion, we get that
\begin{align*}
    A_{n,2}(\xbm)=\sum_{j:|\mathcal{V}_j|>1}\sum_{1\leq|\alpha_1|+|\alpha_2|\leq 2}S_{n,j,\alpha_1,\alpha_2}\frac{\partial^{|\alpha_1|}E}{\partial A^{\alpha_1}}(\xbm;\aj)\frac{\partial^{|\alpha_2|}\mathcal{E}}{\partial\veta^{\alpha_2}}(\xbm;\ej)
    +R_{n,2}(\xbm),
\end{align*}
in which $R_{n,2}(\xbm)$ is a Taylor remainder such that $R_{n,2}(\xbm)/\mathcal{L}_{3n}\to0$ as $n\to\infty$.
\\

\noindent
\textbf{Decomposition of $B_n(\xbm)$.} Recall that we have
\begin{align*}
    B_n(\xbm)&=\sum_{j:|\mathcal{V}_j|=1}\sum_{i\in\mathcal{V}_j}\exp(\cin)\Big[E(\xbm;\ain)-E(\xbm;\aj)\Big]\tilde{f}_{G_n}(\xbm)\\
    &+\sum_{j:|\mathcal{V}_j|>1}\sum_{i\in\mathcal{V}_j}\exp(\cin)\Big[E(\xbm;\ain)-E(\xbm;\aj)\Big]\tilde{f}_{G_n}(\xbm)\\
    &:=B_{n,1}(\xbm) + B_{n,2}(\xbm).
\end{align*}
By invoking first-order and second-order Taylor expansions to $B_{n,1}(\xbm)$ and $B_{n,2}(\xbm)$, it follows that
\begin{align*}
    B_{n,1}(\xbm)&=\sum_{j:|\mathcal{V}_j|=1}\sum_{|\ell|=1}T_{n,j,\ell}\cdot\frac{\partial^{|\ell|}E}{\partial A^{\ell}}(\xbm;\aj)\tilde{f}_{G_n}(\xbm)+R_{n,3}(\xbm),\\
    B_{n,2}(\xbm)&=\sum_{j:|\mathcal{V}_j|>1}\sum_{1\leq|\ell|\leq 2}T_{n,j,\ell}\cdot\frac{\partial^{|\ell|}E}{\partial A^{\ell}}(\xbm;\aj)\tilde{f}_{G_n}(\xbm)+R_{n,4}(\xbm),
\end{align*}
where we define 
\begin{align*}
    T_{n,j,\ell}:=\sum_{i\in\mathcal{V}_j}\frac{\exp(\cin)}{\ell!}(\daijn)^{\ell}.
\end{align*} 
Additionally, $R_{n,3}(\xbm)$ and $R_{n,4}(\xbm)$ are Taylor remainders such that $R_{n,3}(\xbm)/\mathcal{L}_{3n}\to0$ and $R_{n,3}(\xbm)/\mathcal{L}_{3n}\to0$ as $n\to\infty$. 
\\

\noindent
Collect the above results together, we can represent $Q_n(\xbm)$ as
\begin{align}
    \label{eq:mono_Qn_decomposition}
    Q_n(\xbm)&=\sum_{j=1}^{N^*}\sum_{0\leq|\alpha_1|+|\alpha_2|\leq 2}S_{n,j,\alpha_1,\alpha_2}\frac{\partial^{|\alpha_1|}E}{\partial A^{\alpha_1}}(\xbm;\aj)\frac{\partial^{|\alpha_2|}\mathcal{E}}{\partial\veta^{\alpha_2}}(\xbm;\ej),\nonumber\\
    &-\sum_{j=1}^{N^*}\sum_{0\leq|\ell|\leq 2}T_{n,j,\ell}\cdot\frac{\partial^{|\ell|}E}{\partial A^{\ell}}(\xbm;\aj)\tilde{f}_{G_n}(\xbm) +\sum_{i=1}^{4}R_{n,i}(\xbm),
\end{align}
where we define $S_{n,j,\mathbf{0}_{d\times d},\zeroq}=T_{n,j,\mathbf{0}_{d\times d}}=\sum_{i\in\mathcal{V}_j}\exp(\cin)-\exp(\cj)$ for any $j\in[N^*]$.
\\

\noindent
\textbf{Step 2 - Non-vanishing coefficients.} In this step, we demonstrate that at least one among ratios of the forms $S_{n,j,\alpha_1,\alpha_2}/\mathcal{L}_{3n}$ and $T_{n,j,\ell}/\mathcal{L}_{3n}$ goes to zero as $n$ tends to infinity. Indeed, assume by contrary that
\begin{align*}
    \frac{S_{n,j,\alpha_1,\alpha_2}}{\mathcal{L}_{3n}}\to0, \qquad \frac{T_{n,j,\ell}}{\mathcal{L}_{3n}}\to0,
\end{align*}
for any $j\in[N^*]$, $0\leq|\alpha_1|,|\alpha_2|,|\ell|\leq 2$. Then, we get
\begin{align}
    \label{eq:mono_weight_limit}
    \frac{1}{\mathcal{L}_{3n}}\sum_{j=1}^{N^*}\Big|\sum_{i\in\mathcal{V}_j}\exp(\cin)-\exp(\cj)\Big|=\sum_{j=1}^{N^*}\Big|\frac{S_{n,j,\mathbf{0}_{d\times d},\zeroq}}{\mathcal{L}_{3n}}\Big|\to0.
\end{align}
Now, we consider indices $j\in[N^*]$ such that its corresponding Voronoi cell has only one element, i.e. $|\mathcal{V}_j|=1$.
\begin{itemize}
    \item For arbitrary $u,v\in[d]$, let $\alpha_1\in\mathbb{N}^{d\times d}$ and $\alpha_2=\zeroq$ such that $\alpha_1^{(uv)}=1$ while other entries equal to zero. Then, we have $\frac{1}{\mathcal{L}_{3n}}\cdot\sum_{i\in\mathcal{V}_j}\exp(\cin)|(\daijn)^{(uv)}|=|S_{n,j,\alpha_1,\alpha_2}|/\mathcal{L}_{3n}\to0$ as $n\to\infty$. By taking the summation of the previous term with $u,v\in[d]$, we achieve that $\frac{1}{\mathcal{L}_{3n}}\sum_{i\in\mathcal{V}_j}\exp(\cin)\|\daijn\|_1\to0$. Owing to the topological equivalence between norm-1 and norm-2, it follows that
    \begin{align}
        \label{eq:mono_a_limit_1}
        \frac{1}{\mathcal{L}_{3n}}\sum_{i\in\mathcal{V}_j}\exp(\cin)\|\daijn\|\to0.
    \end{align}
    \item For arbitrary $u\in[d]$, let $\alpha_1=\mathbf{0}_{d\times d}$ and $\alpha_2\in\mathbb{N}^q$ such that $\alpha_2^{(u)}=1$ while other entries equal to zero. Then, we get $\frac{1}{\mathcal{L}_{3n}}\cdot\sum_{i\in\mathcal{V}_j}\exp(\cin)|(\deijn)^{(u)}|=|S_{n,j,\alpha_3,\ell_1}|/\mathcal{L}_{3n}\to0$ as $n\to\infty$. By taking the summation of the previous term with $u\in[q]$, we achieve that $\frac{1}{\mathcal{L}_{3n}}\sum_{i\in\mathcal{V}_j}\exp(\cin)\|\deijn\|_1\to0$, or equivalently,
    \begin{align}
        \label{eq:mono_eta_limit_1}
        \frac{1}{\mathcal{L}_{3n}}\sum_{i\in\mathcal{V}_j}\exp(\cin)\|\deijn\|\to0.
    \end{align}
\end{itemize}
Combine the limits in equations~\eqref{eq:mono_a_limit_1} and \eqref{eq:mono_eta_limit_1}, we obtain that 
\begin{align}
    \label{eq:mono_order_1_limit}
    \frac{1}{\mathcal{L}_{3n}}\sum_{j:|\mathcal{V}_j|=1}\sum_{i\in\mathcal{V}_j}\exp(\cin)[\|\daijn\|+\|\deijn\|]\to0,
\end{align}
as $n\to\infty$. 
\\

\noindent
Next, we consider indices $j\in[N^*]$ such that its corresponding Voronoi cell has more than one element, i.e. $|\mathcal{V}_j|>1$. 
\begin{itemize}
    \item For arbitrary $u,v\in[d]$, let $\alpha_1\in\mathbb{N}^{d\times d}$ and $\alpha_2=\zeroq$ such that $\alpha_1^{(uv)}=2$ while other entries equal to zero. Then, we have $\frac{1}{\mathcal{L}_{3n}}\cdot\sum_{i\in\mathcal{V}_j}\exp(\cin)|(\daijn)^{(uv)}|^2=|S_{n,j,\alpha_1,\alpha_2}|/\mathcal{L}_{3n}\to0$ as $n\to\infty$. By taking the summation of the previous term with $u,v\in[d]$, we achieve that 
    \begin{align}
        \label{eq:mono_a_limit_2}
        \frac{1}{\mathcal{L}_{3n}}\sum_{i\in\mathcal{V}_j}\exp(\cin)\|\daijn\|^2\to0.
    \end{align}
    \item For arbitrary $u\in[d]$, let $\alpha_1=\mathbf{0}_{d\times d}$ and $\alpha_2\in\mathbb{N}^q$ such that $\alpha_2^{(u)}=2$ while other entries equal to zero. Then, we get $\frac{1}{\mathcal{L}_{3n}}\cdot\sum_{i\in\mathcal{V}_j}\exp(\cin)|(\deijn)^{(u)}|^2=|S_{n,j,\alpha_3,\ell_1}|/\mathcal{L}_{3n}\to0$ as $n\to\infty$. By taking the summation of the previous term with $u\in[q]$, we achieve that
    \begin{align}
        \label{eq:mono_eta_limit_2}
        \frac{1}{\mathcal{L}_{3n}}\sum_{i\in\mathcal{V}_j}\exp(\cin)\|\deijn\|^2\to0.
    \end{align}
\end{itemize}
Putting the limits in equations~\eqref{eq:mono_a_limit_1} and \eqref{eq:mono_eta_limit_1}, we have
\begin{align}
    \label{eq:mono_order_2_limit}
    \frac{1}{\mathcal{L}_{3n}}\sum_{j:|\mathcal{V}_j|>1}\sum_{i\in\mathcal{V}_j}\exp(\cin)[\|\daijn\|+\|\deijn\|]\to0,
\end{align}
as $n\to\infty$. Taking the summation of three limits in equations~\eqref{eq:mono_weight_limit}, \eqref{eq:mono_order_1_limit} and \eqref{eq:mono_order_2_limit}, we deduce that $1=\mathcal{L}_{3n}/\mathcal{L}_{3n}\to0$ as $n\to\infty$, which is a contradiction. Thus, at least one among ratios of the forms $S_{n,j,\alpha_1,\alpha_2}/\mathcal{L}_{3n}$ and $T_{n,j,\ell}/\mathcal{L}_{3n}$ goes to zero as $n$ tends to infinity.
\\

\noindent
\textbf{Step 3 - Application of Fatou's lemma.} In this step, we show that all the ratios $S_{n,j,\alpha_1,\alpha_2}/\mathcal{L}_{3n}$ and $T_{n,j,\ell}/\mathcal{L}_{3n}$ go to zero as $n\to\infty$, which contradicts to the conclusion in Step 2. In particular, by denoting $m_n$ as the maximum of the absolute values of those ratios. From the result of Step 2, it follows that $1/m_n\not\to\infty$. 
\\

\noindent
Recall from the hypothesis in equation~\eqref{eq:mono_general_ratio_limit} that $\normf{\tilde{f}_{G_n}-\tilde{f}_{G_*}}/\mathcal{L}_{3n}\to0$ as $n\to\infty$, which indicates that $\|\tilde{f}_{G_n}-\tilde{f}_{G_*}\|_{L^1(\mu)}/\mathcal{L}_{3n}\to0$. Therefore, by applying the Fatou's lemma, we get that
\begin{align*}
    0=\lim_{n\to\infty}\frac{\|\tilde{f}_{G_n}-\tilde{f}_{G_*}\|_{L^1(\mu)}}{m_n\mathcal{L}_{3n}}\geq \int \liminf_{n\to\infty}\frac{|\tilde{f}_{G_n}(\xbm)-\tilde{f}_{G_*}(\xbm)|}{m_n\mathcal{L}_{3n}}\dint\mu(\xbm)\geq 0.
\end{align*}
This result implies that $\frac{1}{m_n\mathcal{L}_{3n}}\cdot[\tilde{f}_{G_n}(\xbm)-\tilde{f}_{G_*}(\xbm)]\to0$ as $n\to\infty$ for $\mu$-almost surely $x$. Looking at the formulation of $Q_n(\xbm)$ in equation~\eqref{eq:mono_Qn_formulation}, since the term $\Big[\sum_{j=1}^{N^*}\exp(\xbm^{\top}\aj \xbm+\cj)\Big]$ is bounded, we deduce that the term $\frac{1}{m_n\mathcal{L}_{3n}}\cdot Q_n(\xbm)\to0$ for $\mu$-almost surely $x$.
\\

\noindent
Let us denote
\begin{align*}
    \frac{S_{n,j,\alpha_1,\alpha_2}}{m_n\mathcal{L}_{3n}}\to \phi_{j,\alpha_1,\alpha_2}, \qquad \frac{T_{n,j,\ell}}{m_n\mathcal{L}_{3n}}\to\varphi_{j,\ell},
\end{align*}
with a note that at least one among them is non-zero. Then, from the decomposition of $Q_n(\xbm)$ in equation~\eqref{eq:mono_Qn_decomposition}, we have
\begin{align*}
    \sum_{j=1}^{N^*}\sum_{|\alpha_1|+|\alpha_2|=0}^{1+\mathbf{1}_{\{|\mathcal{V}_j|>1\}}}\phi_{j,\alpha_1,\alpha_2}\cdot&\frac{\partial^{|\alpha_1|}E}{\partial A^{\alpha_1}}(\xbm;\aj)\frac{\partial^{|\alpha_2|}\mathcal{E}}{\partial\veta^{\alpha_2}}(\xbm;\ej),\nonumber\\
    &-\sum_{j=1}^{N^*}\sum_{|\ell|=0}^{1+\mathbf{1}_{\{|\mathcal{V}_j|>1\}}}\varphi_{j,\ell}\cdot\frac{\partial^{|\ell|}E}{\partial A^{\ell}}(\xbm;\aj)\tilde{f}_{G_*}(\xbm) =0,
\end{align*}
for $\mu$-almost surely $x$. It is worth noting that the term $\frac{\partial^{|\alpha_1|}E}{\partial A^{\alpha_1}}(\xbm;\aj)\frac{\partial^{|\alpha_2|}\mathcal{E}}{\partial\veta^{\alpha_2}}(\xbm;\ej)$ can be explicitly expressed as
\begin{itemize}
    \item When $|\alpha_1|=0,|\alpha_2|=0$: $\exp(\xbm^{\top}\aj \xbm)\mathcal{E}(\xbm;\ej)$;
    \item When $|\alpha_1|=1,|\alpha_2|=0$: $x^{(u)}x^{(v)}\exp(\xbm^{\top}\aj \xbm)\mathcal{E}(\xbm;\ej)$;
    \item When $|\alpha_1|=0,|\alpha_2|=1$: $\exp(\xbm^{\top}\aj \xbm)\frac{\partial h}{\partial\veta^{(w)}}(\xbm;\ej)$;
    \item When $|\alpha_1|=1,|\alpha_2|=1$: $x^{(u)}x^{(v)}\exp(\xbm^{\top}\aj\xbm)\frac{\partial h}{\partial\veta^{(w)}}(\xbm;\ej)$;
    \item When $|\alpha_1|=2,|\alpha_2|=0$: $x^{(u)}x^{(v)}x^{(u')}x^{(v')}\exp(\xbm^{\top}\aj \xbm)\mathcal{E}(\xbm;\ej)$;
    \item When $|\alpha_1|=0,|\alpha_2|=2$: $\exp(\xbm^{\top}\aj \xbm)\frac{\partial^2 h}{\partial\veta^{(w)}\partial\veta^{(w')}}(\xbm;\ej)$.
\end{itemize}

\noindent
Recall that the expert function $h$ satisfies the condition in Definition~\ref{def:mono_strong_identifiability}, i.e. the set
\begin{align*}
    \left\{x^{\nu}\cdot\frac{\partial^{|\gamma|} h}{\partial\veta^{\gamma}}(\xbm;\ej):j\in[N^*], \ \frac{|\nu|}{2}\in\{0,1,2\}, \ 0\leq|\gamma|\leq 2-\frac{|\nu|}{2}\right\}
\end{align*}
is linearly independent for $\mu$-almost surely $x$. Therefore, we obtain that $\phi_{j,\alpha_1,\alpha_2}=\varphi_{j,\ell}=0$ for all $j\in[N^*]$, $0\leq|\alpha_1|+|\alpha_2|,|\ell|\leq 1+\mathbf{1}_{\{|\mathcal{V}_j|>1\}}$. This result turns out to contradict the fact that at least one among them is different from zero. Hence, we achieve the inequality in equation~\eqref{eq:mono_general_local_inequality}.

\section{Identifiability of Quadratic Gating MoE}
\label{appendix:auxiliary_results}
In this appendix, we study the identifiability of the MoE models with the quadratic polynomial gate and the quadratic monomial gate in Proposition~\ref{prop:general_identifiability} and Proposition~\ref{prop:mono_general_identifiability}, respectively.
\begin{proposition}
    \label{prop:general_identifiability}
    If $f_{G}(x)=f_{G_*}(x)$ holds true for almost every $\xbm$, then we get that $G\equiv G'$.
\end{proposition}
\begin{proof}[Proof of Proposition~\ref{prop:general_identifiability}]
    Since $f_{G}(x)=f_{G_*}(x)$ for almost every $\xbm$, we have
    \begin{align}
        \label{eq:general_identifiable_equation}
        &\sum_{i=1}^{N}\softmax\Big(\xbm^{\top}\Abm_i\xbm+(\bbm_i)^{\top}\xbm+c_i\Big)\cdot \mathcal{E}(\xbm,\veta_i)\nonumber\\
        &\hspace{3cm}=\sum_{i=1}^{N^*}\softmax\Big(\xbm^{\top}\Abm^*_i\xbm+(\bbm^*_i)^{\top}\xbm+c^*_i\Big)\cdot \mathcal{E}(\xbm,\veta^*_i).
    \end{align}
    Note that since the expert function $h(\cdot,\veta)$ satisfies the conditions in Definition~\ref{def:strong_identifiability}, then given an arbitrary $N'\in\mathbb{N}$, then the set $\{\mathcal{E}(\xbm,\veta'_i):i\in[N']\}$, where $\veta'_1,\ldots,\veta'_{N'}$ are distinct vectors, is linearly independent for almost every $\xbm$. If $N\neq N^*$, then there exists some $i\in[N]$ such that $\veta_i\neq\veta^*_j$ for any $j\in[N^*]$. This implies that $\softmax\Big(\xbm^{\top}\Abm_i\xbm+(\bbm_i)^{\top}\xbm+c_i\Big)=0$, which is a contradiction. Thus, we must have that $N=N^*$. As a result, we get that
    \begin{align*}
        \Big\{\softmax\Big(\xbm^{\top}\Abm_i\xbm+(\bbm_i)^{\top}\xbm+c_i\Big)&:i\in[N]\Big\}\\
        &=\Big\{\softmax\Big(\xbm^{\top}\Abm^*_i\xbm+(\bbm^*_i)^{\top}\xbm+c^*_i\Big):i\in[N^*]\Big\},
    \end{align*}
    for almost every $\xbm$. WLOG, we may assume that 
    \begin{align}
        \label{eq:general_soft-soft}
        \softmax\Big(\xbm^{\top}\Abm_i\xbm+(\bbm_i)^{\top}\xbm+c_i\Big)=\softmax\Big(\xbm^{\top}\Abm^*_i\xbm+(\bbm^*_i)^{\top}\xbm+c^*_i\Big),
    \end{align}
    for almost every $\xbm$ for any $i\in[N^*]$. It is worth noting that the $\softmax$ function is invariant to translations, then equation~\eqref{eq:general_soft-soft} indicates that $\Abm_i=\Abm^*_i+\mathbf{T}_2$ $\bbm_{i}=\bbm^*_i+\mathbf{t}_1$ and $c_{i}=c^*_i+t_0$ for some $\mathbf{T}_2\in\mathbb{R}^{d\times d}$, $\mathbf{t}_1\in\mathbb{R}^d$ and $t_0\in\mathbb{R}$. However, from the assumptions $\Abm_{k}=\Abm^*_{k}$, $\bbm_{k}=\bbm^*_{k}=\zerod$ and $c_{k}=c^*_{k}=0$, we deduce that $\mathbf{T}_2=\mathbf{0}_{d\times d}$, $t_1=\zerod$ and $t_0=0$. Consequently, we get that $\Abm_i=\Abm^*_i$, $\bbm_{i}=\bbm^*_i$ and $c_{i}=c^*_i$ for any $i\in[N^*]$. Then, equation~\eqref{eq:general_identifiable_equation} can be rewritten as
    \begin{align}
        \label{eq:general_new_identifiable_equation}
        \sum_{i=1}^{N^*}\exp(c_{i})\exp\Big(\xbm^{\top}\Abm_i\xbm+(\bbm_i)^{\top}\xbm\Big)\mathcal{E}(\xbm,\veta_i)=\sum_{i=1}^{N^*}\exp(c^*_i)\exp\Big(\xbm^{\top}\Abm^*_i\xbm+(\bbm^*_i)^{\top}\xbm\Big)\mathcal{E}(\xbm,\veta^*_i),
    \end{align}
    for almost every $\xbm$. Next, we denote $P_1,P_2,\ldots,P_M$ as a partition of the index set $[N^*]$, where $M\leq N^*$, such that $\exp(c_{i})=\exp(c^*_{i'})$ for any $i,i'\in P_j$ and $j\in[N^*]$. On the other hand, when $i$ and $i'$ do not belong to the same set $P_j$, we let $\exp(c_{i})\neq\exp(c_{i'})$. Thus, we can reformulate equation~\eqref{eq:general_new_identifiable_equation} as
    \begin{align*}
        &\sum_{j=1}^{M}\sum_{i\in{P}_j}\exp(c_{i})\exp\Big(\xbm^{\top}\Abm_i\xbm+(\bbm_i)^{\top}\xbm\Big)\mathcal{E}(\xbm,\veta_i)\nonumber\\
        &\hspace{3cm}=\sum_{j=1}^{M}\sum_{i\in{P}_j}\exp(c^*_{i})\exp\Big(\xbm^{\top}\Abm^*_i\xbm+(\bbm^*_i)^{\top}\xbm\Big)\mathcal{E}(\xbm,\veta^*_i),
    \end{align*}
    for almost every $\xbm$. Recall that $\Abm_i=\Abm^*_i$, $\bbm_{i}=\bbm^*_i$ and $c_{i}=c^*_i$ for any $i\in[N^*]$, then the above equation implies that
    \begin{align*}
        \{\veta_i:i\in P_j\}\equiv\{\veta^*_i:i\in P_j\},
    \end{align*}
    for almost every $\xbm$ for any $j\in[m]$. 

    \noindent
    As a consequence, 
    \begin{align*}
        G=\sum_{j=1}^{M}\sum_{i\in P_j}\exp(c_{i})\delta_{(\Abm_i,\bbm_i,\veta_i)}=\sum_{j=1}^{M}\sum_{i\in P_j}\exp(c_{i})\delta_{(\Abm^*_i,\bbm^*_i,\veta^*_i)}=G_*.
    \end{align*}
    Hence, we reach the conclusion of this proposition.
\end{proof}

\begin{proposition}
    \label{prop:mono_general_identifiability}
    If $\tilde{f}_{G}(x)=\tilde{f}_{G_*}(x)$ holds true for almost every $\xbm$, then we get that $G\equiv G'$.
\end{proposition}

\noindent
The proof of Proposition~\ref{prop:mono_general_identifiability} can be done in a similar fashion to that of Proposition~\ref{prop:general_identifiability}.

\section{Experimental Details}
\label{appendix:experimental-details}

\subsection{Performance of Active-Attention Mechanism}
\subsubsection{Dataset Details}
\textbf{CIFAR-10 Dataset.} CIFAR-10 \cite{krizhevsky2009learning} is an established computer-vision dataset used for object recognition. It consists of 60,000 32$\times$32 color images containing one of 10 object classes ("plane", "car", "bird", "cat", "deer", "dog", "frog", "horse", "ship", "truck"), with 6000 images per class. 
\\

\noindent
\textbf{ImageNet Dataset.} We use the ImageNet database from ILSVRC2012 \citep{russakovsky2015imagenet} that contains $1.28M$ training images and $50K$ validation images, where the task is to classify images into 1,000 distinct categories, using a vast dataset of over 1.2 million training images and 150,000 validation and test images sourced from the ImageNet database.
\\

\noindent
\textbf{WikiText-103 Dataset.}
WikiText-103\footnote[1]{www.salesforce.com/products/einstein/ai-research/the-wikitext-dependency-language-modeling-dataset/} is a language modeling dataset that contains collection of tokens extracted from good and featured articles from Wikipedia, which is suitable for models that can leverage long-term dependencies. It contains around $268K$ words and its training set consists of about $28K$ articles with $103M$ tokens, this corresponds to text blocks of about 3600 words. The validation set and test sets consist of 60 articles with $218K$ and $246K$ tokens respectively. 
\\

\noindent
\textbf{Multivariate Time Series Forecasting Datasets.} 
The Weather\footnote[2]{https://www.bgc-jena.mpg.de/wetter/} dataset captures 21 meteorological indicators in Germany, such as humidity and air temperature. The Traffic\footnote[3]{https://pems.dot.ca.gov/} dataset records road occupancy rates from various sensors on San Francisco freeways. The Electricity\footnote[4]{https://archive.ics.uci.edu/ml/datasets/ElectricityLoadDiagrams20112014} dataset provides hourly electricity consumption data for 321 customers. The Illness\footnote[5]{https://gis.cdc.gov/grasp/fluview/fluportaldashboard.html} dataset tracks the number of patients and the influenza-like illness ratio on a weekly basis. The ETT\footnote[6]{https://github.com/zhouhaoyi/ETDataset} (Electricity Transformer Temperature) datasets are collected from two different electric transformers, labeled 1 and 2, each containing data at two resolutions: 15 minutes (m) and 1 hour (h). This results in four ETT datasets: ETTm1, ETTm2, ETTh1, and ETTh2. Detailed statistics can be found in Table \ref{tab:time_series_info}.

\begin{table}[t]
\caption{\small Statistics of time series benchmarks.}
\centering
\renewcommand\arraystretch{1.3}
\scalebox{0.9}{ 
\begin{tabular}{ccccccccc} \Xhline{4\arrayrulewidth}
Dataset & Weather & Traffic & Electricity & Illness & ETTh1 & ETTh2 & ETTm1 & ETTm2 \\ \hline
Features & 21 & 862 & 321 & 7 & 7 & 7 & 7 & 7 \\ \hline
Timesteps & 52696 & 17544 & 26304 & 966 & 17420 & 17420 & 69680 & 69680 \\ \Xhline{4\arrayrulewidth}
\end{tabular}}
\label{tab:time_series_info}
\end{table}

\subsubsection{Model Details}
\textbf{ViT-Tiny} is composed of 6 layers that integrate localized self-attention (LSA), feedforward networks (FFN), and PreNorm layer normalization, with 8 attention heads and 512 hidden dimensions. We use a patch size of 4 and GeLU as the activation function for the FFN. This model is employed in the CIFAR-10 experiment displayed in Figure \ref{fig:image_text} (a).
\\

\noindent
\textbf{CaiT-Tiny} consists of 6 layers of patch-to-patch attention and 2 layers of cross-attention between CLS tokens and patches, with 6 attention heads and 256 hidden dimensions. We use a patch size of 4 and GeLU as the activation function for the FFN. This model is utilized in the CIFAR-10 experiment shown in Figure \ref{fig:image_text} (a).
\\

\noindent
\textbf{ViT-Base} comprises 12 transformer layers of patch-to-patch attention with PreNorm layer normalization, it has 12 attention heads and 3072 hidden dimensions. We use a patch size of 16 and GeLU as the activation function for the FFN. This model is used for the ImageNet experiment depicted in Figure \ref{fig:image_text} (b).
\\

\noindent
\textbf{CaiT-Medium} consists of 24 layers of patch-to-patch attention and 2 layers of cross-attention between CLS tokens and patches, with 16 attention heads and 3072 hidden dimensions. We use a patch size of 16 and GeLU as the activation function for the FFN. This model is utilized in the ImageNet experiment shown in Figure \ref{fig:image_text} (b).
\\

\noindent
\textbf{Language Models} ~We used the small versions of language models developed by \citep{schlag2021linear}. For both the Transformer and Performer models, the dimensions of the key, value, and query were set to 128, and the context length for training and evaluation was configured to 256. Each model was assigned 8 self-attention heads, with the feedforward network (FFN) having a hidden dimension of 2048. The number of attention layers was set to 16. These models are used for the WikiText-103 experiment depicted in Figure \ref{fig:image_text} (c).
\\

\noindent
\textbf{Time Series Models} ~We use the supervised PatchTST model with the default parameter configurations proposed in \cite{Yuqietal-2023-PatchTST}. The look-back window is set to 336, with a patch length of 16 and a stride of 8. For the Illness dataset, the prediction length is 36, while for the rest of the datasets, it is set to 192. For the Transformer model, we use a standard self-attention transformer with an encoder-decoder architecture, consisting of 2 encoder layers and 1 decoder layer. The model uses 8 attention heads and a feed-forward network (FFN) hidden dimension of 2048.

\subsection{Quadratic Gating vs. Linear Gating}
\subsubsection{Dataset details}
\textbf{FineWeb-Edu Dataset.} FineWeb-Edu 10BT \citep{penedo2024fineweb} is a large-scale dataset with 10 billion tokens curated from diverse educational sources like textbooks and academic publications, designed to enhance language models for educational tasks. Its high-quality filtering and deduplication processes make it ideal for training relatively small-scale language models.

\subsubsection{Model details.}
\textbf{GPT2-MoE.} We examined the MoE adaptation of the GPT-2~\citep{radford2019language} transformer structure. Specifically, we substituted the FFN layers with an MoE layer consisting of 8 experts. To ensure a fair comparison between MoE models and the dense model, we set the hidden size of the FFN layer ($d_\mathrm{ff}$) so that the active parameters during inference are roughly equivalent. We also set the quadratic gate rank to 32 to limit the number of additional gating parameters. Refer to Table~\ref{tab:param_count} for comprehensive details on the parameter counts.
Details on the hyperparameters related to the architecture and training process are provided in Table~\ref{tab:hyper-params}.

\begin{table}[ht]
    \centering
    \begin{tabular}{c|c|c|c}
         \textbf{Model} & $d_{\mathrm{ff}}$ & Active Params. & Total Params. \\
         \Xhline{4\arrayrulewidth}
         GPT2-dense & $3072$ & $\approx 124.4M$ & $\approx 124.4M$\\
         GPT2-MoE (Linear) & $1536$ & $\approx 124.5M$ & $\approx 166.9M$\\
         GPT2-MoE (Quadratic) & $1536$ & $\approx 126.8M$ & $\approx 169.3M$\\
         \Xhline{4\arrayrulewidth}
    \end{tabular}
    \caption{Number of parameters in models.}
    \label{tab:param_count}
\end{table}

\begin{table}[ht]
    \centering
    \begin{tabular}{c|c}
        \textbf{Parameter} & \textbf{Value}\\
        \Xhline{4\arrayrulewidth}
         \texttt{block\textunderscore size} & 1024 \\
         \texttt{vocab\textunderscore size} & 50257 \\
         \texttt{n\textunderscore layer} & 12 \\
         \texttt{n\textunderscore head} & 12 \\
         \texttt{n\textunderscore embedding} & 768 \\
         \hline
         \texttt{n\textunderscore experts} & 8 \\
         \texttt{quadratic\textunderscore gate\textunderscore rank} & 32 \\
         \hline
         \texttt{total\textunderscore batch\textunderscore size} & 524288 \\
         \texttt{batch\textunderscore size} & 64 \\
         \hline
         \texttt{Optimizer} & AdamW\\
         \texttt{weight\textunderscore decay} & 0.1 \\
         \texttt{lr\textunderscore scheduler} & CosineAnealingLR \\
         \texttt{max\textunderscore lr} & 0.0006 \\
         \texttt{min\textunderscore lr} & 0.00006 \\
         \texttt{warmup\textunderscore steps} & 19073\\
        \Xhline{4\arrayrulewidth}         
    \end{tabular}
    \caption{Hyperparameters values.}
    \label{tab:hyper-params}
\end{table}

\section{Computational Overheads of Quadratic Gating}
\label{appendix:overhead}
In this section, we analyze the computational and memory overheads introduced by incorporating quadratic gating. Consider an MoE layer with $N$ two-layer MLP experts with hidden size of $d_\mathrm{ff}$. Note that each expert has $2 d \times d_\mathrm{ff}$ parameters. The linear gating network has $d$ parameters per expert while the introduction of quadratic terms in the gating network increases the parameter count to at most $d + d(d+1)/2$ per expert. Therefore, in each MoE layer the ratio of additional gating parameters to the total number of parameters is $d/4d_\mathrm{ff}$. 
\\

\noindent
\textbf{Parameter count overheads.} Note that this parameter count overheads is not necessarily negligible; however, this can be effectively managed by using low-rank embeddings for second-order terms in the router. For example, in Mixtral 8x7B, $d_\mathrm{ff} = 3.5 \times d$ which results in approximately $7\%$ increase in MoE layer parameter count if we use quadratic gating. The total parameter count overheads is almost $2.1$B, which is about $4\%$ of total number of parameters ($47$B) or $16\%$ of total number of active parameters ($13$B). In particular, assuming that the query and key linear embeddings in the Att-MoE framework are low-rank with rank $r \ll d$, the number of additional parameters per expert can be reduced to $(\frac{N+1}{N})(r \times d)$. Therefore, the ratio of additional gating parameters to the total number of parameters in the MoE layers can be as low as $(\frac{N+1}{N})(r/2d_\mathrm{ff})$. For the Mixtral 8x7B example, assuming $r=128$, the total parameter count overheads reduces to 150M parameters, which is only $0.3\%$ of all parameters and $1.1\%$ of the active parameters, making it relatively insignificant. Similarly, in our experiments with GPT2 level models, we used $r = 32$, which leads to roughly $2.3$M additional total parameters which is $1.4\%$ of the total parameters and $1.8\%$ of active parameters.
\\

\noindent
\textbf{Computational overheads.} The number of FLOPs is usually considered proportional to the number of non-embedding parameters. In the case of the Mixtral 8x7B example, utilizing low-rank embeddings with $r=128$ for both query and key vectors within the Att-MoE framework adds 150M parameters, resulting in just a $1\%$ increase in the number of FLOPs.
\\

\noindent
\textbf{Memory overheads.} Since all parameters must be loaded into memory for inference or training, the memory usage of MoE models is proportional to their total number of parameters. For example, in the Mixtral 8x7B model, using low-rank embeddings with $r=128$ in the Att-MoE framework, which leads to merely a $0.3\%$ increase in memory overheads.

\section{Ablation Study on the Number of Attention Heads}
\label{appendix:ablation_study}

We present results from an ablation study examining the effect of the number of attention heads. We select the best-performing activation function for each dataset, along with the linear activation, and evaluate their performance across varying numbers of attention heads using the best base model for that dataset. The results in Table \ref{tab:attention_heads} show that performance improves as the number of heads increases in most situations. This is particularly evident in image-based tasks and most time-series forecasting tasks, supporting the intuition that increased model complexity can enhance performance, especially for large-scale datasets. Information on model architectures, hyperparameters, and training details of experiments on non-linear gating functions can be found in Tables \ref{tab:vit_summary} - \ref{tab:patchtst_train}.

\begin{table}[ht]
\centering
\caption{Performance across varying numbers of attention heads.}
\label{tab:attention_heads}
\begin{tabular}{ll|ccccc}
\toprule
\textbf{Dataset} & \textbf{Activation} & \multicolumn{5}{c}{\textbf{\# of Attention Heads}} \\
\cmidrule(lr){3-7}
 & & 2 & 4 & 8 & 12 & 16 \\
\midrule
\multirow{2}{*}{CIFAR-10} 
  & Linear & 79.41 & 82.68 & 84.98 & 85.22 & 85.88 \\
  & GeLU   & 82.04 & 83.28 & 85.79 & 86.50 & 87.26 \\ \hline
\multirow{2}{*}{ImageNet} 
  & Linear & 73.34 & 79.35 & 82.31 & 82.44 & 84.42 \\
  & GeLU   & 78.55 & 80.79 & 84.47 & 85.36 & 86.13 \\ \hline
\multirow{2}{*}{Weather} 
  & Linear & 0.198 & 0.197 & 0.197 & 0.196 & 0.195 \\
  & Tanh   & 0.187 & 0.187 & 0.186 & 0.187 & 0.184 \\ \hline
\multirow{2}{*}{Traffic} 
  & Linear & 0.385 & 0.383 & 0.384 & 0.382 & 0.382 \\
  & Tanh   & 0.379 & 0.375 & 0.375 & 0.373 & 0.372 \\ \hline
\multirow{2}{*}{Electricity} 
  & Linear & 0.153 & 0.152 & 0.152 & 0.153 & 0.151 \\
  & Tanh   & 0.142 & 0.141 & 0.141 & 0.140 & 0.141 \\ \hline
\multirow{2}{*}{Illness} 
  & Linear & 1.517 & 1.474 & 1.462 & 1.458 & 1.449 \\
  & Tanh   & 1.498 & 1.447 & 1.453 & 1.442 & 1.439 \\ \hline
\multirow{2}{*}{ETTh1} 
  & Linear & 0.419 & 0.414 & 0.414 & 0.413 & 0.412 \\
  & Tanh   & 0.417 & 0.410 & 0.412 & 0.414 & 0.408 \\ \hline
\multirow{2}{*}{ETTh2} 
  & Linear  & 0.342 & 0.338 & 0.339 & 0.335 & 0.333 \\
  & Sigmoid & 0.329 & 0.325 & 0.323 & 0.318 & 0.317 \\ \hline
\multirow{2}{*}{ETTm1} 
  & Linear & 0.328 & 0.331 & 0.327 & 0.327 & 0.325 \\
  & Tanh   & 0.322 & 0.325 & 0.324 & 0.325 & 0.326 \\ \hline
\multirow{2}{*}{ETTm2} 
  & Linear & 0.212 & 0.212 & 0.212 & 0.211 & 0.212 \\
  & Tanh   & 0.212 & 0.212 & 0.212 & 0.211 & 0.212 \\ \hline
\bottomrule
\end{tabular}
\end{table}

\begin{table}[ht]
\centering
\caption{Summary of ViT-Small Model for CIFAR-10}
\label{tab:vit_summary}
\small
\begin{tabular}{@{}ll@{}}
\toprule
\multicolumn{2}{c}{\textbf{Model Architecture}} \\ \midrule
Base Model & Vision Transformer (ViT) \\
Patch Embedding & Shifted Patch Tokenization (SPT) \\
Transformer Blocks & 6 \\
Attention Mechanism & Local Self-Attention (LSA) \\
FFN Activation & GELU \\
Classifier Head & LayerNorm + Linear Layer \\
\bottomrule \\
\toprule
\multicolumn{2}{c}{\textbf{Hyperparameters}} \\ \midrule
Input Image Size & 32x32 \\
Patch Size & 4x4 \\
Embedding Dimension & 512 \\
Attention Heads & 8 \\
Dimension per Head & 64 \\
Transformer MLP Dimension & 512 \\
LSA Activation Function & GELU \\
Pooling & CLS Token \\
Dropout & 0.1 \\
Embedding Dropout & 0.1 \\
\bottomrule \\
\toprule
\multicolumn{2}{c}{\textbf{Training Details}} \\ \midrule
Dataset & CIFAR-10 \\
Epochs & 400 \\
Batch Size & 512 \\
Optimizer & Adam \\
Learning Rate & 1e-4 \\
LR Scheduler & Cosine Annealing \\
Loss Function & Cross-Entropy Loss \\
Data Augmentation & RandAugment, RandomCrop, RandomHorizontalFlip \\
Mixed Precision & Enabled \\ \bottomrule
\end{tabular}
\end{table}

\begin{table}[ht]
\centering
\caption{Summary of ViT Model for ImageNet}
\label{tab:vit_imagenet_summary}
\small
\begin{tabular}{@{}ll@{}}
\toprule
\multicolumn{2}{c}{\textbf{Model Architecture}} \\ \midrule
Base Model & Vision Transformer (ViT) \\
Patch Embedding & Linear Projection \\
Transformer Blocks & 12 \\
Attention Mechanism & Multi-Head Self-Attention \\
Attention Activation & Configurable (e.g., GELU, ReLU, etc.) \\
FFN Activation & GELU \\
Classifier Head & LayerNorm + Linear Layer \\
\bottomrule \\
\toprule
\multicolumn{2}{c}{\textbf{Hyperparameters}} \\ \midrule
Input Image Size & 224x224 \\
Patch Size & 16x16 \\
Embedding Dimension & 768 \\
Attention Heads & 12 \\
Dimension per Head & 64 \\
Transformer MLP Dimension & 3072 (4 * Embedding Dim) \\
Pooling & CLS Token \\
Dropout & 0.1 \\
Embedding Dropout & 0.1 \\
\bottomrule \\
\toprule
\multicolumn{2}{c}{\textbf{Training Details}} \\ \midrule
Dataset & ImageNet \\
Epochs & 300 \\
Batch Size & 8 \\
Optimizer & Adam \\
Learning Rate & 1e-4 \\
LR Scheduler & Cosine Annealing \\
Loss Function & Cross-Entropy Loss \\
Data Augmentation & RandAugment, RandomResizedCrop, RandomHorizontalFlip \\
Mixed Precision & Enabled \\ \bottomrule
\end{tabular}
\end{table}

\begin{table}[htbp]
\centering
\caption{Summary of PatchTST model architecture, hyperparameters, and training details.}
\label{tab:patchtst_summary}
\resizebox{\textwidth}{!}{%
\begin{tabular}{p{0.3\linewidth} p{0.7\linewidth}}
\toprule
\multicolumn{2}{c}{\textbf{Model Architecture}} \\
\midrule
\textbf{Component} & \textbf{Description} \\
\midrule
Input Processing & The input time series is first normalized using Reversible Instance Normalization (RevIN). \\
Patching & The normalized series is partitioned into overlapping patches. \\
Patch Embedding & Each patch is linearly projected to an embedding. \\
Positional Encoding & Learned positional encodings are added to the patch embeddings. \\
Encoder & A standard Transformer encoder architecture is used to process the sequence of patch embeddings. \\
Attention & Standard multi-head self-attention. The implementation includes residual attention connections as in Realformer. A non-standard application of a non-linear activation function is applied to the value projection (V) inside the attention mechanism. \\
Feed-Forward Network & Each encoder layer contains a position-wise feed-forward network. \\
Head & The output embeddings from the encoder are flattened and passed through a linear layer to produce a forecast. An option for channel-independent heads is available. \\
\midrule
\multicolumn{2}{c}{\textbf{Hyperparameters}} \\
\midrule
\textbf{Hyperparameter} & \textbf{Value} \\
\midrule
Sequence Length & 336 \\
Prediction Length & 192 \\
Patch Length & 16 \\
Stride & 8 \\
Encoder Layers & 3 \\
Attention Heads & A range is tested, e.g., 2, 4, 8, 16 \\
Model Dimension & 16 \\
FFN Dimension & 128 \\
Dropout & 0.3 \\
FC Dropout & 0.3 \\
Head Dropout & 0.0 \\
FFN Activation & GeLU \\
Attention Value Activation & Configurable (e.g., GELU, ReLU, etc.) \\
\bottomrule
\end{tabular}
}
\end{table}

\begin{table}[ht]
\centering
\caption{PatchTST Training Details}
\label{tab:patchtst_train}
\begin{tabular}{ll}
\toprule
\textbf{Detail} & \textbf{Configuration} \\
\midrule
Optimizer & Adam \\
Learning Rate & 0.0001 \\
Learning Rate Scheduler & OneCycleLR \\
Loss Function & Mean Squared Error (MSE) \\
Batch Size & 16 \\
Epochs & 100 \\
Early Stopping & Enabled with a patience of 100 epochs. \\
\bottomrule
\end{tabular}
\end{table}

\clearpage
\bibliography{ArXiv/ref}
\bibliographystyle{abbrv}

\end{document}